\documentclass{article}

\usepackage[preprint]{neurips_2025}

\usepackage[utf8]{inputenc} % allow utf-8 input
\usepackage[T1]{fontenc}    % use 8-bit T1 fonts
\usepackage{hyperref}       % hyperlinks
\usepackage{url}            % simple URL typesetting
\usepackage{booktabs}       % professional-quality tables
\usepackage{amsfonts}       % blackboard math symbols
\usepackage{nicefrac}       % compact symbols for 1/2, etc.
\usepackage{microtype}      % microtypography
\usepackage{xcolor}         % colors

\usepackage{dsfont}
\usepackage{float}
\usepackage{wrapfig}

\usepackage{algorithm}
\usepackage{algorithmic}
\usepackage{amsmath}
\usepackage{amssymb}
\usepackage{mathtools}
\usepackage{amsthm}
\usepackage{multirow}
\usepackage{caption}
\usepackage{subcaption}

\theoremstyle{plain}
\newtheorem{theorem}{Theorem}[section]

\newtheorem{lemma}[theorem]{Lemma}
\newtheorem{corollary}[theorem]{Corollary}
\theoremstyle{definition}

\theoremstyle{remark}

\title{Clustered Federated Learning \\ via Embedding Distributions}

\author{%
  Dekai Zhang, Matthew Williams, Francesca Toni \\
  Imperial College London\\
  \texttt{dz819@ic.ac.uk} \\
}

\begin{document}

\maketitle
\begin{abstract}
Federated learning (FL) is a widely used framework for machine learning in distributed data environments where clients hold data that cannot be easily centralised, such as for data protection reasons. FL, however, is known to be vulnerable to non-IID data. Clustered FL addresses this issue by finding more homogeneous clusters of clients. We propose a novel one-shot clustering method, EMD-CFL, using the Earth Mover's distance (EMD) between data distributions in embedding space. We theoretically motivate the use of EMDs using results from the domain adaptation literature and demonstrate empirically superior clustering performance in extensive comparisons against 16 baselines and on a range of challenging datasets.\footnote{For source code, see \url{https://github.com/dkaizhang/emdcfl}.}
\end{abstract}
\section{Introduction}
Federated learning (FL) \citep{konevcny2016federated,mcmahan2017communication} is a key learning framework for training machine learning models in a decentralised manner. In contrast to traditional centralised learning, FL holds potential benefits for privacy and naturally adapts to the distributed data environment in domains such as mobile devices or healthcare \citep{bonawitz2019towards, rieke2020future}. The standard algorithm in FL, FedAvg \citep{mcmahan2017communication}, enables clients to collaboratively train a global model by interleaving local training on client data and aggregating the local models to create an updated global model. FedAvg, however, operates under the assumption that the data held by the different clients are independently and identically distributed (IID). In practice, non-IID data is much more common and presents a particular problem for the standard FedAvg algorithm \citep{li2020federated,kairouz2021advances}.

Clustered FL \citep{sattler2020clustered, ghosh2020efficient, mansour2020three} addresses this problem by identifying a cluster structure amongst the clients with the goal of recovering the IID assumption within each cluster. Thus, instead of learning a single global model as in standard FL, clustered FL groups clients into clusters and learns cluster-specific models within each group. Broadly, most existing methods in this area group clients based on a distance between model parameters or parameter gradients \citep{sattler2020clustered, briggs2020federated, duan2021flexible, long2023multi, islam2024fedclust, kim2024clustered} or based on the local loss of cluster models \citep{ghosh2020efficient, marfoq2021federated, ruan2022fedsoft, guo2023fedrc, cai2023fedce}. A unifying assumption of most of these methods is knowledge of the number of clusters, which is unlikely to be known in advance. Exceptions exist. CFL \citep{sattler2020clustered} recursively partitions clients based on a gradient norm threshold, while FedClust \citep{islam2024fedclust} clusters clients based on a distance threshold between model parameters, and PACFL \citep{vahidian2023efficient} clusters clients based on a threshold on the principal angles derived from the clients' raw data. Other than PACFL, all of the above methods rely on repeated rather than one-shot clustering.

In this paper, we use insights from the domain adaptation literature, which lead us to (uniquely) focus on the Earth Mover's distance (EMD) between distributions in embedding space. We call our method EMD-CFL. By leveraging representations learned by deep neural networks, our method achieves superior clustering performance in one shot and better scalability to deep architectures. We discuss related works in Appendix~\ref{appx:related_works}.

\paragraph{Contributions.} We theoretically motivate the use of EMDs for clustered FL with insights from the domain adaptation literature and use these to design our new method, EMD-CFL (Figure~\ref{alg:cfl}).

We explore the theoretical relationship with prior parameter and gradient clustering methods and show that our method targets an upper bound of the parameter and gradient distances, which prior methods aim to reduce. 

We extensively evaluate our approach on 5 datasets and against 16 baselines using both simple and modern deep neural networks. We further test our method for robustness against partial participation, hyperparameter selection and model architecture choices.

\section{Theoretical Motivation}
\subsection{Clustered Federated Learning}
We consider a common clustered FL setup with a central server and $C$ clients, where each client $c \in [C] = \{1,\ldots,C\}$ owns a dataset consisting of labelled inputs $\{(x_i^c, y_i^c)\}_{i=1}^{N_c}$, where $x_i \in \mathcal{X}$, $y_i \in \mathcal{Y}$ and $(x_i, y_i) \in \mathcal{D} = \mathcal{X} \times \mathcal{Y}$. We further assume that the client datasets are sampled from one of $K$ distinct underlying data distributions $\mathcal{D}_1,\ldots,\mathcal{D}_K$, with $K \leq C$. The clients therefore follow a ground-truth clustering based on the underlying data distribution, which can be defined as $S^*_k = \{c \in [C] | (x^c,y^c) \sim \mathcal{D}_k \}$. We define an associated ground-truth cluster assignment function $\pi^*:[C] \rightarrow [K]$ so that $\pi^*(c) = k \text{ if and only if } c \in S^*_k$. Note that as we generally do not know the ground-truth clustering, we will need to define a method for finding an empirical cluster assignment function $\pi$, with $S_k = \{c \in [C] | \pi(c) = k \}$.

In this setting, our ultimate goal is to learn models $f_\theta : \mathcal{X} \rightarrow \mathcal{Y}$ parameterised by $\theta \in \Theta$ where $\Theta \subset \mathbb{R}^d$ is the space of parametric models. Let $\ell(f_\theta(x), y) : \mathcal{Y} \times \mathcal{Y} \rightarrow \mathbb{R}$ be the loss of model $f_\theta$ on an instance $(x,y)$. Given a cluster assignment $\pi$, we then aim to learn parameters $\theta_k$ which minimise the sum of the expected population losses across clients: 
\begin{equation}
    \sum_{c=1}^{C} w_c \mathcal{L}(\theta_{\pi(c)}) = \sum_{c=1}^{C} w_c \mathbb{E}_{(x,y) \sim \mathcal{D}_{\pi^*}}[\ell (f_{\theta_{\pi(c)}}(x), y)]    
\end{equation}
where $w_c$ denotes the weight of client $c$'s loss, which is typically set to the client's share of the total data $N_c/N$. Since we only have access to a finite sample $\{(x_i^c, y_i^c)\}_{i=1}^{N_c}$ per client, we 
% adopt the empirical risk minimisation approach \citep{vapnik1991principles} and 
instead minimise the empirical risk:
\begin{equation}
    \sum_{k=c}^{C} w_c \mathfrak{L}(\theta_{\pi(c)}) = \sum_{c=1}^{C} w_c \frac{1}{N_c} \sum_{i=1}^{N_c} \ell (f_{\theta_{\pi(c)}}(x_i), y_i)
\end{equation}

\subsection{A Domain Adaptation Perspective}
\label{sec:da}
We present insights from the domain adaptation literature. While this body of work typically deals with standard centralised learning, we can use it to derive a bottom-up approach to finding a clustering function $\pi$ for clustered FL. 

For the analysis in this section, we consider a common setting of binary classification \citep{ben2006analysis, ben2010theory, mansour2009domain, redko2017theoretical, shen2018wasserstein} and assume that there exists a common ground-truth labelling function $\psi : \mathcal{X} \rightarrow [0,1]$. Let a model consist of an embedding model and a hypothesis, so that $f_{\theta} = h_{\phi} \circ g_{\omega}$. We assume that $g_{\omega} : \mathcal{X} \rightarrow \mathcal{Z}$ is parameterised by $\omega$ and maps inputs to an embedding space $\mathcal{Z}$. $h_{\phi} : \mathcal{Z} \rightarrow \mathcal{Y}$ is parameterised by $\phi$ and subsequently maps embeddings to the output space. The concatenation of their parameters forms $\theta = [\omega, \phi]$. 
In the remainder of this section, we suppress the parameters for ease of notation. We will also make use of the EMD, also known as the 1-Wasserstein distance, which for two probability distributions $\mu_a, \mu_b \in \mathcal{P}(A)$ over space $A$ with a joint distribution $\gamma \in \Gamma(A \times A)$ and a transport cost $\zeta(x,y)$ is defined as:

\begin{equation}
\label{eq:wasserstein}
    W_1(\mu_a,\mu_b) = \inf_{\gamma \in \Gamma(\mu_a,\mu_b)} \int \zeta(x,y) \text{ d}\gamma(x,y) 
\end{equation}

Given a distribution $\mu^{\mathcal{X}}$ over $\mathcal{X}$, we can then induce a distribution $\mu^{\mathcal{Z}}$ over the embedding space $\mathcal{Z}$ using an embedding model $g$. We can then define the induced target function $\hat{\psi} : \mathcal{Z} \rightarrow [0,1]$, as $\hat{\psi}(z) = \mathbb{E}_{x \sim \mu^\mathcal{X}}[\psi(x) | g(x)=z]$ \citep{ben2006analysis}. In the following, $\mu^\mathcal{X}_c$ denotes the distribution of inputs and $\mu^\mathcal{Z}_c$ the induced distribution over embeddings of a client $c$. We also define the disagreement between two hypotheses on a distribution $\mu^\mathcal{Z}_c$ as $\epsilon_c(h,h') = \mathbb{E}_{z \sim \mu^\mathcal{Z}_c} [|h(z) - h'(z)|]$ and abbreviate $\epsilon_s(h) = \epsilon_s(h,\hat{\psi})$. We can now state an adapted version of a result due to \citet{shen2018wasserstein}.

\begin{lemma} \citep{shen2018wasserstein}
\label{lem:shen}
    Let $\mu^\mathcal{Z}_i, \mu^\mathcal{Z}_j \in \mathcal{P}(\mathcal{Z})$ be two probability distributions over $\mathcal{Z}$. Assume the hypotheses $h, h'$ are L-Lipschitz continuous with respect to $\mathcal{Z}$. Then the following holds 
    \begin{equation*}
        \epsilon_i(h,h') \leq \epsilon_j(h,h') + 2LW_1(\mu^\mathcal{Z}_i, \mu^\mathcal{Z}_j)
    \end{equation*}
\end{lemma}

In words, this implies that the disagreement of hypotheses on one distribution can be bounded above using their disagreement on another distribution and a quantity proportional to the EMD between distributions. We next derive an EMD-based generalisation bound, similar to \citet{shen2018wasserstein}, but focus on the generalisation bound of a hypothesis $h \in H$ on a target probability measure that is a mixture $T = \alpha \mu^\mathcal{Z}_i + (1-\alpha) \mu^\mathcal{Z}_j$. 

\begin{theorem}
\label{thm:mixture-ub}
    Under the assumptions from Lemma~\ref{lem:shen},
    \begin{equation*}
        \epsilon_T(h) \leq \alpha \epsilon_j(h) + (1-\alpha) \epsilon_i(h) + 2LW_1(\mu^\mathcal{Z}_i, \mu^\mathcal{Z}_j) + \lambda
    \end{equation*}
    where $\lambda = min_{h \in H} \epsilon_i(h) + \epsilon_j(h)$ is the sum error of the ideal hypothesis.
\end{theorem}

Theorem~\ref{thm:mixture-ub} (proofs in Appendix~\ref{appx:proofs}) allows us to state a generalisation bound for the ensemble hypothesis of two clients, which is often used to analyse the federated hypothesis \citep{peng2019federated, lin2020ensemble, zhu2021data}. 

\begin{corollary}
\label{cor:fl-ub}
    Under the assumptions of Lemma~\ref{lem:shen}, let $h_i$ and $h_j$ be hypotheses learned on $\mu^\mathcal{Z}_i$ and $\mu^\mathcal{Z}_j$ respectively, so that the ensemble hypothesis is $h_{e} = \alpha h_i + (1-\alpha) h_j$. We then have
    \begin{equation}
    \begin{split}
        \epsilon_T(h_e) \leq &\ \alpha \left[\alpha \epsilon_j(h_i) + (1-\alpha) \epsilon_i(h_i) \right] \\
                            &+ (1-\alpha) \left[\alpha \epsilon_j(h_j) + (1-\alpha) \epsilon_i(h_j)\right] \\
                            &+ 2LW_1(\mu^\mathcal{Z}_i, \mu^\mathcal{Z}_j) + \lambda
    \end{split}
    \end{equation}
\end{corollary}

Corollary~\ref{cor:fl-ub} shows that the generalisation bound of the federated hypothesis depends on the EMD between the distributions of the clients, Notably, to constrain the generalisation error via the upper bound, we therefore want to only cluster clients $i$ and $j$ if the EMD between their embedding distributions is small. This provides us with a bottom-up clustering strategy, which we develop into an algorithm in the next section. 

\paragraph{Discussion.} We decided to focus on the generalisation bounds based on distances in the embedding space, as the generalisation bounds rely on the assumption of Lipschitz continuity. The Lipschitz constant can be very large when considering the continuity of a deep neural network with respect to the input space \citep{szegedy2013intriguing, virmaux2018lipschitz}, potentially making the bound very loose. By focusing on the continuity of the hypothesis with respect to the embedding space, the bounds are likely to be more meaningful, as hypotheses are often implemented as shallow networks if not a single linear layer, where the latter is inherently Lipschitz continuous.

\subsection{Gradient and Parameter Bounds}
\label{sec:grad_bound}
Priors works have proposed clustering approaches using gradients \citep{sattler2020clustered, duan2021flexible, kim2024clustered} or parameter distances \citep{long2023multi, islam2024fedclust}. We establish a relationship with these prior works by showing that our approach of focusing on the EMD between embedding distributions also targets upper bounds to the distance between gradients and between parameters.

We first re-write our earlier loss function as $\hat{\ell}(z;\phi) = \ell(h_{\phi}(z),\hat{\psi}(z))$, so that $\hat{\ell} : \mathcal{Z} \rightarrow \mathbb{R}$. We can then use a result due to \citet{fallah2020personalized} to show that the distance between expected gradients can be bounded using the EMD between embedding distributions.
 
\begin{theorem} \citep{fallah2020personalized}
\label{thm:gradbound}
    Let $\mu^\mathcal{Z}_i, \mu^\mathcal{Z}_j \in \mathcal{P}(\mathcal{Z})$ be probability distributions over $\mathcal{Z}$. Suppose $\hat{\ell}$ is M-Lipschitz continuous with respect to $\mathcal{Z}$ and letting $\hat{\mathcal{L}}_i(z;\phi) = \mathbb{E}_{z \sim \mu^\mathcal{Z}_i}[\hat{\ell}(z;\phi)]$ then
    \begin{equation*}
        \lVert \nabla \hat{\mathcal{L}}_i(z;\phi) - \nabla \hat{\mathcal{L}}_j(z;\phi) \rVert \leq MW_1(\mu^\mathcal{Z}_i, \mu^\mathcal{Z}_j)
    \end{equation*}
\end{theorem}

Since clients start with the same model and take gradient steps locally, the expected parameters of client $i$ in round $t+1$ can be written as $\mathbb{E}[\phi_i^{t+1}] = \phi_i^{t} - \kappa \mathbb{E}[\hat{\mathcal{L}}_i(z;\phi)]$ where $\kappa$ is the learning rate. We can now use Theorem~\ref{thm:gradbound} to show that the EMD between embedding distributions also bounds the expected distance between parameters. In Appendix~\ref{appx:cluster_total_param_distances}, we show empirical evidence.

\begin{corollary}
    Under the assumptions of Theorem~\ref{thm:gradbound} and for some learning rate $\kappa$ we have for a given round $t$
    \begin{equation*}
        \mathbb{E}[\lVert \phi_i^{t+1} - \phi_j^{t+1} \rVert] \leq \frac{M}{\kappa}W_1(\mu^\mathcal{Z}_i, \mu^\mathcal{Z}_j)
    \end{equation*}
\end{corollary}

These results again focus on the EMD in the embedding space and bound the gradient and parameter distances of the hypotheses, which recent works have proposed for clustering \citep{kim2024clustered, islam2024fedclust}. 

\section{Algorithm Design of EMD-CFL}

The results from the previous section suggest that clients with different embedding distributions should not be clustered together, motivating the use of EMDs in making the clustering decision. Based on this intuition we design EMD-CFL, a bottom-up one-shot clustering method which estimates the pairwise EMD between clients and clusters them accordingly. 
\begin{wrapfigure}{r}{0.5\textwidth} % r: right side, width: 40% of text width
    \centering
    \includegraphics[width=0.5\textwidth]{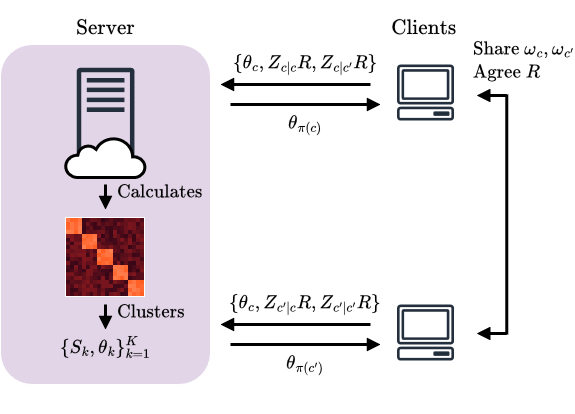}
    \caption{Illustration of Algorithm~\ref{alg:cfl}.}
    \label{fig:overview}
\end{wrapfigure}

Calculating the EMDs between embedding distributions is straightforward if we have central access to the raw data. In FL, however, that is usually not given due to data privacy requirements. We maintain these requirements by only sharing embeddings, which have been randomly projected. We specifically make use of the Johnson-Lindenstrauss Lemma \citep{johnson1984extensions} (stated in Appendix~\ref{appx:JL}), which states that random projections to a low-dimensional space nearly preserve all distances. We use a 10\% dimensionality reduction for our experiments but find that our method remains robust to much greater reductions of up to 90\%.
We then consider a system in which a trusted central server calculates the EMD between randomly projected embeddings and performs one-shot clustering. The server maintains the clusters $\{S_k\}_{k=1}^K$ and the corresponding cluster models $\{\theta_k\}_{k=1}^K$ as usual in clustered FL. We provide a diagrammatic representation in Figure~\ref{fig:overview} and show a more detailed implementation of the clustering in Algorithm~\ref{alg:cfl}.

\begin{algorithm*}[!htbp]
    \caption{EMD-CFL}
    \label{alg:cfl}
\begin{algorithmic}[1]
    \STATE {\bfseries Input:} Number of clients $C$, initial model $\theta^{(0)}$, distance tolerance $\epsilon$, learning rate $\kappa$, global epochs $T$, local epochs $E$
    \STATE Initialise each client $c \in [C]$ with a copy of $\theta^{(0)}$.
    \STATE Initialise identity adjacency matrix $M = I \in \mathbb{R}^{C \times C}$ and empty $W \in \mathbb{R}^{C \times C}$ to store EMDs
    \FOR{$t=1$ {\bfseries to} $T$}
        \STATE Server gets participating clients $P^{(t)} \subseteq [C]$.
        \FOR{all $c \in P^{(t)}$ in parallel}
            \STATE $\theta^{(t)}_c \leftarrow$ \textsc{ClientUpdate}($\theta^{(t-1)}_c$, $E$, $\kappa$)
            % \STATE Client computes $\tau_c \leftarrow W_1(g_{\omega_c}(X_c), g_{\omega_c}(X^{val}_c))$
            \STATE Clients sends $\theta^{(t)}_c$ to server
        \ENDFOR
        \FOR{all $c \in P^{(t)}$}
            \IF{$W[c][c']$ is None}
                \STATE Clients $c$ and $c'$ agree on $R$
                \STATE Clients $c$ and $c'$ exchange $\omega_c$ and $\omega_{c'}$
                \STATE Client $c$ sends $\tau_c$, $Z_{c|c}R$ and $Z_{c|c'}R$ to server
                \STATE Client $c'$ sends $\tau_{c'}$, $Z_{c'|c}R$ and $Z_{c'|c'}R$ to server 
                \STATE Server updates $W[c][c'] \leftarrow W_1(Z_{c|c}R, Z_{c'|c}R) - \tau_c$ and $W[c'][c] \leftarrow W_1(Z_{c'|c}R, Z_{c'|c'}R) - \tau_{c'}$ 
                \STATE Server updates $M[c][c'] = M[c'][c] \leftarrow \mathds{1}\{W[c][c'] < \epsilon \text{ and } W[c'][c] < \epsilon\}$
            \ENDIF
        \ENDFOR
        \STATE Identify clusters $\{S_k\}_{k=1}^K \leftarrow \textsc{Neighbourhoods($M$)}$ where $K$ is the number of distinct neighbourhoods
        \STATE Server performs \textsc{AggregateModels($\{\theta^{(t)}_c\}_{c \in S_k}$)} and broadcasts updated cluster models $\theta^{(t)}_k$.
    \ENDFOR
    \STATE {\bfseries Output:} Cluster models $\{\theta^{(T)}_k\}_{k=1}^K$
\end{algorithmic}
\end{algorithm*}

\paragraph{Input (Line 1).} Our main hyperparameter is $\epsilon$, which thresholds the maximum tolerated distance between clients of the same cluster in embedding space. 
% Notably, we do not require the number of clusters as an input.

\paragraph{Initialisation (Lines 2-3).} The server distributes a copy of the model $\theta^{(0)}$ to each client. The server initialises a $C \times C$ adjacency matrix $M$ to track the client clustering and a $C \times C$ matrix to store the pairwise EMDs. 

\paragraph{Partial Participation (Line 5).} The server receives notice of the clients who will be participating in the current epoch. For our experiments under partial participation, we make the common assumption that clients are randomly sampled each epoch.

\paragraph{Client Training (Lines 6-9).} The clients perform \textsc{ClientUpdate} by training on their local datasets for $E$ epochs and send the updated models back to the server.

\paragraph{Client Embeddings (Lines 11-15).} The clustering is one-shot for each pair of clients $c,c'$. They choose a common random projection matrix $R$ (e.g., by agreeing a projection dimension, a random seed and a random generator) which is kept private from the server and other clients. We treat clients equally and let them exchange embedding models $\omega_c, \omega_{c'}$ which allows each client to induce two distributions of their own data in embedding space. The exchange can either be orchestrated via the server or done peer-to-peer. $c$ can then produce embeddings $g_{\omega_c}(X_c) = Z_{c|c}$ using its own model and $g_{\omega_{c'}}(X_c) = Z_{c|c'}$ using the model of $c'$. We further propose each client calculate $\tau_c = W_1(Z_{c|c}, Z_{c|c}^{val})$ between their own training and validation data as a reference distance. The motivation is that we expect the validation data of client $c$ to be similarly far way from its training data as the data of client $c'$ if both clients' data came from the same distribution. Privacy is preserved, since the server only receives scalar reference distances and randomly projected embeddings without knowing the random projection matrix. Note that while $c$ and $c'$ know the random projection matrix, they do not share any embeddings with each other. Additional security could be provided by applying differential privacy to the projected embeddings \citep{dwork2006differential, kenthapadi2012privacy}.
% or federating the EMD calculation as well \citep{rakotomamonjy2024federated}.

\paragraph{Clustering (Lines 16-17).}
The server calculates the pairwise EMDs and updates the adjacency matrix $M$ by checking if each distance falls within a tolerance threshold $\epsilon$, which is intended to account for the variance in the EMD between two samples even when they come from the same distribution. We keep a fixed $\epsilon = 0.025$ for most of our experiments, finding it to be robust across datasets and models, but show that our method is stable across $\epsilon$-values. The adjacency matrix is symmetric, allowing each client to unilaterally refuse clustering with another client. 

\paragraph{Aggregation (Lines 20-21).} The neighbourhood of each client in the adjacency matrix represents the client's cluster, where identical neighbourhoods are mapped to the same cluster. The server performs \textsc{AggregateModels} within each cluster $S_k$. We use the commonly used FedAvg \citep{mcmahan2017communication} and take a weighted average with weights given by the dataset size of each client belonging to the cluster: $\theta^{(t)}_k = \sum_{c \in S_k}\frac{N_c}{N_k} \theta^{(t)}_c$ where $N_k = \sum_{c \in S_k} N_c$.

\paragraph{Next Epoch (Lines 22-23).} Note that once all clients have been clustered, training in each subsequent epoch will simply consist of alternating between clients performing \textsc{ClientUpdate($\theta_c^{(t-1)}, E, \kappa$)} of their respective cluster models and sending them back for the server to perform \textsc{AggregateModels($\{\theta^{(t)}_c\}_{c \in S_k}$)}. The final output are the cluster models $\{\theta^{(T)}_k\}_{k=1}^K$.

\subsection{Computational and Communication Considerations} 
\label{sec:cost}
The time complexity is $O(C^2 N^3\log N )$. The quadratic term is common in clustered FL. Calculating EMDs amounts to solving an optimal transport problem. We use the solver from Python Optimal Transport \citep{flamary2021pot}, which achieves $O(N^3\log N)$ time complexity. Since this can be expensive, we restrict the number of samples to the lesser of 10\% and 512 instances. Since the clustering is one-shot, however, the complexity is constant with respect to the number of epochs $T$, so that the cost can be amortised. In Appendix~\ref{appx:timing}, we show that the increase in wall time is only about 9\% over not clustering, which roughly equals one extra epoch. We demonstrate in Appendix~\ref{appx:sinkhorn} that our method is robust to approximate Sinkhorn distances, which are closer to $O(N^2)$ \citep{cuturi2013sinkhorn}.

The communication complexity is $O(C^2 |\omega|)$, since each pair of clients exchange embedding models. EMD-CFL trades off higher one-off communication costs for more accurate and one-shot clustering. This makes our method particularly suitable for cross-silo FL, in which clients typically number in the tens rather than millions as in cross-device FL \citep{kairouz2021advances}. For the cross-device setting, the communication cost could be reduced to $O(C |\omega|)$ (in line with other methods) if clients stored a small set of samples server-side, so that their embeddings could be directly computed by the server. We do not assume this scenario for the purpose of maintaining a fair experimental setup and leave a detailed exploration of reducing communication costs to future work. 
\section{Experiments}
\label{sec:experiments}

\subsection{Setups}
\paragraph{Induced Clustering.}
We use the Rotated MNIST and Rotated CIFAR10 benchmarks (illustrated in Appendix~\ref{appx:mnist_examples} and \ref{appx:cifar_examples}) \citep{ghosh2020efficient, sattler2020clustered}, which are based on MNIST \citep{lecun1998mnist} and CIFAR10 \citep{krizhevsky2009learning}. 
We retain the provided training and test splits and use 10\% of the training data for validation. As in prior works, we induce a clustering structure by replicating the data and applying a rotation of $\{0^\circ, 90^\circ, 180^\circ, 270^\circ\}$. This yields $K=4$ clusters with images differing only in their rotation, giving 240000 images across 10 classes for each of Rotated MNIST and Rotated CIFAR10. We consider $C=40$ clients and split them evenly across clusters.

\paragraph{Natural Clustering.}
We consider PACS \citep{li2017deeper} to study our method under a pre-existing clustering structure. PACS consists of 9991 images with 7 classes sourced from four different domains---photo, art, cartoon, and sketch (illustrated in Appendix~\ref{appx:pacs_examples}). Our hypothesis is that the four domains correspond to $K=4$ different clusters. We split data from each domain so that each client holds at least 500 samples. We obtain $C=18$ clients with 3 clients with photos, 4 with art, 4 with cartoons, and 7 with sketches. We use a 80-10-10 split for training, validation and testing.

\paragraph{Backdoor Features.}
We provide Backdoor MNIST and Backdoor CIFAR10 (illustrated in Appendix~\ref{appx:cmnist_examples} and \ref{appx:cifarls_examples}), on which we test our method's ability to distinguish distribution shifts caused by visually subtle differences. We consider adversarial settings in which some clients have inserted visually subtle backdoor features with the goal of making a set of target clients depend on them, allowing them to be manipulated \citep{gu2019badnets, bagdasaryan2020backdoor, qin2023revisiting}. We test if our method can be used to isolate clients with backdoor features. For Backdoor MNIST, we induce an obvious difference by creating two subsets with either green or uniformly purple digits. We induce a subtle difference by further splitting the green digits into a subset with \emph{uniformly} green digits and one subset in which the \emph{intensity} of the green is correlated with the digit. ``Greenness'' thus acts as an undesirable backdoor feature. We instantiate the setting with $C=30$ clients split evenly between the three subsets. For Backdoor CIFAR10, we induce an obvious difference by creating label-skewed subsets with only the first and the last five classes. We induce a subtle difference within the first five classes by further splitting it in a clean subset and a subset where each image has a $3 \times 3$ colour patch in the corner, where the colour is correlated with the class and thus acts as the backdoor feature. We instantiate $C=15$ clients split evenly across the three subsets.

\subsection{Method Hyperparameters}
Our main hyperparameter is the distance tolerance $\epsilon$, which in general depends on the dataset, the embedding model and the random projection dimension $\text{dim}(R)$. We found $\epsilon=0.025$ and $\text{dim}(R)=0.9 \cdot \text{dim}(Z)$, or a 10\% dimensionality reduction, to be robust choices. We, however, show that our method remains stable across datasets and model architectures under wide ranges of $\epsilon$ and under much greater dimensionality reduction levels of up to 90\%.

\subsection{Baselines}
We compare against 16 baselines, including hard and soft clustered FL methods. CFL \citep{sattler2020clustered}, FlexCFL \citep{duan2021flexible}, FeSEM \citep{long2023multi}, FedClust \citep{islam2024fedclust} and CFLGP \citep{kim2024clustered} focus on distances between model parameters or parameter gradients. PACFL \citep{vahidian2023efficient} uses principle angles between client data. IFCA \citep{ghosh2020efficient} focuses on the local loss of cluster models instead. FedEM \citep{marfoq2021federated}, FedSoft \citep{ruan2022fedsoft}, FedRC \citep{guo2023fedrc} and FedCE \citep{cai2023fedce} also use the local loss but soft cluster clients, allowing clients to contribute to multiple clusters and thus potentially offering more efficiency. 
For fair comparison, we provide the ground-truth number of clusters to all baselines (or tune the distance threshold to achieve it), thus evaluating them under their best-case scenario. Any performance differences can then be attributed to incorrect assignments of clients to wrong clusters. 
We include an Oracle baseline with manually defined ground-truth cluster memberships. For non-clustering methods we use FedAvg \citep{mcmahan2017communication} and FedProx \citep{li2020federated} as standard baselines, as well as pFedGraph \citep{ye2023personalized} and FedSaC \citep{yan2024balancing}, which personalise the aggregation for each client. We detail hyperparameters of baselines in Appendix~\ref{appx:baseline_hparams}.

\subsection{Evaluation Metrics}
We evaluate downstream performance and report the average and worst accuracy across all clients. We evaluate the clustering performance of the hard clustering methods by comparing the empirical with the ground-truth clustering in the final epoch. Note that as one-shot methods, the clusters under PACFL and our method remain constant across epochs. We measure the adjusted Rand index (ARI) \citep{hubert1985comparing, steinley2004properties}, with values in $[-0.5, 1]$ where 0 corresponds to random and 1 corresponds to optimal clustering. We report averages over three runs.

\subsection{Models}
We use a simple two-layer CNN for all MNIST experiments. We use ResNet18 \citep{he2016deep} using ImageNet weights for all CIFAR10 experiments and PACS. 
We show additional results for ResNet50, EfficientNetB3 \citep{tan2019efficientnet}, ViTB16 \citep{dosovitskiy2020image}, and ConvNeXt Base \citep{liu2022convnet}. We train all models for $T=10$ global epochs with $E=10$ local epochs each. Further details on architectures (including embedding dimensions) and training are in Appendix~\ref{appx:hparams}.
\section{Results}
\label{sec:results}
We show that EMD-CFL achieves superior clustering performance. Our method is also robust to hyperparameter choices, larger model architectures and partial participation.

\subsection{Induced Clustering}
\paragraph{Rotated MNIST.} 
The left-hand side of Table~\ref{tab:mnistmax_cifarmax} show that several methods besides EMD-CFL are able to identify the correct clustering (bold). Most do not achieve Oracle accuracy, indicating that the correct clustering is only attained later during training. In contrast, PACFL and EMD-CFL achieve this in the first epoch, as one-shot clustering methods.
Soft clustering methods clearly underperform, especially when considering worst client accuracy. This indicates that soft clustering is not necessarily optimal under the presence of a clear clustering structure.

\begin{table}[htbp]
\centering
\caption{Performance comparison on both Rotated MNIST and Rotated CIFAR10. EMD-CFL is the only method to match the Oracle and obtains optimal clustering (bold) on both datasets.}
\label{tab:mnistmax_cifarmax}
\resizebox{0.9\textwidth}{!}{
\begin{tabular}{lccc@{\hspace{1.5em}}ccc}
\toprule
 & \multicolumn{3}{c}{\textbf{Rotated MNIST}} & \multicolumn{3}{c}{\textbf{Rotated CIFAR10}} \\ \cmidrule(lr){2-4} \cmidrule(lr){5-7}
\textbf{Method} & \textbf{Avg Acc} & \textbf{Worst Acc} & \textbf{ARI} & \textbf{Avg Acc} & \textbf{Worst Acc} & \textbf{ARI} \\
\midrule
Oracle & 98.86$\pm$0.05 & 97.58$\pm$0.26 & 1.00$\pm$0.00 & 94.83$\pm$0.02 & 93.03$\pm$0.38 & 1.00$\pm$0.00 \\ \cmidrule(lr){1-7}
EMD-CFL & \textbf{98.86$\pm$0.04} & \textbf{97.58$\pm$0.26} & \textbf{1.00$\pm$0.00} & \textbf{94.84$\pm$0.09} & \textbf{93.03$\pm$0.14} & \textbf{1.00$\pm$0.00} \\
CFL & 98.62$\pm$0.33 & 96.97$\pm$0.95 & 0.90$\pm$0.14 & 90.10$\pm$1.44 & 87.55$\pm$1.34 & 0.00$\pm$0.00 \\
PACFL & \textbf{98.86$\pm$0.02} & \textbf{97.73$\pm$0.00} & \textbf{1.00$\pm$0.00} & 94.30$\pm$0.09 & 80.50$\pm$0.95 & 0.93$\pm$0.00 \\
FedClust & \textbf{96.93$\pm$0.43} & \textbf{94.39$\pm$1.46} & \textbf{1.00$\pm$0.00} & 74.97$\pm$0.73 & 65.60$\pm$4.16 & -0.01$\pm$0.02 \\
IFCA & 98.44$\pm$0.33 & 96.67$\pm$0.95 & 0.80$\pm$0.14 & 92.14$\pm$0.68 & 89.25$\pm$1.06 & 0.44$\pm$0.18 \\
FlexCFL & \textbf{97.25$\pm$0.28} & \textbf{94.85$\pm$0.26} & \textbf{1.00$\pm$0.00} & 88.55$\pm$0.06 & 77.30$\pm$4.95 & 0.01$\pm$0.00 \\
FeSEM & 87.80$\pm$4.10 & 64.09$\pm$22.17 & 0.34$\pm$0.10 & 86.52$\pm$0.09 & 83.05$\pm$0.21 & 0.00$\pm$0.00 \\
CFLGP & \textbf{98.83$\pm$0.09} & \textbf{97.73$\pm$0.00} & \textbf{1.00$\pm$0.00} & 88.39$\pm$0.23 & 80.85$\pm$0.49 & -0.01$\pm$0.01 \\
FedEM & 94.86$\pm$1.24 & 84.55$\pm$4.17 & - & 90.27$\pm$0.07 & 87.55$\pm$0.07 & - \\
FedSoft & 98.59$\pm$0.05 & 96.82$\pm$0.45 & - & 93.35$\pm$0.20 & 89.00$\pm$0.57 & - \\
FedRC & 95.10$\pm$1.65 & 85.76$\pm$5.17 & - & 90.16$\pm$0.84 & 87.75$\pm$1.20 & - \\
FedCE & 98.73$\pm$0.15 & 96.82$\pm$0.79 & - & 93.61$\pm$0.87 & 90.50$\pm$0.87 & - \\
FedAvg & 91.04$\pm$0.12 & 86.67$\pm$0.69 & - & 90.09$\pm$1.38 & 87.80$\pm$1.84 & - \\
FedProx & 89.05$\pm$0.37 & 85.61$\pm$0.52 & - & 89.86$\pm$1.52 & 87.30$\pm$2.12 & - \\
pFedGraph & 98.20$\pm$0.09 & 96.06$\pm$0.52 & - & 88.08$\pm$0.02 & 77.80$\pm$1.41 & - \\
FedSaC & 97.19$\pm$0.95 & 89.70$\pm$5.46 & - & 85.77$\pm$0.25 & 71.85$\pm$0.07 & - \\
\bottomrule
\end{tabular}
}
\end{table}

\paragraph{Rotated CIFAR10.} The right-hand side of Table~\ref{tab:mnistmax_cifarmax} shows that on a more challenging dataset all methods, other than EMD-CFL, fail to recover the correct clustering. Methods based on parameter or gradient distances (including CFL and FedClust) perform the worst. We hypothesise that with larger models (with millions of parameters), the curse of dimensionality makes it more difficult for these approaches to succeed. 
PACFL achieves a similar average accuracy as EMD-CFL but a much poorer worst client accuracy due to imperfect clustering. Its focus on the input space may also suffer from the curse of dimensionality ($\text{dim}(X) = 224 \times 224 \times 3$). EMD-CFL focuses on relatively lower-dimensional embeddings and overcomes this issue ($\text{dim}(Z)=768$ for ResNet-18). 
In Appendix~\ref{appx:failure_analysis}, we show that the distances used by these methods are in fact unable to reflect the underlying clustering structure well. In Appendix~\ref{appx:induced}, we show that our method succeeds on 6 additional versions of MNIST and CIFAR10.

\paragraph{Robustness to Hyperparameter Choices.}
We study the clustering performance under different choices of $\epsilon$ and $\text{dim}(R)$ in Appendix~\ref{appx:hparams_flex}. Our results show that EMD-CFL remains stable on a wide range of $\epsilon$-values as well as under large dimensionality reductions of up to 90\%. In Appendix~\ref{appx:sinkhorn}, we show that these findings continue to hold under approximate Sinkhorn distances.

\paragraph{Robustness to Model Architectures.}

The choice of $\epsilon$ depends on the embedding model. Appendix~\ref{appx:architecture_flex} demonstrates that EMD-CFL is robust to model architecture choices and, in particular, larger model architectures.

\paragraph{Robustness to Partial Participation.}
In Appendix~\ref{appx:partial}, we test EMD-CFL under partial participation of $\{10,20,30\}$ clients and show that EMD-CFL continues to perform well.

\subsection{Natural Clustering}

Table~\ref{tab:pacs} shows that EMD-CFL successfully distinguishes the four domains of PACS and matches the Oracle in performance. Interestingly, the EMDs show that some domains are closer than others (Figure~\ref{fig:pacs_distance}). Notably, a model trained on the art domain considers other domains to be relatively closer. Intuitively, this may be due the visual diversity of the art domain, which may overlap with other domains (examples in Appendix~\ref{appx:pacs_examples}). This implies that a model trained on the art domain may be exposed to the other domains to some other degree and therefore consider them too ``distant''. In Appendix~\ref{appx:full_pacs}, we include the results for soft clustering methods (which underperform) and show that EMD-CFL is stable for different hyperparameter choices. The observation that some of the domains were closer to one another points to the fact that it is difficult to know in advance if different domains are best captured by separate clusters or not. Visualising the EMDs could be a helpful first step in determining this in general.

\begin{figure}[htbp]
  \centering
  \begin{minipage}{0.3\textwidth}
    \centering
    \includegraphics[width=\linewidth]{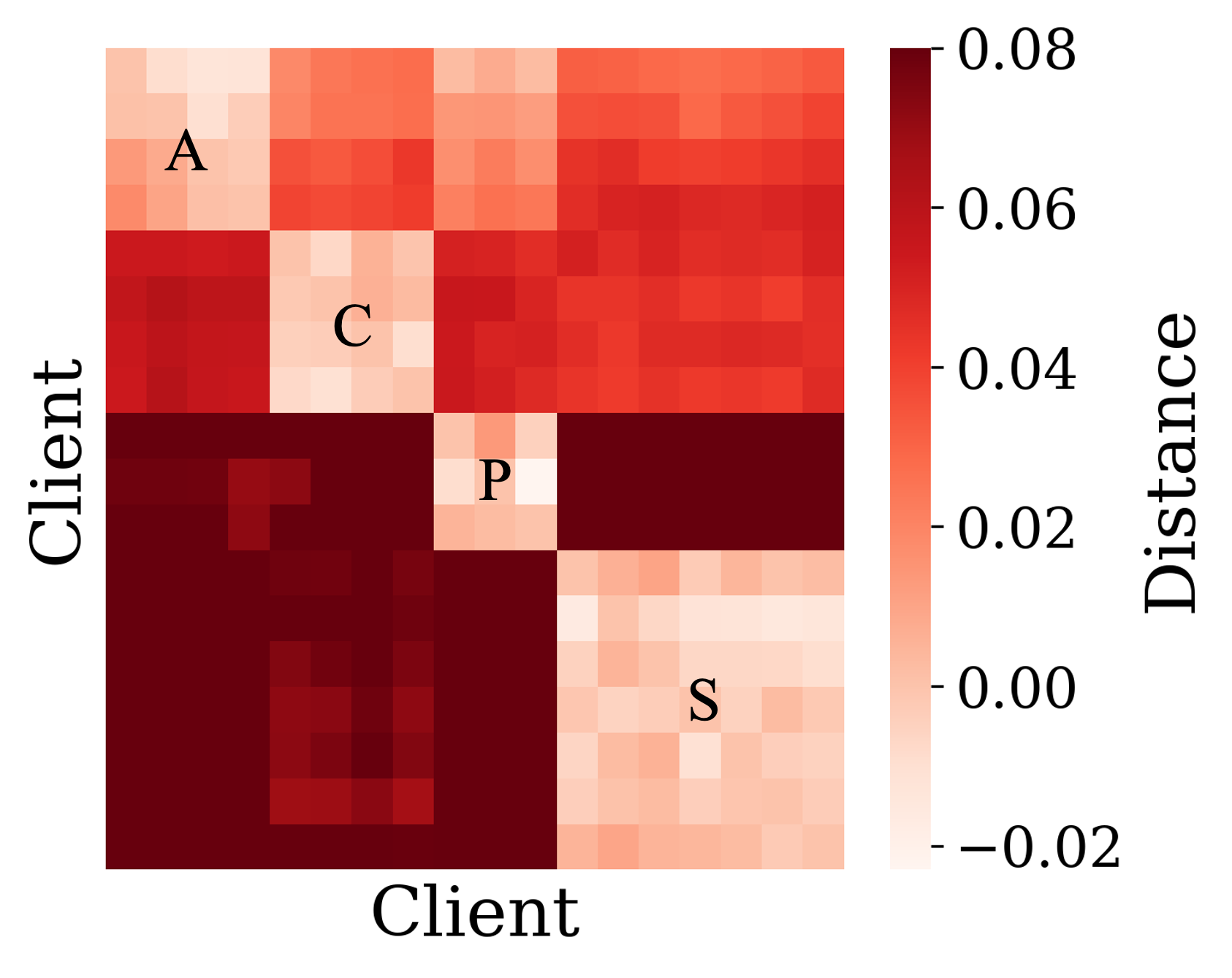}
    \caption{Pairwise EMDs.}
    \label{fig:pacs_distance}
  \end{minipage}\hfill
  \begin{minipage}{0.6\textwidth}
    \centering
    \resizebox{0.75\textwidth}{!}{
    \begin{tabular}{lccc}
    \toprule
    \textbf{Method} & \textbf{Avg Acc} & \textbf{Worst Acc} & \textbf{ARI} \\
    \midrule
    Oracle & 92.43$\pm$0.19 & 85.76$\pm$0.07 & 1.00$\pm$0.00 \\ \cmidrule(lr){1-4}
    \textbf{EMD-CFL } & \textbf{92.47$\pm$0.32} & \textbf{86.31$\pm$1.36} & \textbf{1.00$\pm$0.00} \\
    CFL & 86.03$\pm$0.79 & 79.17$\pm$1.86 & 0.00$\pm$0.00 \\
    PACFL & 88.43$\pm$0.52 & 45.31$\pm$9.18 & 0.70$\pm$0.00 \\
    \textbf{FedClust} & \textbf{87.63$\pm$0.36} & \textbf{75.74$\pm$1.99} & \textbf{1.00$\pm$0.00} \\
    IFCA & 90.39$\pm$2.07 & 82.47$\pm$5.52 & 0.69$\pm$0.22 \\
    \textbf{FlexCFL} & \textbf{92.36$\pm$0.27} & \textbf{85.42$\pm$1.03} & \textbf{1.00$\pm$0.00} \\
    FeSEM & 81.77$\pm$0.70 & 72.62$\pm$1.03 & 0.00$\pm$0.00 \\
    \textbf{CFLGP} & \textbf{92.07$\pm$0.21} & \textbf{84.28$\pm$1.30} & \textbf{1.00$\pm$0.00} \\
    \bottomrule
    \end{tabular}%
    }
    \captionof{table}{EMD-CFL matches the Oracle.}
    \label{tab:pacs}
  \end{minipage}
\end{figure}

\subsection{Backdoor Features}
We test which methods can segregate clients with visually subtle backdoor features, with which they aim to backdoor a set of target clients with similar but clean data. We evaluate the target clients' sensitivity to the backdoor by comparing the accuracies on clean data and on data with the backdoor feature manipulated to induce a misclassification. Targets which have been backdoored deteriorate in the latter scenario. 
On Backdoor MNIST (left-hand side of Table~\ref{tab:backdoor_combined}), several methods succeed by the end of training. FlexCFL, however, is not able to do so in the first epoch, resulting in unstable final accuracies, and both FlexCFL and CFLGP require specifying the correct number of clusters, which may be difficult when the backdoored data is visually similar to clean data. On Backdoor CIFAR10 (right-hand side of Table~\ref{tab:backdoor_combined}), only our method still succeeds. In Appendix~\ref{appx:backdoor}, we include results for soft clustering methods (which underperform) and show robustness to hyperparameter choices.

\begin{table}[htbp]
\centering
\caption{Several methods segregate clients with backdoor features on MNIST (bold). Only EMD-CFL successfully segregates clients with backdoor features on CIFAR10.}
\label{tab:backdoor_combined}
\resizebox{\textwidth}{!}{
\begin{tabular}{lccc@{\hspace{1.5em}}ccc}
\toprule
 & \multicolumn{3}{c}{\textbf{Backdoor MNIST}} & \multicolumn{3}{c}{\textbf{Backdoor CIFAR10}} \\ \cmidrule(lr){2-4} \cmidrule(lr){5-7}
\textbf{Method} & \textbf{Clean Acc} & \textbf{Backdoor Acc} & \textbf{ARI} & \textbf{Clean Acc} & \textbf{Backdoor Acc} & \textbf{ARI} \\
\midrule
Oracle & 98.37$\pm$0.11 & 98.27$\pm$0.16 & 1.00$\pm$0.00 & 97.89$\pm$0.12 & 97.63$\pm$0.06 & 1.00$\pm$0.00 \\ \cmidrule(lr){1-7}
EMD-CFL & \textbf{98.40$\pm$0.04} & \textbf{98.30$\pm$0.06} & \textbf{1.00$\pm$0.00} & \textbf{97.84$\pm$0.23} & \textbf{97.73$\pm$0.32} & \textbf{1.00$\pm$0.00} \\
CFL & 98.43$\pm$0.07 & 95.98$\pm$1.39 & 0.46$\pm$0.02 & 92.37$\pm$0.27 & 39.94$\pm$4.85 & 0.00$\pm$0.00 \\
PACFL & 95.55$\pm$0.14 & 81.55$\pm$1.45 & 0.42$\pm$0.00 & 97.71$\pm$0.11 & 84.57$\pm$3.13 & 0.53$\pm$0.00 \\
FedClust & 93.04$\pm$3.35 & 93.82$\pm$1.54 & 0.53$\pm$0.02 & 82.61$\pm$9.00 & 78.84$\pm$9.60 & 0.16$\pm$0.06 \\
IFCA & 98.54$\pm$0.05 & 98.15$\pm$0.14 & 0.55$\pm$0.00 & 97.62$\pm$0.14 & 77.69$\pm$4.06 & 0.53$\pm$0.00 \\
FlexCFL & \textbf{39.51$\pm$51.12} & \textbf{35.46$\pm$44.10} & \textbf{1.00$\pm$0.00} & 76.39$\pm$22.40 & 29.83$\pm$8.14 & 0.00$\pm$0.10 \\
FeSEM & 97.23$\pm$0.00 & 92.68$\pm$2.50 & -0.00$\pm$0.01 & 94.39$\pm$0.08 & 91.65$\pm$0.37 & 0.00$\pm$0.00 \\
CFLGP & \textbf{98.34$\pm$0.08} & \textbf{98.27$\pm$0.10} & \textbf{1.00$\pm$0.00} & 76.42$\pm$25.39 & 35.19$\pm$4.33 & 0.06$\pm$0.16 \\
\bottomrule
\end{tabular}
}
\end{table}

\section{Conclusion}
We propose EMD-CFL, a novel and theoretically motivated one-shot method for clustered FL. We empirically demonstrate that the use of EMDs consistently and robustly results in the best clustering performance in one-shot.

\section*{Acknowledgements}
This research was supported by the UKRI CDT in AI for Healthcare \url{https://ai4health.io/} (grant no. EP/S023283/1), by NIHR Imperial BRC (grant no. RDB01 79560), and by Brain Tumour Research Centre of Excellence at Imperial.

\bibliographystyle{plainnat}
\bibliography{refs}

\begin{thebibliography}{51}
\providecommand{\natexlab}[1]{#1}
\providecommand{\url}[1]{\texttt{#1}}
\expandafter\ifx\csname urlstyle\endcsname\relax
  \providecommand{\doi}[1]{doi: #1}\else
  \providecommand{\doi}{doi: \begingroup \urlstyle{rm}\Url}\fi

\bibitem[Bagdasaryan et~al.(2020)Bagdasaryan, Veit, Hua, Estrin, and Shmatikov]{bagdasaryan2020backdoor}
Eugene Bagdasaryan, Andreas Veit, Yiqing Hua, Deborah Estrin, and Vitaly Shmatikov.
\newblock How to backdoor federated learning.
\newblock In \emph{International conference on artificial intelligence and statistics}, pages 2938--2948. PMLR, 2020.

\bibitem[Ben-David et~al.(2006)Ben-David, Blitzer, Crammer, and Pereira]{ben2006analysis}
Shai Ben-David, John Blitzer, Koby Crammer, and Fernando Pereira.
\newblock Analysis of representations for domain adaptation.
\newblock \emph{Advances in neural information processing systems}, 19, 2006.

\bibitem[Ben-David et~al.(2010)Ben-David, Blitzer, Crammer, Kulesza, Pereira, and Vaughan]{ben2010theory}
Shai Ben-David, John Blitzer, Koby Crammer, Alex Kulesza, Fernando Pereira, and Jennifer~Wortman Vaughan.
\newblock A theory of learning from different domains.
\newblock \emph{Machine learning}, 79:\penalty0 151--175, 2010.

\bibitem[Bonawitz(2019)]{bonawitz2019towards}
Keith Bonawitz.
\newblock Towards federated learning at scale: System design.
\newblock \emph{arXiv preprint arXiv:1902.01046}, 2019.

\bibitem[Briggs et~al.(2020)Briggs, Fan, and Andras]{briggs2020federated}
Christopher Briggs, Zhong Fan, and Peter Andras.
\newblock Federated learning with hierarchical clustering of local updates to improve training on non-iid data.
\newblock In \emph{2020 international joint conference on neural networks (IJCNN)}, pages 1--9. IEEE, 2020.

\bibitem[Cai et~al.(2023)Cai, Chen, Cao, He, and Li]{cai2023fedce}
Luxin Cai, Naiyue Chen, Yuanzhouhan Cao, Jiahuan He, and Yidong Li.
\newblock Fedce: Personalized federated learning method based on clustering ensembles.
\newblock In \emph{Proceedings of the 31st ACM international conference on multimedia}, pages 1625--1633, 2023.

\bibitem[Cuturi(2013)]{cuturi2013sinkhorn}
Marco Cuturi.
\newblock Sinkhorn distances: Lightspeed computation of optimal transport.
\newblock \emph{Advances in neural information processing systems}, 26, 2013.

\bibitem[Dosovitskiy(2020)]{dosovitskiy2020image}
Alexey Dosovitskiy.
\newblock An image is worth 16x16 words: Transformers for image recognition at scale.
\newblock \emph{arXiv preprint arXiv:2010.11929}, 2020.

\bibitem[Duan et~al.(2021)Duan, Liu, Ji, Wu, Liang, Chen, Tan, and Ren]{duan2021flexible}
Moming Duan, Duo Liu, Xinyuan Ji, Yu~Wu, Liang Liang, Xianzhang Chen, Yujuan Tan, and Ao~Ren.
\newblock Flexible clustered federated learning for client-level data distribution shift.
\newblock \emph{IEEE Transactions on Parallel and Distributed Systems}, 33\penalty0 (11):\penalty0 2661--2674, 2021.

\bibitem[Dwork(2006)]{dwork2006differential}
Cynthia Dwork.
\newblock Differential privacy.
\newblock In \emph{International colloquium on automata, languages, and programming}, pages 1--12. Springer, 2006.

\bibitem[Fallah et~al.(2020)Fallah, Mokhtari, and Ozdaglar]{fallah2020personalized}
Alireza Fallah, Aryan Mokhtari, and Asuman Ozdaglar.
\newblock Personalized federated learning with theoretical guarantees: A model-agnostic meta-learning approach.
\newblock \emph{Advances in neural information processing systems}, 33:\penalty0 3557--3568, 2020.

\bibitem[Flamary et~al.(2021)Flamary, Courty, Gramfort, Alaya, Boisbunon, Chambon, Chapel, Corenflos, Fatras, Fournier, et~al.]{flamary2021pot}
R{\'e}mi Flamary, Nicolas Courty, Alexandre Gramfort, Mokhtar~Z Alaya, Aur{\'e}lie Boisbunon, Stanislas Chambon, Laetitia Chapel, Adrien Corenflos, Kilian Fatras, Nemo Fournier, et~al.
\newblock Pot: Python optimal transport.
\newblock \emph{Journal of Machine Learning Research}, 22\penalty0 (78):\penalty0 1--8, 2021.

\bibitem[Gan et~al.(2021)Gan, Mathur, Isopoussu, Kawsar, Berthouze, and Lane]{gan2021fruda}
Shaoduo Gan, Akhil Mathur, Anton Isopoussu, Fahim Kawsar, Nadia Berthouze, and Nicholas~D Lane.
\newblock Fruda: Framework for distributed adversarial domain adaptation.
\newblock \emph{IEEE Transactions on Parallel and Distributed Systems}, 33\penalty0 (11):\penalty0 3153--3164, 2021.

\bibitem[Ghosh et~al.(2020)Ghosh, Chung, Yin, and Ramchandran]{ghosh2020efficient}
Avishek Ghosh, Jichan Chung, Dong Yin, and Kannan Ramchandran.
\newblock An efficient framework for clustered federated learning.
\newblock \emph{Advances in Neural Information Processing Systems}, 33:\penalty0 19586--19597, 2020.

\bibitem[Gu et~al.(2019)Gu, Liu, Dolan-Gavitt, and Garg]{gu2019badnets}
Tianyu Gu, Kang Liu, Brendan Dolan-Gavitt, and Siddharth Garg.
\newblock Badnets: Evaluating backdooring attacks on deep neural networks.
\newblock \emph{IEEE Access}, 7:\penalty0 47230--47244, 2019.

\bibitem[Guo et~al.(2023)Guo, Tang, and Lin]{guo2023fedrc}
Yongxin Guo, Xiaoying Tang, and Tao Lin.
\newblock Fedrc: Tackling diverse distribution shifts challenge in federated learning by robust clustering.
\newblock \emph{arXiv preprint arXiv:2301.12379}, 2023.

\bibitem[He et~al.(2016)He, Zhang, Ren, and Sun]{he2016deep}
Kaiming He, Xiangyu Zhang, Shaoqing Ren, and Jian Sun.
\newblock Deep residual learning for image recognition.
\newblock In \emph{Proceedings of the IEEE conference on computer vision and pattern recognition}, pages 770--778, 2016.

\bibitem[Hubert and Arabie(1985)]{hubert1985comparing}
Lawrence Hubert and Phipps Arabie.
\newblock Comparing partitions.
\newblock \emph{Journal of classification}, 2:\penalty0 193--218, 1985.

\bibitem[Islam et~al.(2024)Islam, Javaherian, Xu, Yuan, Chen, and Tzeng]{islam2024fedclust}
Md~Sirajul Islam, Simin Javaherian, Fei Xu, Xu~Yuan, Li~Chen, and Nian-Feng Tzeng.
\newblock Fedclust: Optimizing federated learning on non-iid data through weight-driven client clustering.
\newblock \emph{arXiv preprint arXiv:2403.04144}, 2024.

\bibitem[Johnson(1984)]{johnson1984extensions}
William~B Johnson.
\newblock Extensions of lipshitz mapping into hilbert space.
\newblock In \emph{Conference modern analysis and probability, 1984}, pages 189--206, 1984.

\bibitem[Kairouz et~al.(2021)Kairouz, McMahan, Avent, Bellet, Bennis, Bhagoji, Bonawitz, Charles, Cormode, Cummings, et~al.]{kairouz2021advances}
Peter Kairouz, H~Brendan McMahan, Brendan Avent, Aur{\'e}lien Bellet, Mehdi Bennis, Arjun~Nitin Bhagoji, Kallista Bonawitz, Zachary Charles, Graham Cormode, Rachel Cummings, et~al.
\newblock Advances and open problems in federated learning.
\newblock \emph{Foundations and trends{\textregistered} in machine learning}, 14\penalty0 (1--2):\penalty0 1--210, 2021.

\bibitem[Kenthapadi et~al.(2012)Kenthapadi, Korolova, Mironov, and Mishra]{kenthapadi2012privacy}
Krishnaram Kenthapadi, Aleksandra Korolova, Ilya Mironov, and Nina Mishra.
\newblock Privacy via the johnson-lindenstrauss transform.
\newblock \emph{arXiv preprint arXiv:1204.2606}, 2012.

\bibitem[Kim et~al.(2024)Kim, Kim, and De~Veciana]{kim2024clustered}
Heasung Kim, Hyeji Kim, and Gustavo De~Veciana.
\newblock Clustered federated learning via gradient-based partitioning.
\newblock In Ruslan Salakhutdinov, Zico Kolter, Katherine Heller, Adrian Weller, Nuria Oliver, Jonathan Scarlett, and Felix Berkenkamp, editors, \emph{Proceedings of the 41st International Conference on Machine Learning}, volume 235 of \emph{Proceedings of Machine Learning Research}, pages 24137--24193. PMLR, 21--27 Jul 2024.

\bibitem[Kone{\v{c}}n{\`y} et~al.(2016)Kone{\v{c}}n{\`y}, McMahan, Ramage, and Richt{\'a}rik]{konevcny2016federated}
Jakub Kone{\v{c}}n{\`y}, H~Brendan McMahan, Daniel Ramage, and Peter Richt{\'a}rik.
\newblock Federated optimization: Distributed machine learning for on-device intelligence.
\newblock \emph{arXiv preprint arXiv:1610.02527}, 2016.

\bibitem[Krizhevsky et~al.(2009)Krizhevsky, Hinton, et~al.]{krizhevsky2009learning}
Alex Krizhevsky, Geoffrey Hinton, et~al.
\newblock Learning multiple layers of features from tiny images.
\newblock 2009.

\bibitem[LeCun(1998)]{lecun1998mnist}
Yann LeCun.
\newblock The mnist database of handwritten digits.
\newblock \emph{http://yann. lecun. com/exdb/mnist/}, 1998.

\bibitem[Li et~al.(2017)Li, Yang, Song, and Hospedales]{li2017deeper}
Da~Li, Yongxin Yang, Yi-Zhe Song, and Timothy~M Hospedales.
\newblock Deeper, broader and artier domain generalization.
\newblock In \emph{Proceedings of the IEEE international conference on computer vision}, pages 5542--5550, 2017.

\bibitem[Li et~al.(2020)Li, Sahu, Zaheer, Sanjabi, Talwalkar, and Smith]{li2020federated}
Tian Li, Anit~Kumar Sahu, Manzil Zaheer, Maziar Sanjabi, Ameet Talwalkar, and Virginia Smith.
\newblock Federated optimization in heterogeneous networks.
\newblock \emph{Proceedings of Machine learning and systems}, 2:\penalty0 429--450, 2020.

\bibitem[Lin et~al.(2020)Lin, Kong, Stich, and Jaggi]{lin2020ensemble}
Tao Lin, Lingjing Kong, Sebastian~U Stich, and Martin Jaggi.
\newblock Ensemble distillation for robust model fusion in federated learning.
\newblock \emph{Advances in neural information processing systems}, 33:\penalty0 2351--2363, 2020.

\bibitem[Liu et~al.(2022)Liu, Mao, Wu, Feichtenhofer, Darrell, and Xie]{liu2022convnet}
Zhuang Liu, Hanzi Mao, Chao-Yuan Wu, Christoph Feichtenhofer, Trevor Darrell, and Saining Xie.
\newblock A convnet for the 2020s.
\newblock In \emph{Proceedings of the IEEE/CVF conference on computer vision and pattern recognition}, pages 11976--11986, 2022.

\bibitem[Long et~al.(2023)Long, Xie, Shen, Zhou, Wang, and Jiang]{long2023multi}
Guodong Long, Ming Xie, Tao Shen, Tianyi Zhou, Xianzhi Wang, and Jing Jiang.
\newblock Multi-center federated learning: clients clustering for better personalization.
\newblock \emph{World Wide Web}, 26\penalty0 (1):\penalty0 481--500, 2023.

\bibitem[Mansour et~al.(2009)Mansour, Mohri, and Rostamizadeh]{mansour2009domain}
Yishay Mansour, Mehryar Mohri, and Afshin Rostamizadeh.
\newblock Domain adaptation: Learning bounds and algorithms.
\newblock \emph{arXiv preprint arXiv:0902.3430}, 2009.

\bibitem[Mansour et~al.(2020)Mansour, Mohri, Ro, and Suresh]{mansour2020three}
Yishay Mansour, Mehryar Mohri, Jae Ro, and Ananda~Theertha Suresh.
\newblock Three approaches for personalization with applications to federated learning.
\newblock \emph{arXiv preprint arXiv:2002.10619}, 2020.

\bibitem[Marfoq et~al.(2021)Marfoq, Neglia, Bellet, Kameni, and Vidal]{marfoq2021federated}
Othmane Marfoq, Giovanni Neglia, Aur{\'e}lien Bellet, Laetitia Kameni, and Richard Vidal.
\newblock Federated multi-task learning under a mixture of distributions.
\newblock \emph{Advances in Neural Information Processing Systems}, 34:\penalty0 15434--15447, 2021.

\bibitem[McMahan et~al.(2017)McMahan, Moore, Ramage, Hampson, and y~Arcas]{mcmahan2017communication}
Brendan McMahan, Eider Moore, Daniel Ramage, Seth Hampson, and Blaise~Aguera y~Arcas.
\newblock Communication-efficient learning of deep networks from decentralized data.
\newblock In \emph{Artificial intelligence and statistics}, pages 1273--1282. PMLR, 2017.

\bibitem[Peng et~al.(2019)Peng, Huang, Zhu, and Saenko]{peng2019federated}
Xingchao Peng, Zijun Huang, Yizhe Zhu, and Kate Saenko.
\newblock Federated adversarial domain adaptation.
\newblock \emph{arXiv preprint arXiv:1911.02054}, 2019.

\bibitem[Qin et~al.(2023)Qin, Yao, Chen, Li, Ding, and Cheng]{qin2023revisiting}
Zeyu Qin, Liuyi Yao, Daoyuan Chen, Yaliang Li, Bolin Ding, and Minhao Cheng.
\newblock Revisiting personalized federated learning: Robustness against backdoor attacks.
\newblock In \emph{Proceedings of the 29th ACM SIGKDD Conference on Knowledge Discovery and Data Mining}, pages 4743--4755, 2023.

\bibitem[Redko et~al.(2017)Redko, Habrard, and Sebban]{redko2017theoretical}
Ievgen Redko, Amaury Habrard, and Marc Sebban.
\newblock Theoretical analysis of domain adaptation with optimal transport.
\newblock In \emph{Machine Learning and Knowledge Discovery in Databases: European Conference, ECML PKDD 2017, Skopje, Macedonia, September 18--22, 2017, Proceedings, Part II 10}, pages 737--753. Springer, 2017.

\bibitem[Rieke et~al.(2020)Rieke, Hancox, Li, Milletari, Roth, Albarqouni, Bakas, Galtier, Landman, Maier-Hein, et~al.]{rieke2020future}
Nicola Rieke, Jonny Hancox, Wenqi Li, Fausto Milletari, Holger~R Roth, Shadi Albarqouni, Spyridon Bakas, Mathieu~N Galtier, Bennett~A Landman, Klaus Maier-Hein, et~al.
\newblock The future of digital health with federated learning.
\newblock \emph{NPJ digital medicine}, 3\penalty0 (1):\penalty0 1--7, 2020.

\bibitem[Ruan and Joe-Wong(2022)]{ruan2022fedsoft}
Yichen Ruan and Carlee Joe-Wong.
\newblock Fedsoft: Soft clustered federated learning with proximal local updating.
\newblock In \emph{Proceedings of the AAAI conference on artificial intelligence}, volume~36, pages 8124--8131, 2022.

\bibitem[Sattler et~al.(2020)Sattler, M{\"u}ller, and Samek]{sattler2020clustered}
Felix Sattler, Klaus-Robert M{\"u}ller, and Wojciech Samek.
\newblock Clustered federated learning: Model-agnostic distributed multitask optimization under privacy constraints.
\newblock \emph{IEEE transactions on neural networks and learning systems}, 32\penalty0 (8):\penalty0 3710--3722, 2020.

\bibitem[Shen et~al.(2018)Shen, Qu, Zhang, and Yu]{shen2018wasserstein}
Jian Shen, Yanru Qu, Weinan Zhang, and Yong Yu.
\newblock Wasserstein distance guided representation learning for domain adaptation.
\newblock In \emph{Proceedings of the AAAI conference on artificial intelligence}, volume~32, 2018.

\bibitem[Steinley(2004)]{steinley2004properties}
Douglas Steinley.
\newblock Properties of the hubert-arable adjusted rand index.
\newblock \emph{Psychological methods}, 9\penalty0 (3):\penalty0 386, 2004.

\bibitem[Szegedy(2013)]{szegedy2013intriguing}
C~Szegedy.
\newblock Intriguing properties of neural networks.
\newblock \emph{arXiv preprint arXiv:1312.6199}, 2013.

\bibitem[Tan and Le(2019)]{tan2019efficientnet}
Mingxing Tan and Quoc Le.
\newblock Efficientnet: Rethinking model scaling for convolutional neural networks.
\newblock In \emph{International conference on machine learning}, pages 6105--6114. PMLR, 2019.

\bibitem[Vahidian et~al.(2023)Vahidian, Morafah, Wang, Kungurtsev, Chen, Shah, and Lin]{vahidian2023efficient}
Saeed Vahidian, Mahdi Morafah, Weijia Wang, Vyacheslav Kungurtsev, Chen Chen, Mubarak Shah, and Bill Lin.
\newblock Efficient distribution similarity identification in clustered federated learning via principal angles between client data subspaces.
\newblock In \emph{Proceedings of the AAAI conference on artificial intelligence}, volume~37, pages 10043--10052, 2023.

\bibitem[Virmaux and Scaman(2018)]{virmaux2018lipschitz}
Aladin Virmaux and Kevin Scaman.
\newblock Lipschitz regularity of deep neural networks: analysis and efficient estimation.
\newblock \emph{Advances in Neural Information Processing Systems}, 31, 2018.

\bibitem[Yan et~al.(2024)Yan, Cui, Wuerkaixi, Zhang, Han, Niu, Sugiyama, and Zhang]{yan2024balancing}
Kunda Yan, Sen Cui, Abudukelimu Wuerkaixi, Jingfeng Zhang, Bo~Han, Gang Niu, Masashi Sugiyama, and Changshui Zhang.
\newblock Balancing similarity and complementarity for federated learning.
\newblock \emph{arXiv preprint arXiv:2405.09892}, 2024.

\bibitem[Ye et~al.(2023)Ye, Ni, Wu, Chen, and Wang]{ye2023personalized}
Rui Ye, Zhenyang Ni, Fangzhao Wu, Siheng Chen, and Yanfeng Wang.
\newblock Personalized federated learning with inferred collaboration graphs.
\newblock In \emph{International Conference on Machine Learning}, pages 39801--39817. PMLR, 2023.

\bibitem[Zhang and Gao(2022)]{zhang2022transfer}
Lei Zhang and Xinbo Gao.
\newblock Transfer adaptation learning: A decade survey.
\newblock \emph{IEEE Transactions on Neural Networks and Learning Systems}, 2022.

\bibitem[Zhu et~al.(2021)Zhu, Hong, and Zhou]{zhu2021data}
Zhuangdi Zhu, Junyuan Hong, and Jiayu Zhou.
\newblock Data-free knowledge distillation for heterogeneous federated learning.
\newblock In \emph{International conference on machine learning}, pages 12878--12889. PMLR, 2021.

\end{thebibliography}

\appendix
\section{Theory}

\subsection{Proofs}
\label{appx:proofs}

\begin{theorem}
\label{appx:thm:mixture-ub}
    Under the assumptions from Lemma~\ref{lem:shen}, we have
    \begin{equation*}
        \epsilon_T(h) \leq \alpha \epsilon_j(h) + (1-\alpha) \epsilon_i(h) + 2LW_1(\mu^\mathcal{Z}_i, \mu^\mathcal{Z}_j) + \lambda
    \end{equation*}
    where $\lambda \equiv min_{h \in H} \epsilon_i(h) + \epsilon_j(h)$ is the sum error of the ideal hypothesis.  
\end{theorem}

\begin{proof}
    \begin{equation*}
    \begin{split}
        \epsilon_T(h) &= \alpha \epsilon_i(h) + (1-\alpha) \epsilon_j(h) \\
                    &\leq \alpha \left[ \epsilon_i(h,h^*) + \epsilon_i(h^*)\right] + (1 - \alpha) \left[ \epsilon_j(h,h^*) + \epsilon_j(h^*)\right] \\ 
                    &\text{(triangle inequality)} \\ 
                    &= \alpha \left[ \epsilon_i(h,h^*) - \epsilon_j(h,h^*) + \epsilon_j(h,h^*) + \epsilon_i(h^*)\right] + (1 - \alpha) \left[ \epsilon_j(h,h^*) - \epsilon_i(h,h^*) + \epsilon_i(h,h^*) + \epsilon_j(h^*)\right] \\
                    &\leq \alpha \left[ 2LW_1(\mu^\mathcal{Z}_i, \mu^\mathcal{Z}_j) + \epsilon_j(h,h^*) + \epsilon_i(h^*)\right] + (1 - \alpha) \left[ 2LW_1(\mu^\mathcal{Z}_i, \mu^\mathcal{Z}_j) + \epsilon_i(h,h^*) + \epsilon_j(h^*)\right] \\
                    &\text{(Lemma~\ref{lem:shen})} \\ 
                    &= 2LW_1(\mu^\mathcal{Z}_i, \mu^\mathcal{Z}_j) + \alpha \epsilon_j(h,h^*) + (1-\alpha) \epsilon_i(h,h^*)  + \left[\alpha \epsilon_i(h^*) + (1-\alpha) \epsilon_j(h^*) \right] \\
                    &\leq 2LW_1(\mu^\mathcal{Z}_i, \mu^\mathcal{Z}_j) + \alpha \left[\epsilon_j(h) + \epsilon_j(h^*)\right] + (1-\alpha) \left[ \epsilon_i(h) + \epsilon_i(h^*) \right] + \left[\alpha \epsilon_i(h^*) + (1-\alpha) \epsilon_j(h^*) \right] \\
                    &\text{(triangle inequality)} \\ 
                    &= 2LW_1(\mu^\mathcal{Z}_i, \mu^\mathcal{Z}_j) + \lambda + \alpha \epsilon_j(h) + (1-\alpha) \epsilon_i(h)   
    \end{split}
    \end{equation*}
\end{proof}

\begin{corollary}
    Under the assumptions of Lemma~\ref{lem:shen}, let $h_i$ and $h_j$ be hypotheses learned on $\mu^\mathcal{Z}_i$ and $\mu^\mathcal{Z}_j$ respectively, so that the ensemble hypothesis is $h_{e} = \alpha h_i + (1-\alpha) h_j$. We then have
    \begin{equation*}
    \begin{split}
        \epsilon_T(h_e) \leq &\ \alpha \left[\alpha \epsilon_j(h_i) + (1-\alpha) \epsilon_i(h_i) \right] \\
                            &+ (1-\alpha) \left[\alpha \epsilon_j(h_j) + (1-\alpha) \epsilon_i(h_j)\right] \\
                            &+ 2LW_1(\mu^\mathcal{Z}_i, \mu^\mathcal{Z}_j) + \lambda
    \end{split}
    \end{equation*}
\end{corollary}

\begin{proof}
We apply Theorem~\ref{appx:thm:mixture-ub} after the inequality below and the result follows.
\begin{equation*}
    \begin{split}
        \epsilon_T(h_e) &= \epsilon_T(\alpha h_i + (1-\alpha) h_j) \\
                        &\leq \alpha \epsilon_T(h_i) + (1-\alpha) \epsilon_T(h_j) \\
                        &\text{(convex error)}
    \end{split}
\end{equation*}
\end{proof}

\begin{corollary}
    Under the assumptions of Theorem~\ref{thm:gradbound} and for some learning rate $\kappa$ we have for a given round $t$
    \begin{equation*}
        \mathbb{E}[\lVert \phi_i^{t+1} - \phi_j^{t+1} \rVert] \leq \frac{M}{\kappa}W_1(\mu^\mathcal{Z}_i, \mu^\mathcal{Z}_j)
    \end{equation*}
\end{corollary}

\begin{proof}
The expected parameters of client $i$ in round $t+1$ can be written as $\mathbb{E}[\phi_i^{t+1}] = \phi_i^{t} - \kappa \mathbb{E}[\hat{\mathcal{L}}_i(z;\phi)]$ where $\kappa$ denotes a learning rate. We apply Theorem~\ref{thm:gradbound} and the result follows.
\end{proof}

\subsection{Johnson-Lindenstrauss Lemma}
\label{appx:JL}
\begin{lemma} \citep{johnson1984extensions}
    Given a set of $N$ points $A \subset \mathbb{R}^{d}$, $\epsilon_{JL} \in (0,1)$ and $s = \Omega(\frac{\log{N}}{\lambda^2_{JL}})$, there is a $d \times s$ projection matrix $P$ such that for any $a,b \in A$ with probability $1-\delta_{JL}$ where $\log{(1/\delta_{JL})} = O(s\lambda_{JL})$:
    \begin{equation*}
        (1-\epsilon_{JL}) \lVert a - b \rVert^2 \leq \lVert (a - b)P \rVert^2 \leq (1+ \epsilon_{JL}) \lVert a - b \rVert^2
    \end{equation*}
\end{lemma}

\newpage
\section{Related Works}
\label{appx:related_works}
\subsection{Federated Learning}
Federated learning (FL) \citep{konevcny2016federated, mcmahan2017communication} is a key learning framework for training machine learning models in a decentralised setting with multiple clients owning local datasets. A common feature of these decentralised settings is that the raw local data of the clients cannot be collected in a central place, such as when privacy concerns restrict its sharing. Since traditional centralised learning often requires large volumes of data, this poses a problem which FL addresses. FL is of particular interest for naturally decentralised domains such as mobile devices or healthcare \citep{bonawitz2019towards, rieke2020future}. Standard FL algorithms, such as FedAvg \citep{mcmahan2017communication}, typically make the assumption that clients have IID data, which often does not hold in practice, and results in suboptimal models \citep{li2020federated}. As an example in the context of healthcare this heterogeneity naturally arises, as different populations live in different places and are thus served by different hospitals, which will therefore own differently distributed data. Dealing with this heterogeneity amongst clients is important in order to personalise the output of the machine learning models for the client at hand while still maximally exploiting the insights that can be gained from other clients.

\subsection{Clustered Federated Learning}
Clustered FL \citep{sattler2020clustered, ghosh2020efficient} was proposed to address the issue of clients with non-IID data. This approach is motivated by the idea that the heterogeneity amongst clients is in fact characterised by a clustering structure, so that there are natural clusters of clients which are internally more homogeneous. Identifying these clusters can therefore recover the IID assumption on the cluster-level, allowing optimal models to be trained for each cluster. Existing methods in this area differ primarily in the approach with which they identify these clusters. A dominant approach relies on the distance between model parameters or parameter gradients \citep{sattler2020clustered, briggs2020federated, duan2021flexible, long2023multi, islam2024fedclust, kim2024clustered}. Another common approach is based on the local loss of cluster models \citep{ghosh2020efficient,marfoq2021federated,ruan2022fedsoft,guo2023fedrc,cai2023fedce}. \citet{vahidian2023efficient} is closest in spirit to our approach and also focuses on comparing differences in data directly, but does so on the raw input level. In addition to the clustering identification approach, another point of distinction is if the clustering is soft or hard. Soft clustering effectively allows clients to participate in training multiple cluster models \citep{marfoq2021federated, ruan2022fedsoft, guo2023fedrc,cai2023fedce}, while hard clustering forces each client to belong to only one cluster. Personalisation methods such as pFedGraph and FedSaC \citep{ye2023personalized, yan2024balancing} could be seen as relatives of the soft clustering approach, where each client effectively maintains its own cluster model to which other clients contribute. Our method takes the hard clustering approach, which we show is beneficial for segregating certain client clusters such as in the backdoored data setting.

\subsection{Domain Adaptation}
%move to related works
Domain adaptation is typically concerned with the setting in which training and test domains experience a shift and therefore violate the IID assumption usually made by empirical risk minimising algorithms \citep{zhang2022transfer}. When the IID assumption is violated, a learner that does not take into account such a shift is expected to have a greater error on the test domain. Theoretical work has aimed to place an upper bound on this error. Approaches for deriving the upper bound use different measures of distribution divergence, such as the Wasserstein distance \citep{redko2017theoretical, shen2018wasserstein} on which we focus. Other approaches use other divergences, such as H-divergence \citep{ben2006analysis, ben2010theory} and discrepancy \citep{mansour2009domain}, but the former typically requires separate classifiers for each pair of distributions while the latter is algorithmically more expensive \citep{redko2017theoretical}. There are works that aim to achieve domain adaptation in a federated setting \citep{peng2019federated, gan2021fruda}. While we draw on results from the domain adaptation literature and evaluate our method on some datasets also used in domain adaptation, our method fundamentally differs in approach. We do not aim to adapt from some source to some target domain but instead identify different domains to then train domain-specific models.

\newpage
\section{Examples of Data}
\label{appx:data_examples}
\subsection{Rotated MNIST}
\label{appx:mnist_examples}

\begin{figure}[hbtp]
    \centering

    % First subfigure
    \begin{subfigure}[b]{0.24\textwidth}
        \centering
        \includegraphics[width=\linewidth]{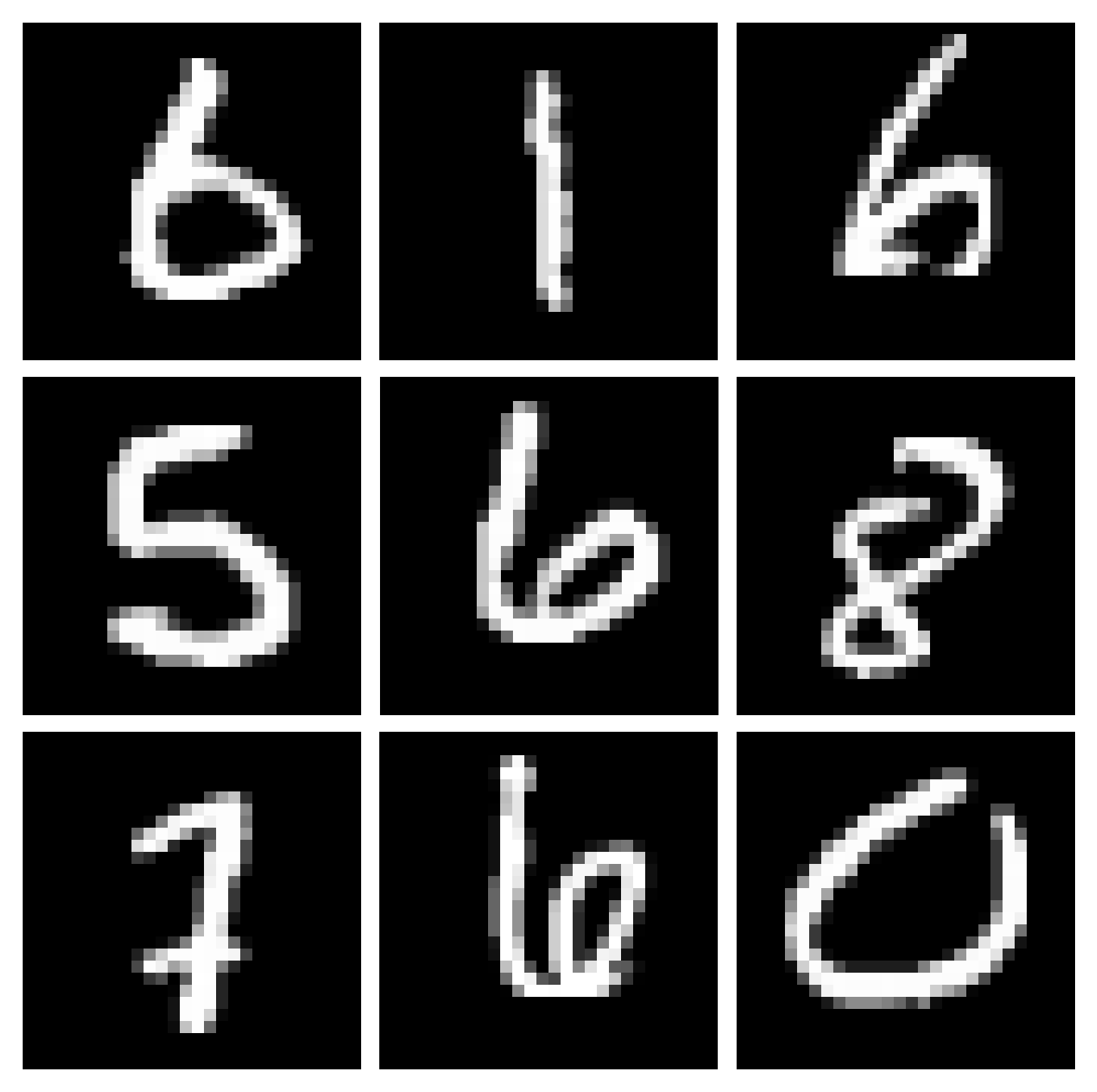}
        \caption{$0^\circ$}
    \end{subfigure}
    % Second subfigure
    \begin{subfigure}[b]{0.24\textwidth}
        \centering
        \includegraphics[width=\linewidth]{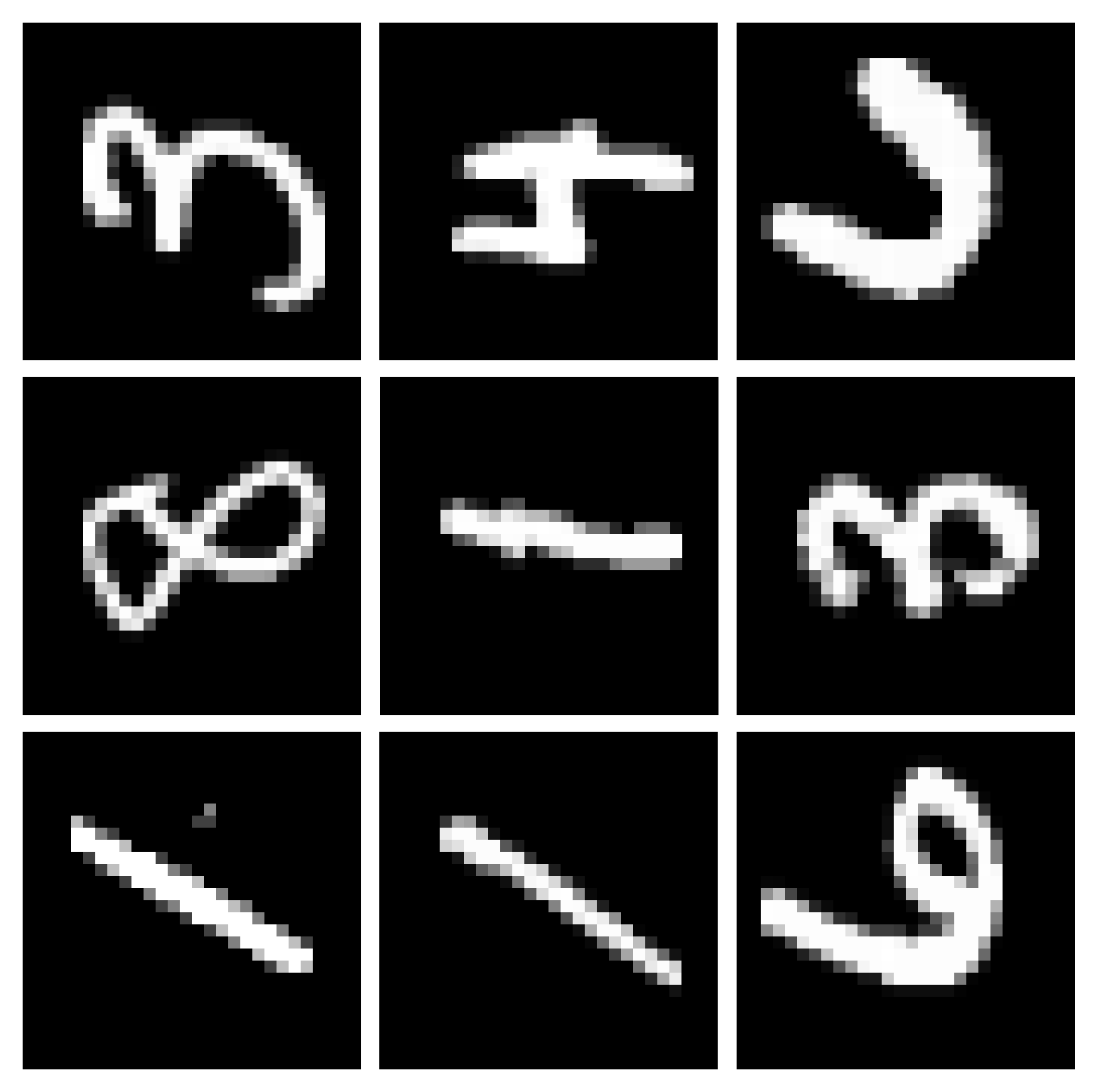}
        \caption{$90^\circ$}
    \end{subfigure}
    % Third subfigure
    \begin{subfigure}[b]{0.24\textwidth}
        \centering
        \includegraphics[width=\linewidth]{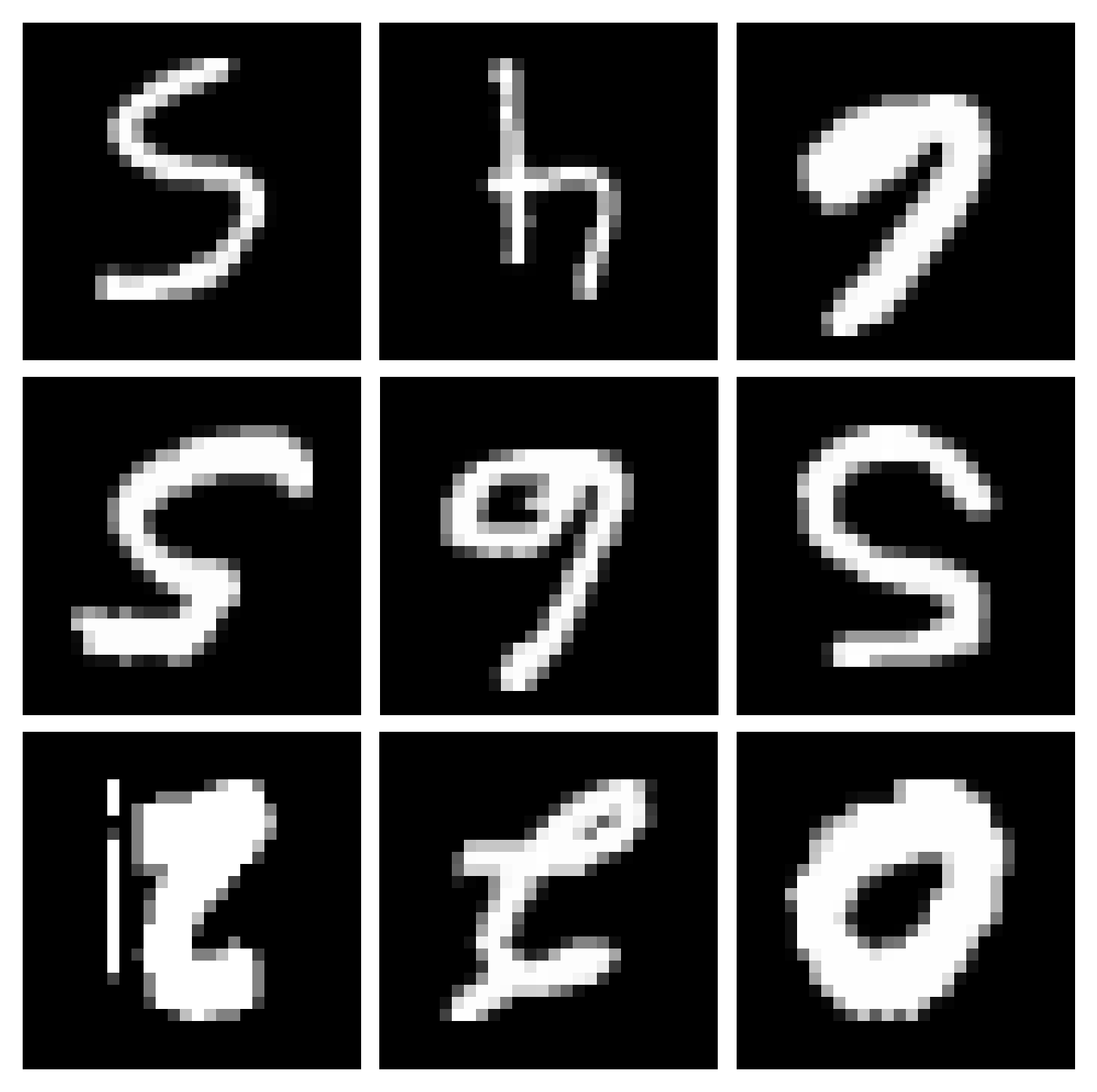}
        \caption{$180^\circ$}
    \end{subfigure}
    % Fourth subfigure
    \begin{subfigure}[b]{0.24\textwidth}
        \centering
        \includegraphics[width=\linewidth]{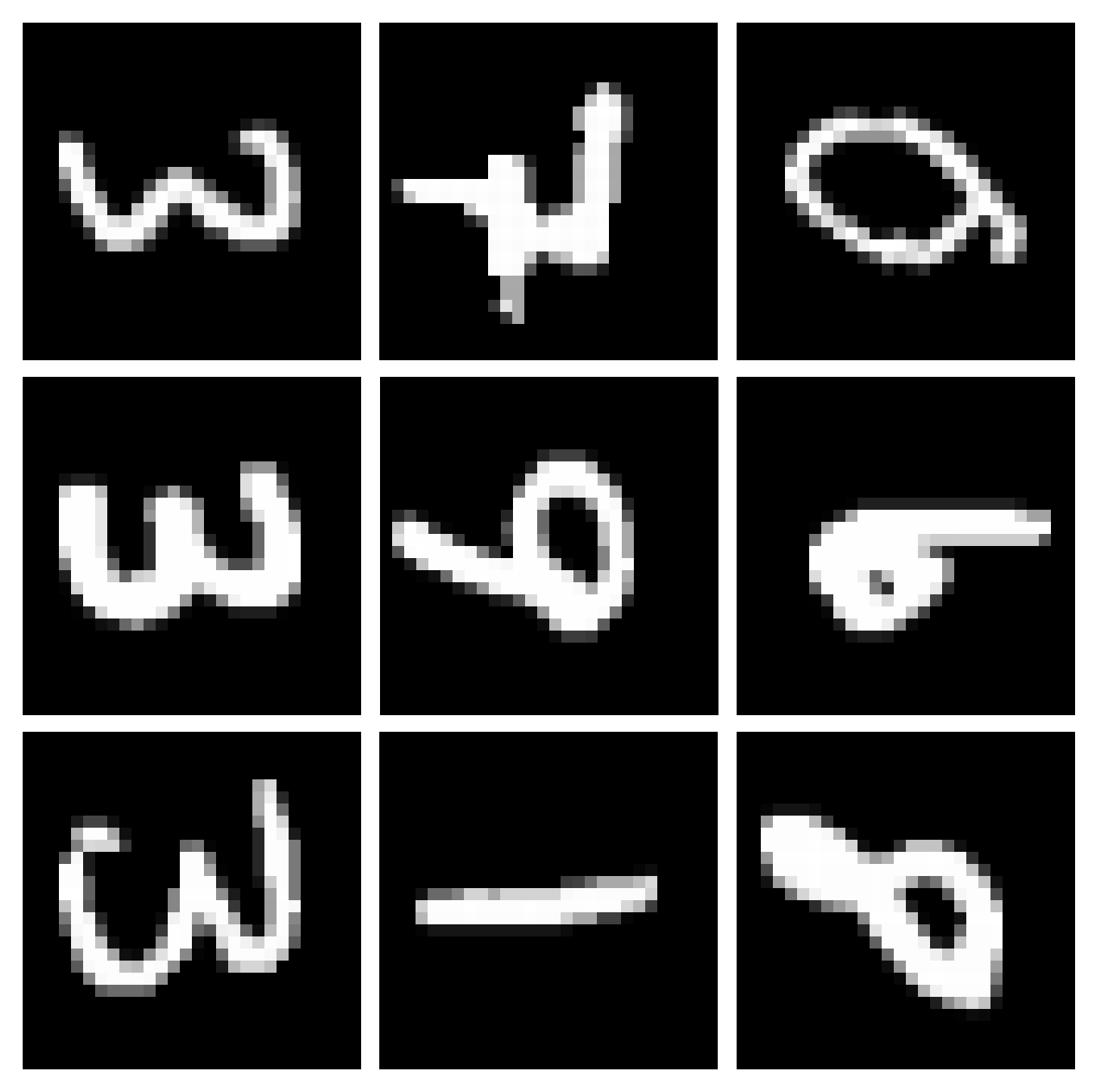}
        \caption{$270^\circ$}
    \end{subfigure}

    \caption{Examples from the four rotations of Rotated MNIST.}
    \label{fig:mnist_examples}
\end{figure}

\subsection{Rotated CIFAR10}
\label{appx:cifar_examples}

\begin{figure}[hbtp]
    \centering

    % First subfigure
    \begin{subfigure}[b]{0.24\textwidth}
        \centering
        \includegraphics[width=\linewidth]{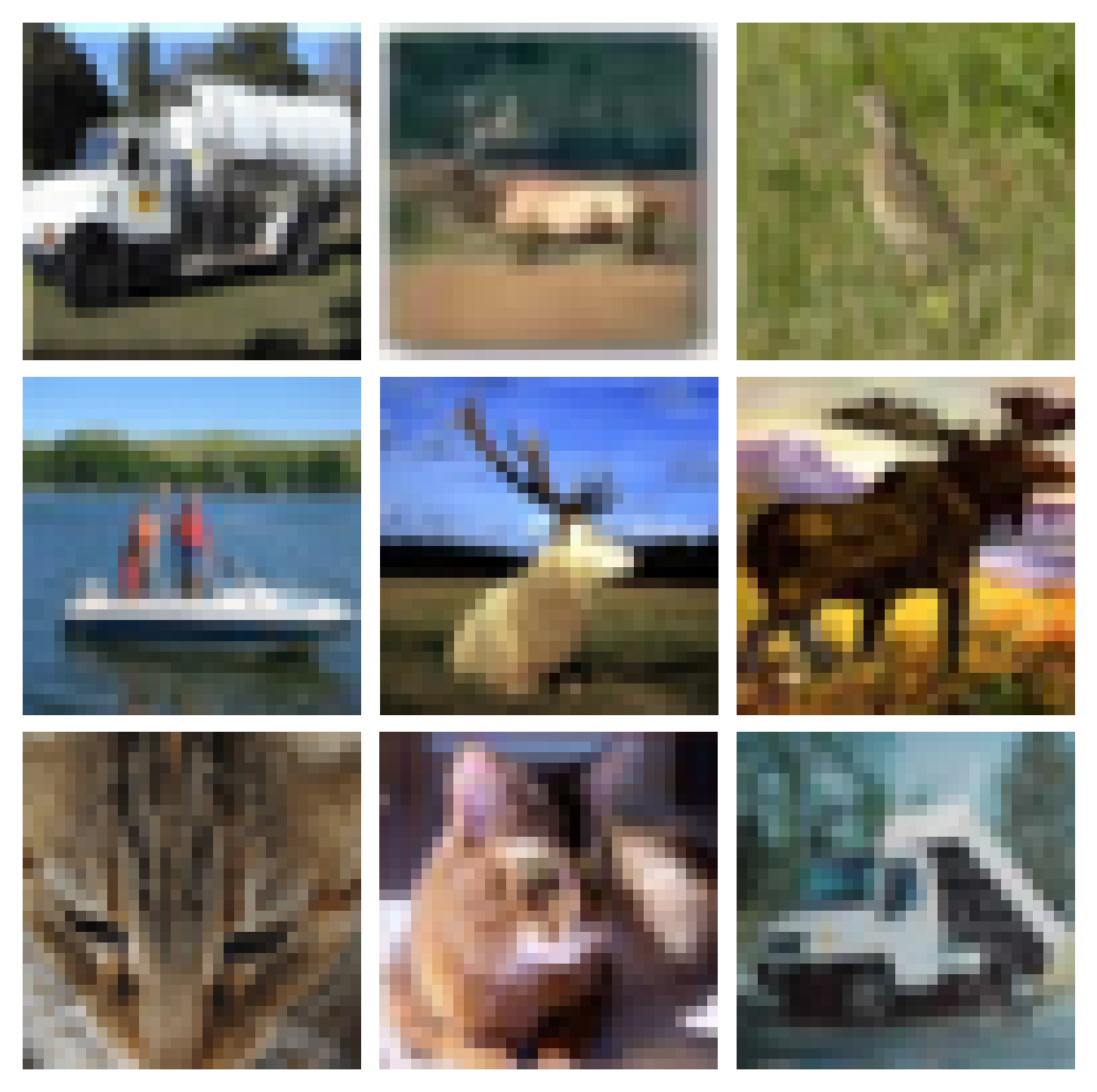}
        \caption{$0^\circ$}
    \end{subfigure}
    % Second subfigure
    \begin{subfigure}[b]{0.24\textwidth}
        \centering
        \includegraphics[width=\linewidth]{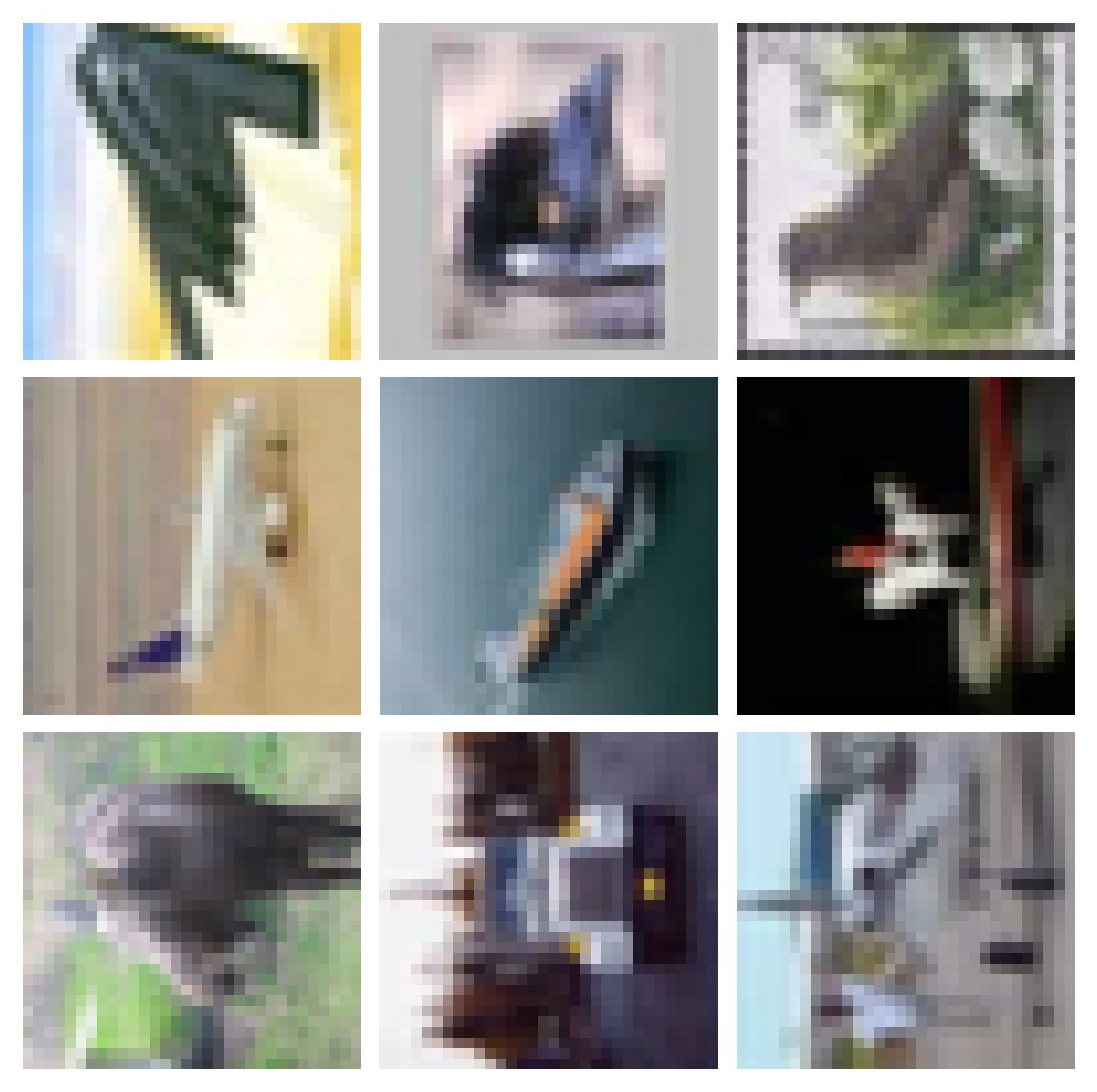}
        \caption{$90^\circ$}
    \end{subfigure}
    % Third subfigure
    \begin{subfigure}[b]{0.24\textwidth}
        \centering
        \includegraphics[width=\linewidth]{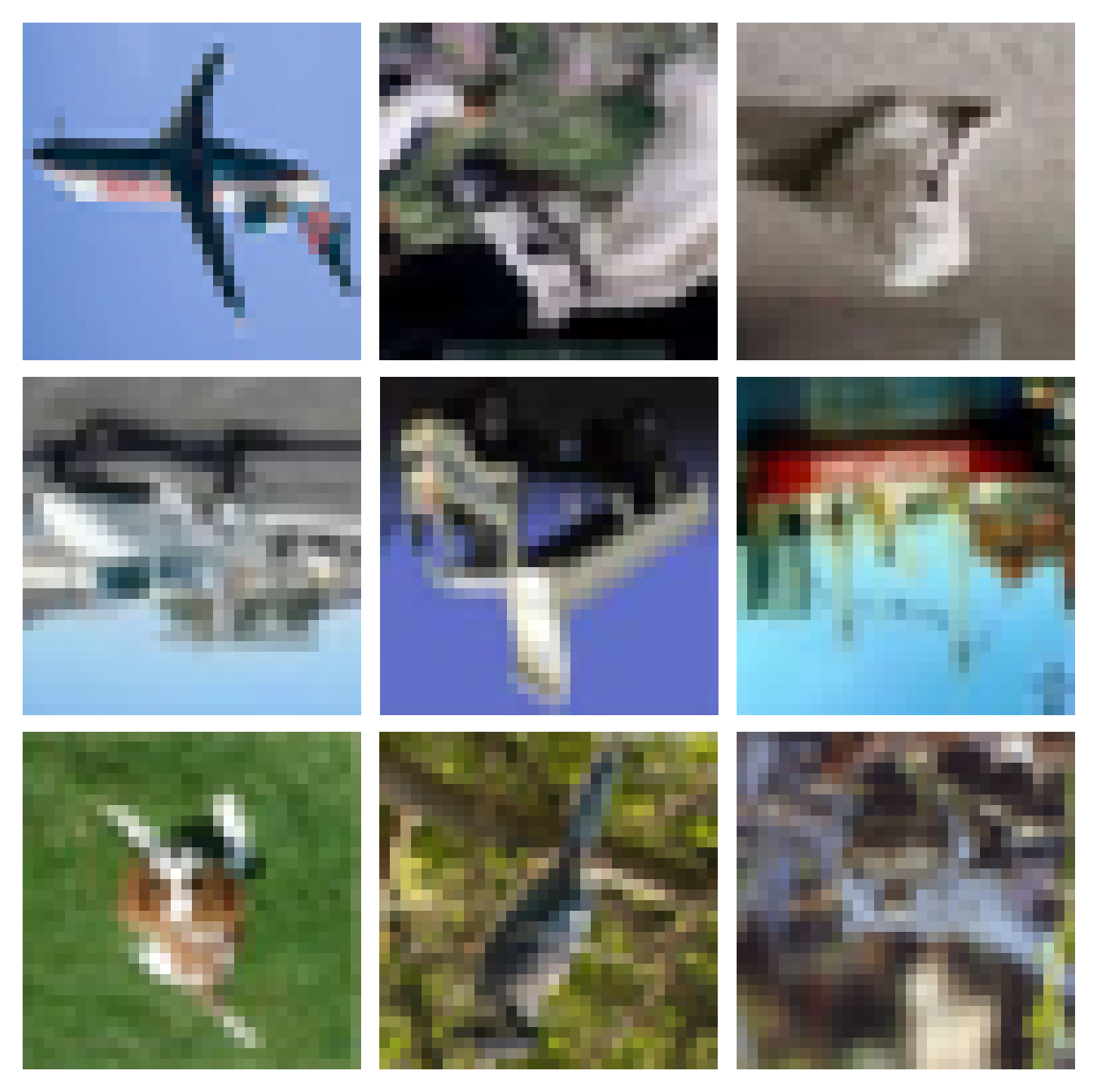}
        \caption{$180^\circ$}
    \end{subfigure}
    % Fourth subfigure
    \begin{subfigure}[b]{0.24\textwidth}
        \centering
        \includegraphics[width=\linewidth]{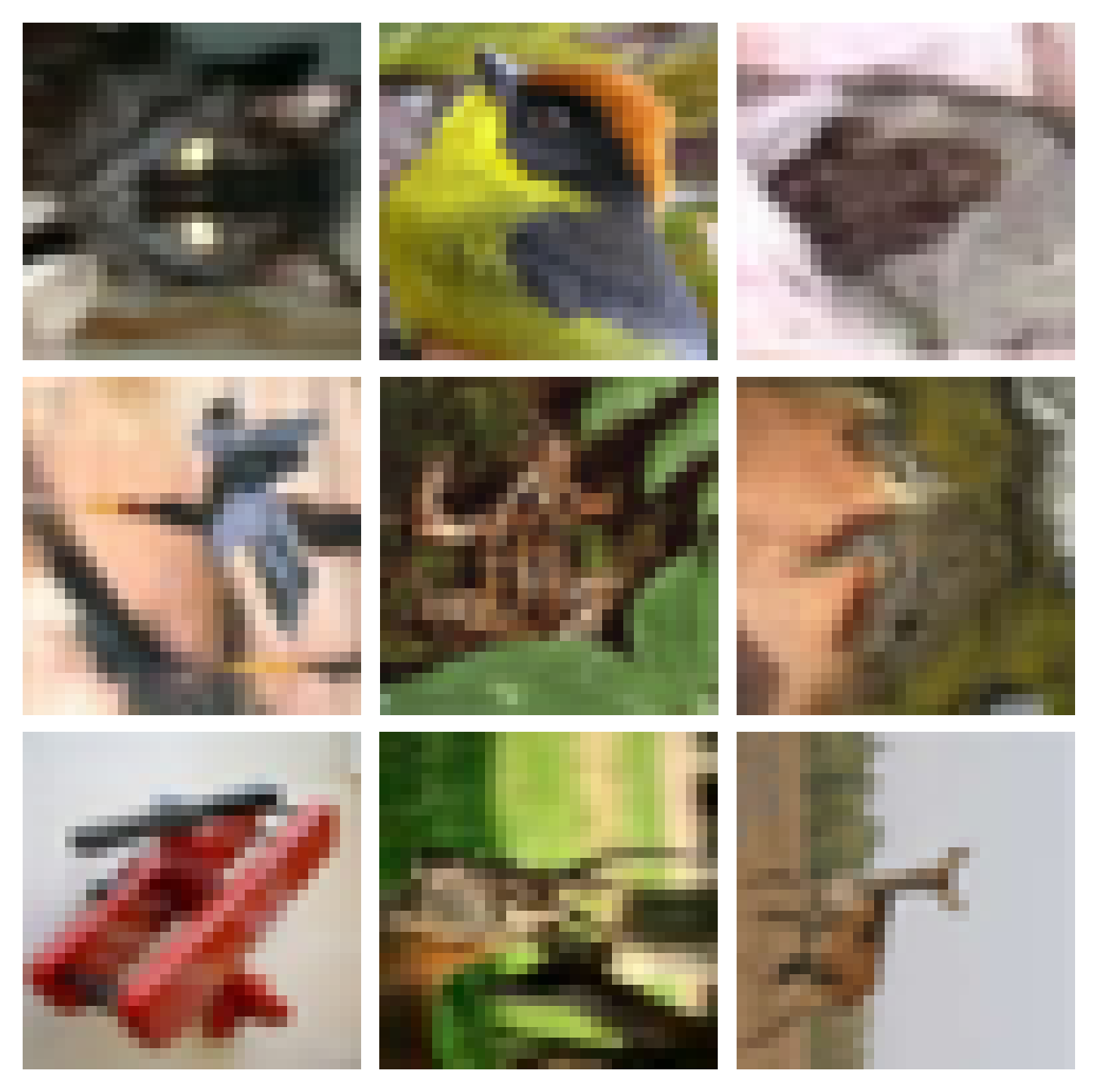}
        \caption{$270^\circ$}
    \end{subfigure}

    \caption{Examples from the four rotations of Rotated CIFAR10.}
    \label{fig:cifar_examples}
\end{figure}

\subsection{PACS}
\label{appx:pacs_examples}

\begin{figure}[!hbtp]
    \centering

    \begin{subfigure}[b]{0.24\textwidth}
        \centering
        \includegraphics[width=\linewidth]{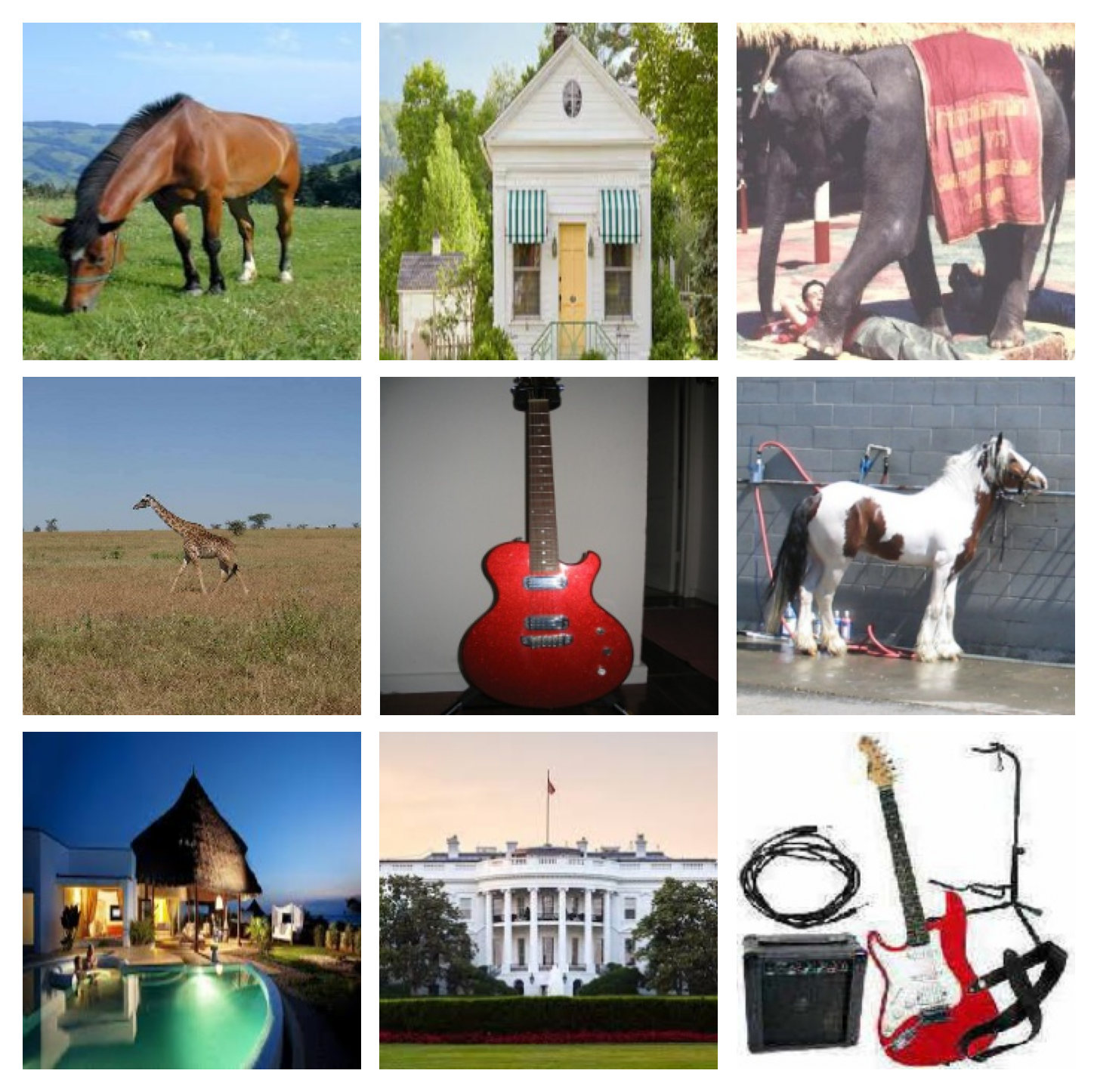}
        \caption{Photos}
    \end{subfigure}
    % First subfigure
    \begin{subfigure}[b]{0.24\textwidth}
        \centering
        \includegraphics[width=\linewidth]{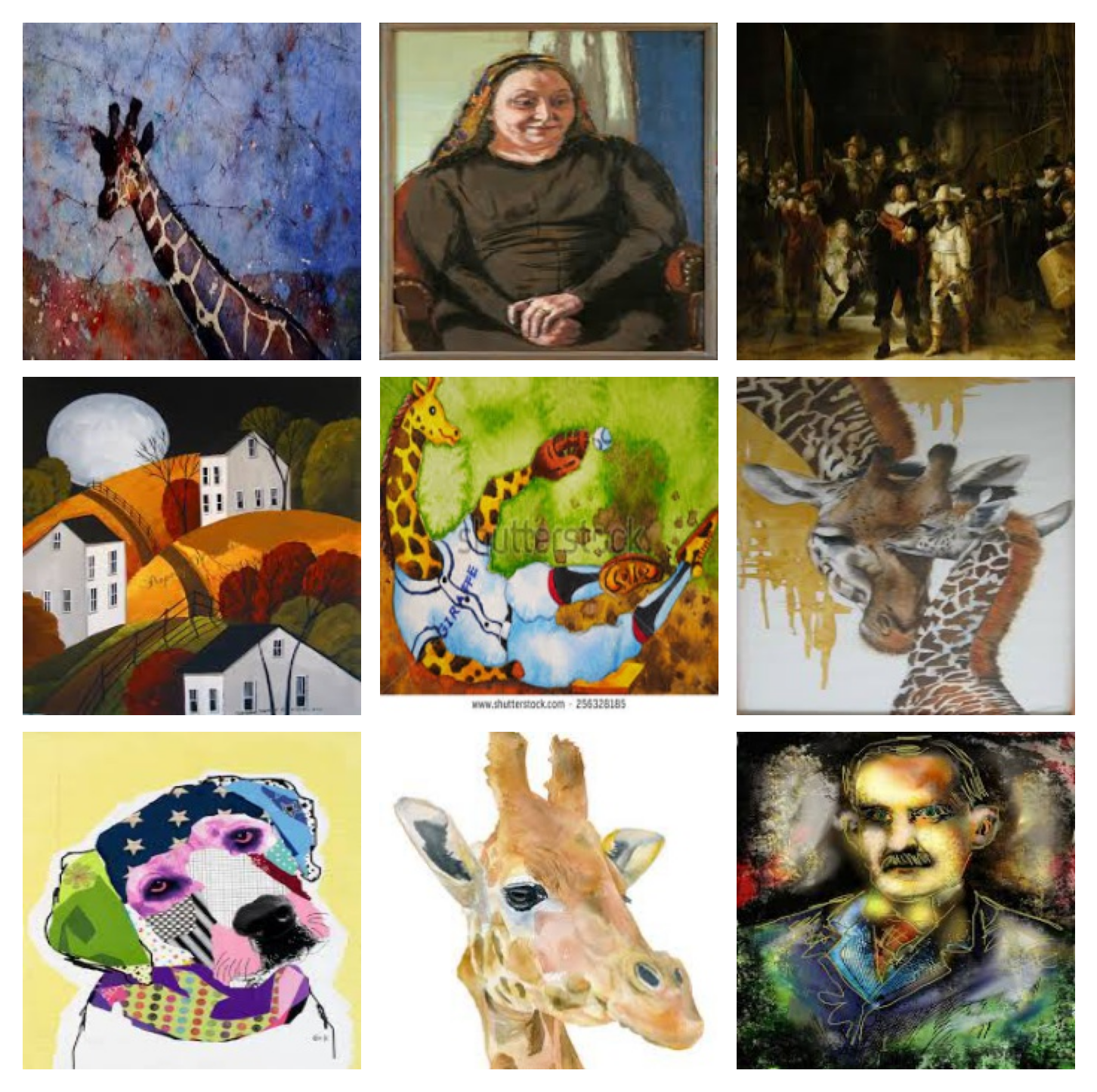}
        \caption{Art}
    \end{subfigure}
    % Second subfigure
    \begin{subfigure}[b]{0.24\textwidth}
        \centering
        \includegraphics[width=\linewidth]{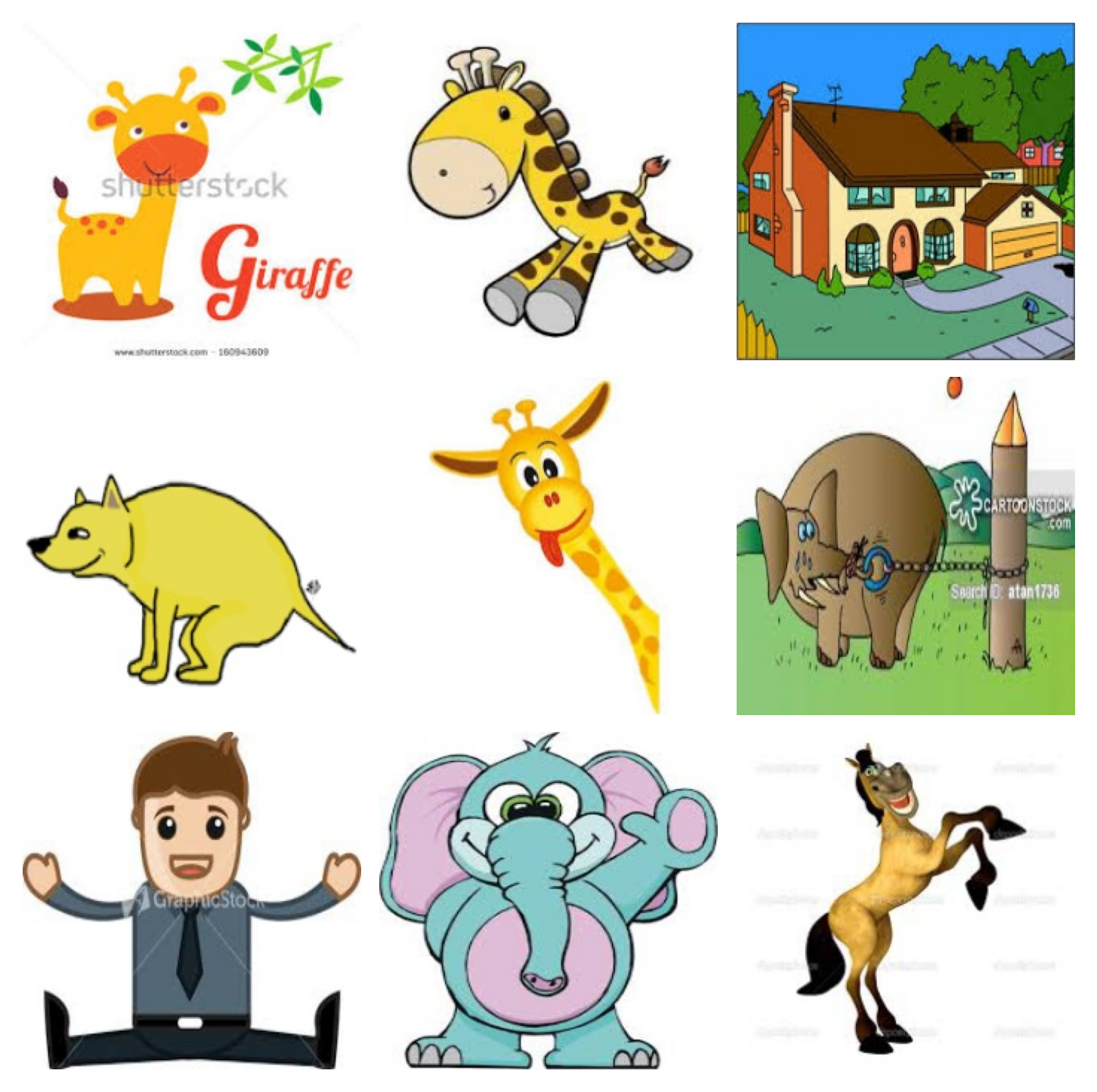}
        \caption{Cartoons}
    \end{subfigure}
    % Third subfigure
    \begin{subfigure}[b]{0.24\textwidth}
        \centering
        \includegraphics[width=\linewidth]{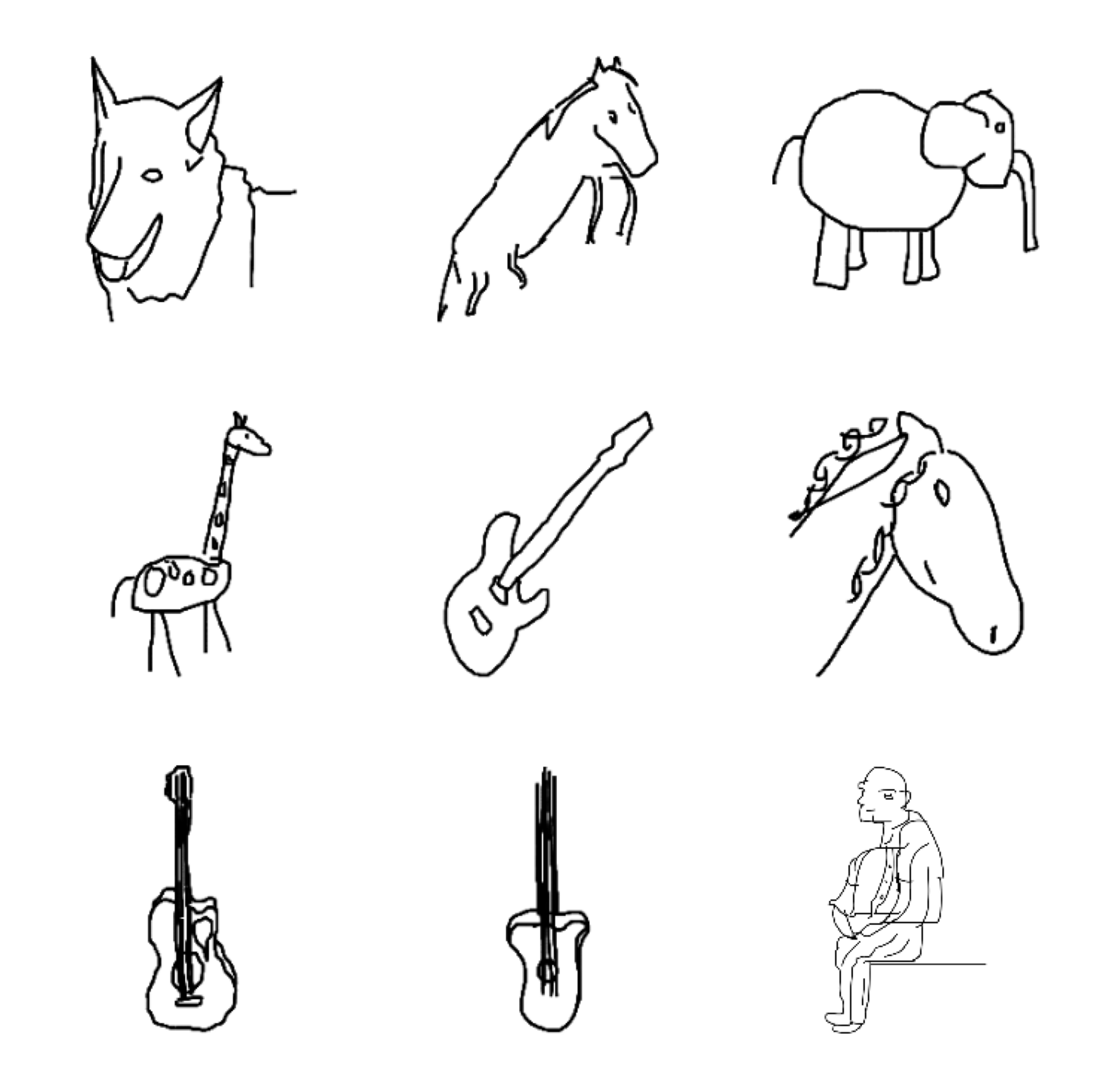}
        \caption{Sketches}
    \end{subfigure}

    \caption{Examples from the four domains of PACS.}
    \label{fig:pacs_examples}
\end{figure}

\subsection{Backdoor MNIST}
\label{appx:cmnist_examples}

Figure~\ref{fig:cmnist} shows uniformly green digits (left), similarly green digits (centre), where the intensity of the green is subtly correlated with the digit, and uniformly purple data (right). There is a visually obvious difference between the green and purple datasets, and there is a subtle difference between the two green datasets. The backdoor feature is the intensity of the greenness, which has a spurious correlation with the label that a model trained on this data might learn. If the backdoor feature is subsequently manipulated, here by varying the greenness, then a model that has learned the backdoor feature could be induced to misclassify. We therefore aim to segregate clients holding backdoored data from clients holding clean data. In our experiments, we test our method's ability to automatically segregate the clients, even when the differences between datasets appear visually subtle. 

\begin{figure}[!hbtp]
    \centering
    % First subfigure
    \begin{subfigure}[b]{0.3\columnwidth}
        \centering
        \includegraphics[width=\linewidth]{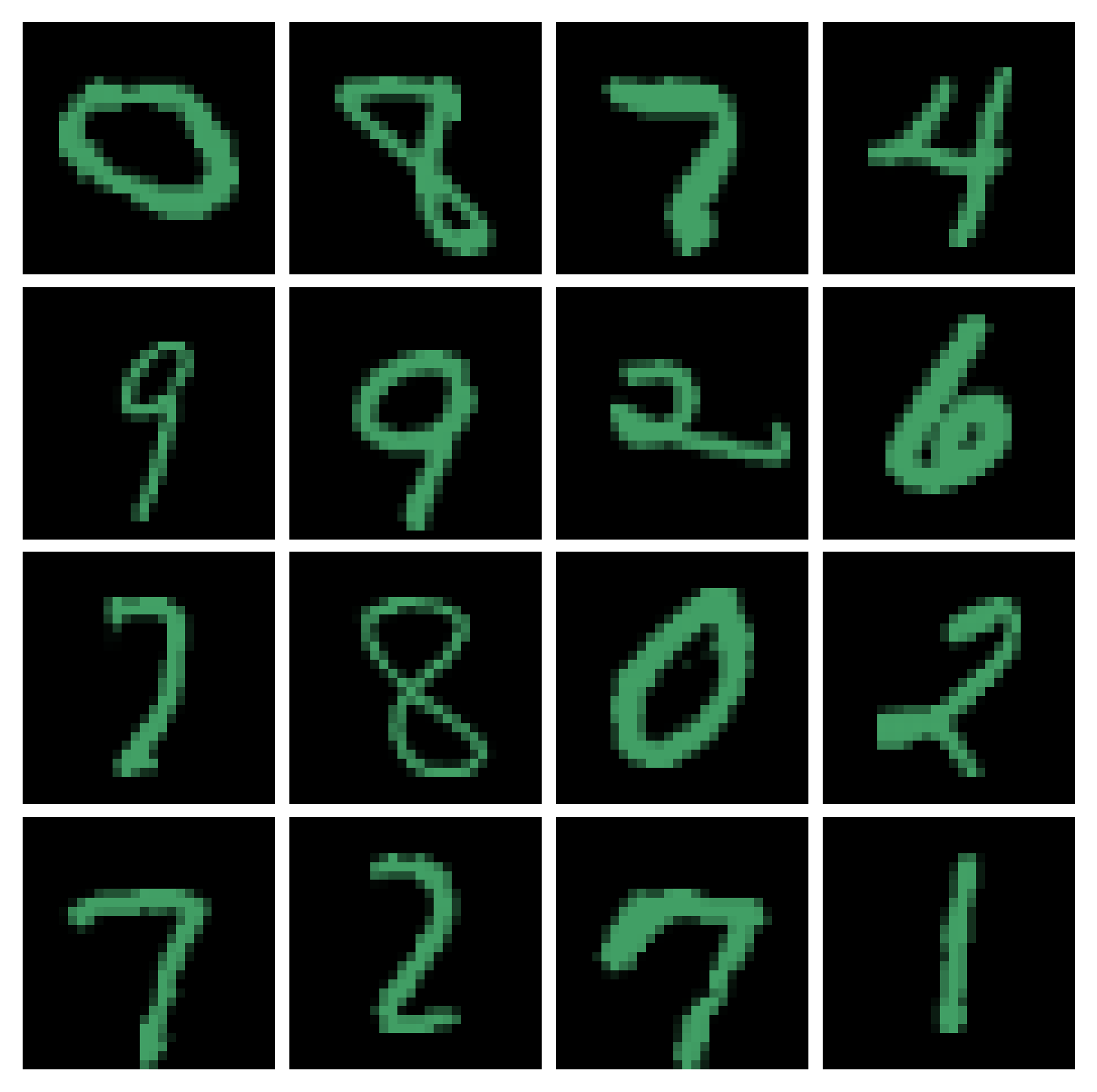}
        \caption{MNIST with uniformly green digits.}
    \end{subfigure}
    % Second subfigure
    \begin{subfigure}[b]{0.3\columnwidth}
        \centering
        \includegraphics[width=\linewidth]{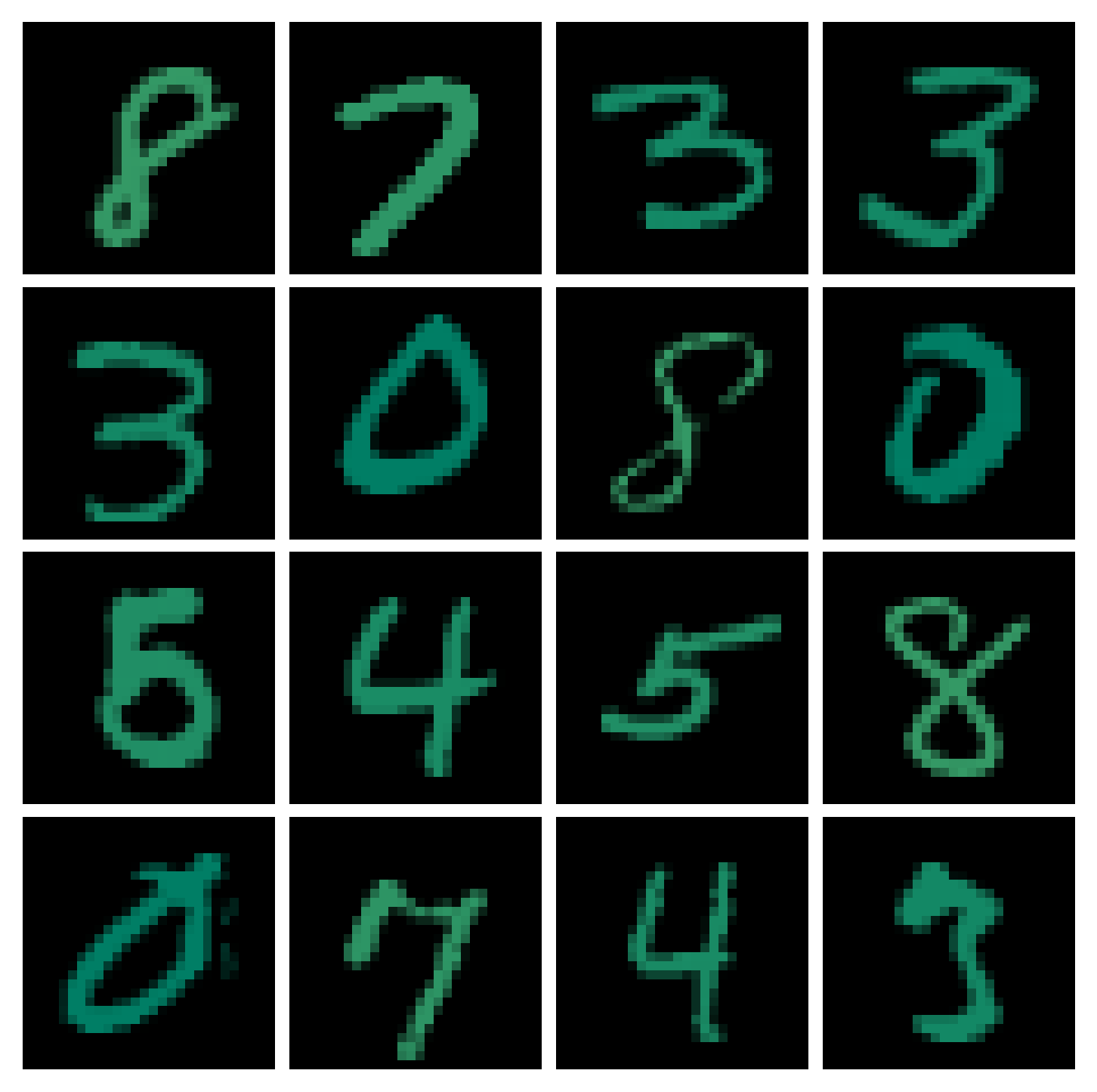}
        \caption{MNIST with varying green digits.}
    \end{subfigure}
    % Third subfigure
    \begin{subfigure}[b]{0.3\columnwidth}
        \centering
        \includegraphics[width=\linewidth]{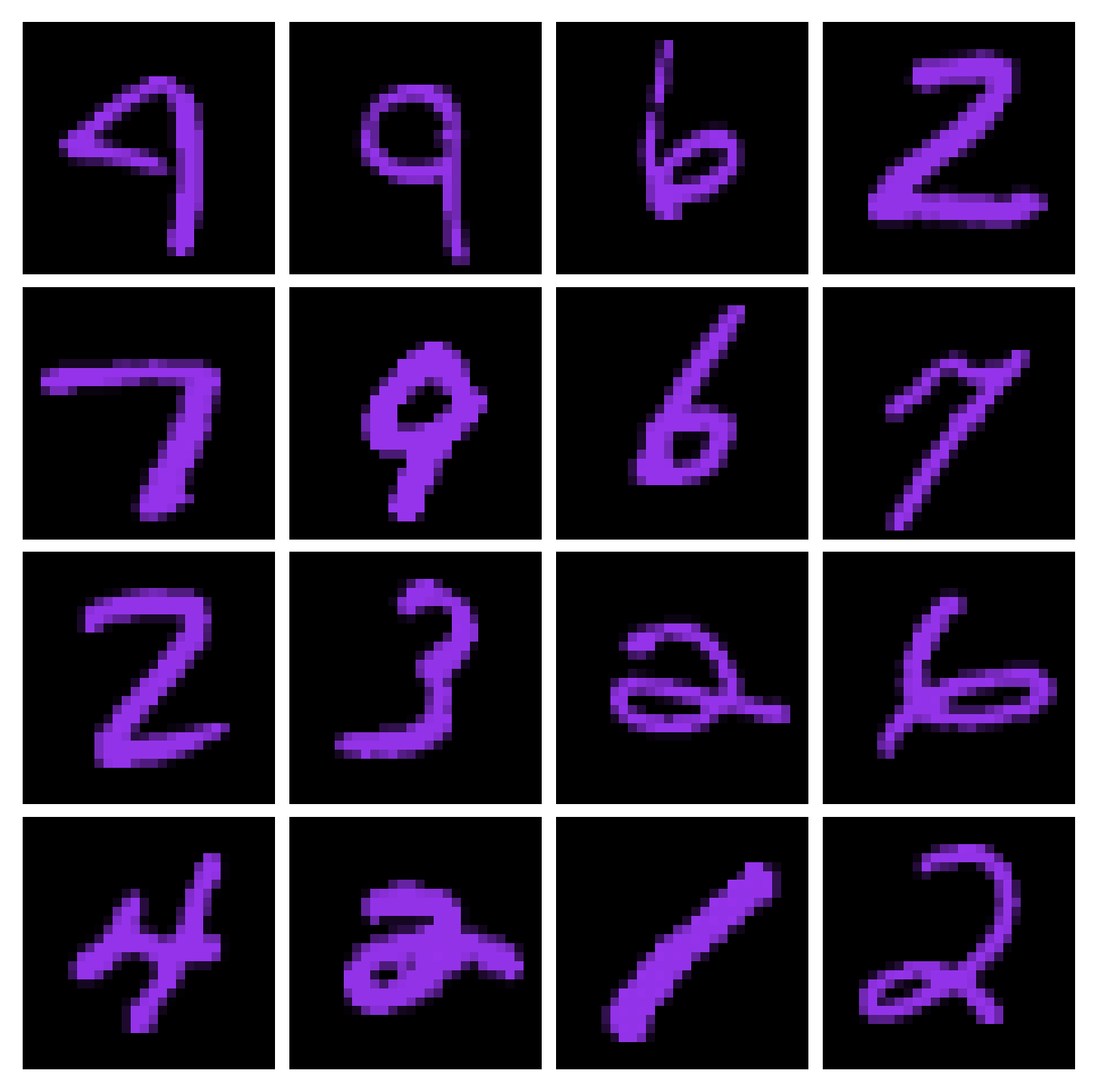}
        \caption{MNIST with uniformly purple digits.}
    \end{subfigure}
    \caption{The first set of clients have uniformly green digits. These clients are targeted by the second set of clients, whose data contains a backdoor feature. The backdoor is the intensity of the green which is correlated with the digit. The first two sets of clients have visually similar data. The third set of clients has uniformly purple digits, which are visually dissimilar from the first two.}
    \label{fig:cmnist}
\end{figure}

\subsection{Backdoor CIFAR10}
\label{appx:cifarls_examples}

Figure~\ref{fig:cifarls} shows a similar scenario with images from the first five classes (left), images from the first five classes but with a small colour patch in the corner (centre), and images from the last five classes. The backdoor feature is the colour patch, where the colour is correlated with the class.

\begin{figure}[!hbtp]
    \centering
    % First subfigure
    \begin{subfigure}[b]{0.3\columnwidth}
        \centering
        \includegraphics[width=\linewidth]{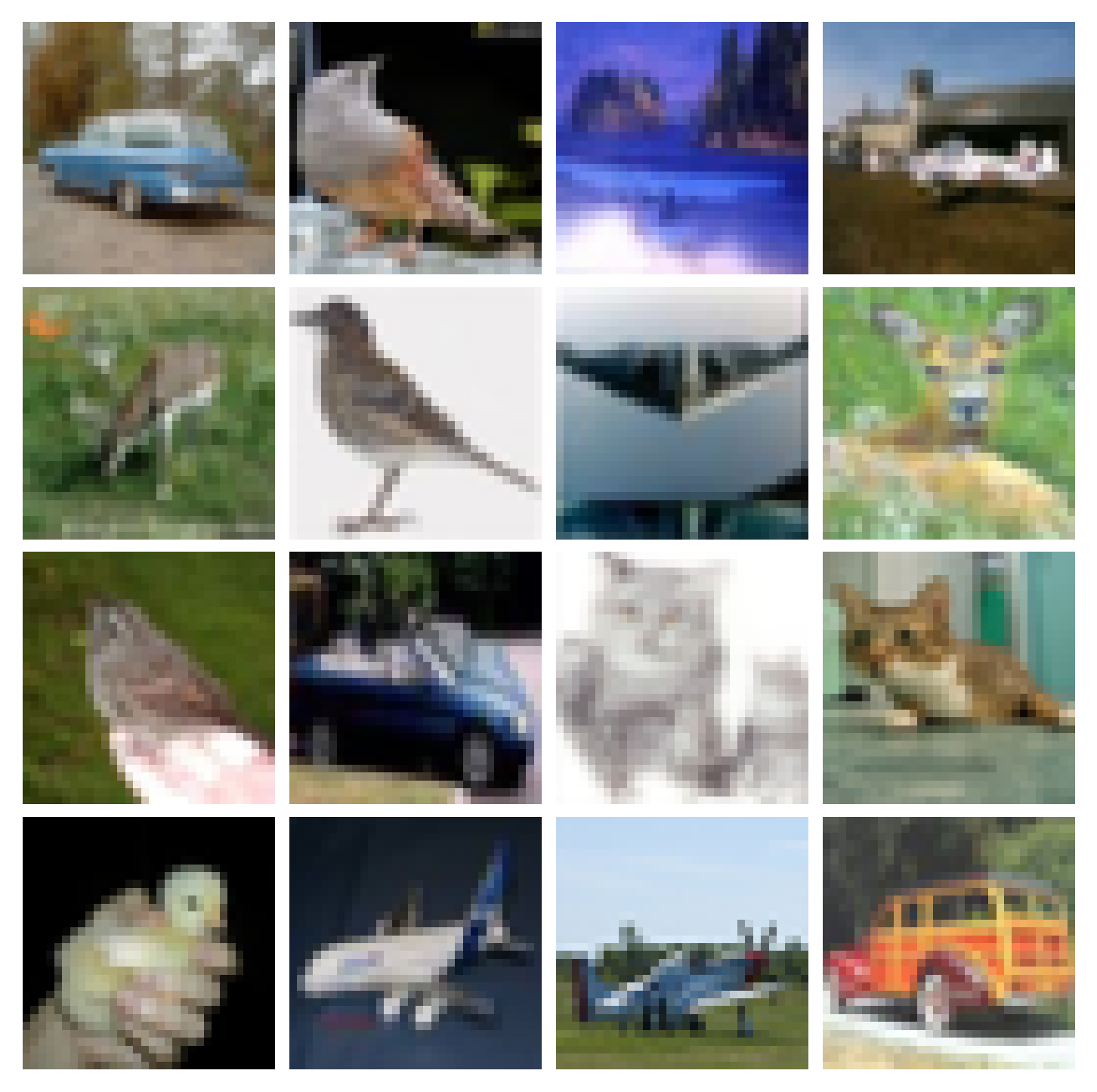}
        \caption{First five classes of Rotated CIFAR10.}
    \end{subfigure}
    % Second subfigure
    \begin{subfigure}[b]{0.3\columnwidth}
        \centering
        \includegraphics[width=\linewidth]{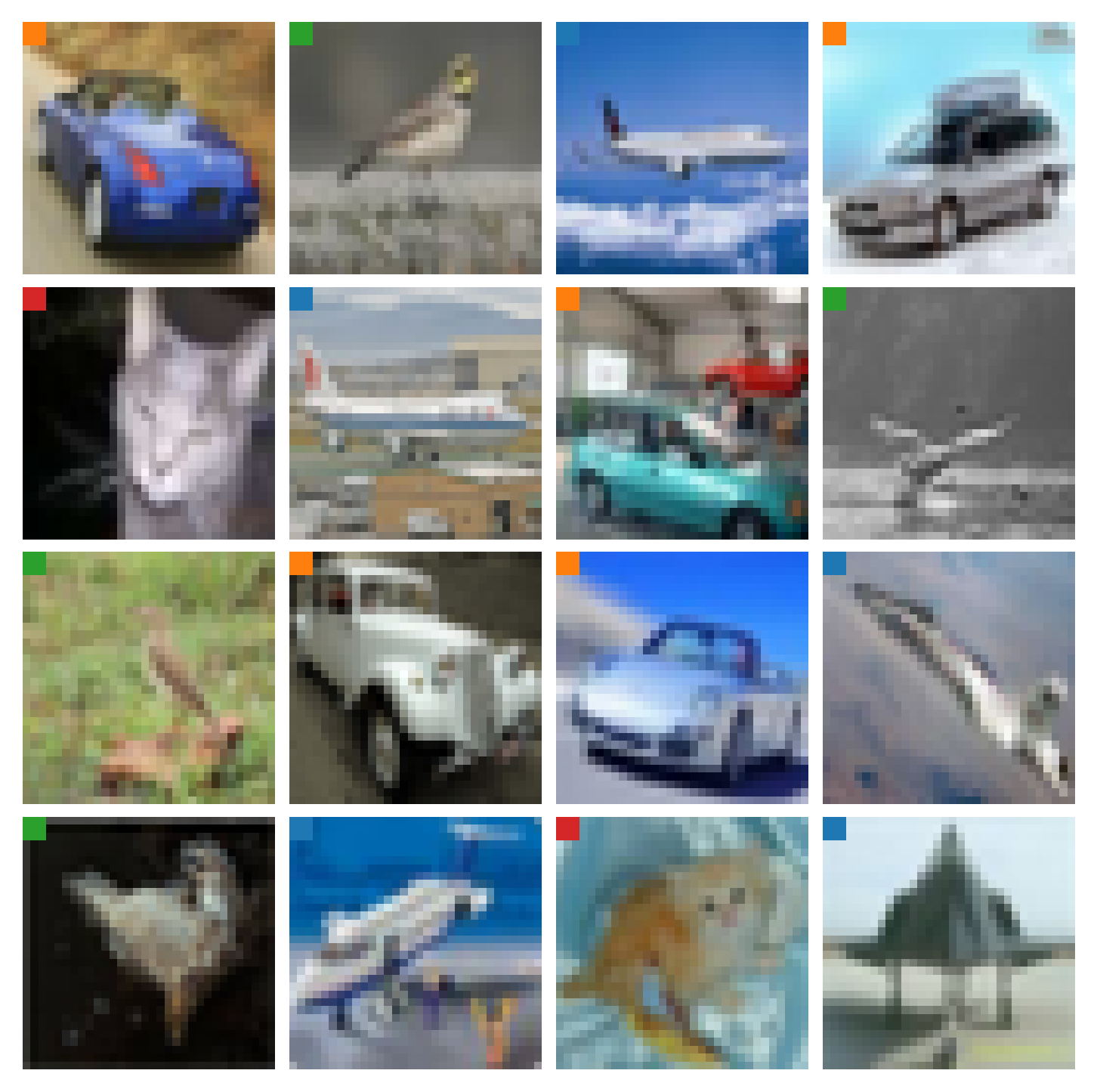}
        \caption{First five classes with colour patch in top left corner.}
    \end{subfigure}
    % Third subfigure
    \begin{subfigure}[b]{0.3\columnwidth}
        \centering
        \includegraphics[width=\linewidth]{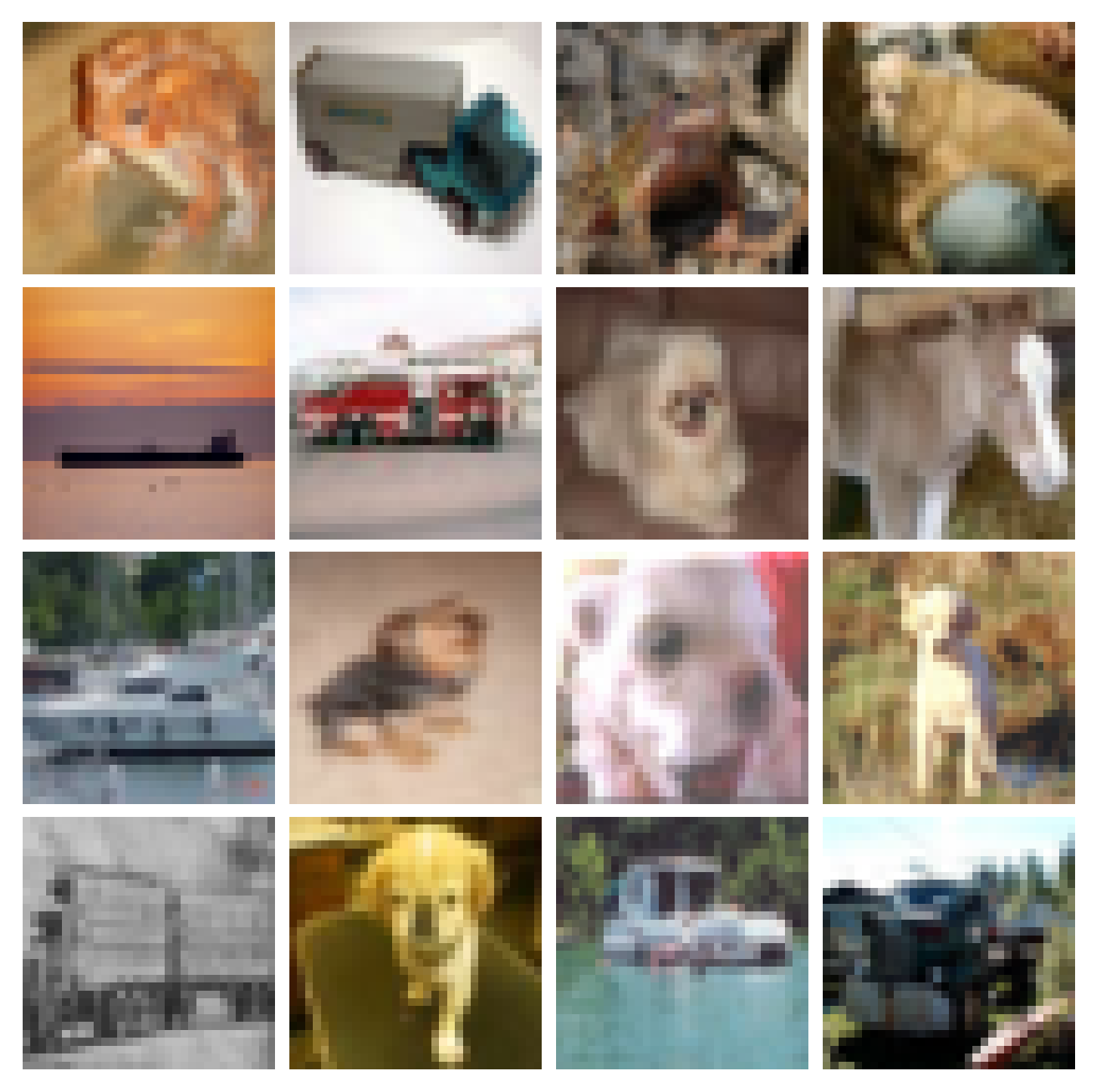}
        \caption{Last five classes of Rotated CIFAR10.}
    \end{subfigure}
    \caption{The first set of clients have images from the first five classes of CIFAR10. These clients are targeted by the second set of clients, who have the same images but also contain a backdoor feature. The backdoor is the small colour patch in the top left, which is correlated with the digit. The first two sets of clients have visually similar data from the same classes. The third set of clients has images from the last five classes of CIFAR10, which are visually dissimilar from the first two.}
    \label{fig:cifarls}
\end{figure}

\newpage
\section{Additional Training Details}
\subsection{Wall Time Comparison}
\label{appx:timing}
We measure the wall time of training using our method over $T=10$ global epochs with $E=10$ local epochs each. For comparability, we use the Oracle as baseline, which has access to the ground-truth clustering and therefore only differs from our method in skipping the clustering step.

\begin{table}[ht]
\centering
\caption{EMD-CFL achieves Oracle clustering with only a 9\% increase in training wall time.}
\label{tab:timing}
\begin{tabular}{lc}
\toprule
\textbf{Method} & \textbf{Hours} \\
\midrule
Oracle & 7.47 \\
EMD-CFL & 8.11 \\
\bottomrule
\end{tabular}
\end{table}

\subsection{Model and Training Details}
\label{appx:hparams}
We use SGD as optimiser with momentum of 0.9 and weight decay of $10^{-6}$. 
We use a simple two-layer CNN (two convolutional layers of size 64 and 128 followed by two fully connected layers with a hidden dimension of 128) on all MNIST experiments. We use ResNet18 for all CIFAR10 experiments and PACS.
We use ResNet50, EfficientNetB3, ViTB16 and ConvNeXt Base for additional experiments on CIFAR10. We provide an overview of native embedding dimensions (before random projection) and state all learning rates in Table~\ref{tab:model_details}. We implemented all experiments in PyTorch 2.5 and trained on an HPC cluster with NVIDIA L40S GPUs. We use Python Optimal Transport 0.9.5 \citep{flamary2021pot} for the EMD calculations.

\begin{table}[ht]
\centering
\caption{Native embedding dimensions and learning rates used for training.}
\label{tab:model_details}
% \resizebox{0.7\columnwidth}{!}{%
\begin{tabular}{lcc}
\toprule
\textbf{Model} & \textbf{dim}\boldmath{$(Z)$} & \boldmath{$\kappa$} \\
\midrule
Simple CNN &  128 & 0.01 \\
ResNet18 on CIFAR10 &  512 & 0.01 \\
ResNet18 on PACS &  512 & 0.001 \\
ResNet50 &  2048 & 0.01 \\
EfficientNetB3 &  1536 & 0.01 \\
ViTB16 & 768 & 0.003 \\
ConvNeXt Base & 1024 & 0.003 \\
\bottomrule
\end{tabular}
% }
\end{table}

\subsection{Hyperparameters for Baselines}
\label{appx:baseline_hparams}
For the main experiments, we provide the ground-truth number of clusters to IFCA, FlexCFL, FeSEM, CFLGP, FedEM, FedSoft, FedRC and FedCE. CFL recursively partitions clients and requires parameters that control the starting and stopping of the partitioning. We use the hyperparameters specified by \citet{sattler2020clustered} for this purpose. PACFL and FedClust require a threshold similar to our $\epsilon$ which we tune to find the closest to the optimal split in the first epoch. For PACFL, we used an $\epsilon$ of 12 for Rotated MNIST and CIFAR10 experiments, 11 for PACS,
2.5 for Backdoor MNIST and 10.5 for Backdoor CIFAR10. For FedClust, we used an $\epsilon$ of 5 for Rotated MNIST, 45 for Rotated CIFAR10, 8 for PACS,
4 for Backdoor MNIST and 28 for Backdoor CIFAR10.

\newpage
\section{Additional Experimental Results}
\label{appx:extra_results}
Appendix~\ref{appx:hparams_flex} shows EMD-CFL is robust to hyperparameter choices. Appendix~\ref{appx:sinkhorn} shows these findings hold under approximate Sinkhorn distances. Appendix~\ref{appx:architecture_flex} shows that EMD-CFL works well for larger model architectures. Appendix~\ref{appx:partial} shows EMD-CFL is robust to partial participation of clients, matching the Oracle in each case. Appendix~\ref{appx:full_pacs} reports the full results for PACS, showing that soft clustering underperforms and that EMD-CFL is robust to hyperparameter choices. Appendix~\ref{appx:backdoor} reports the full results on the backdoor experiments, showing that soft clustering underperforms and that EMD-CFL is robust to hyperparameter choices. Appendix~\ref{appx:failure_analysis} qualitatively compares EMDs with distances used by other methods and suggests that EMDs reflect the underlying clustering structure more clearly. Appendix~\ref{appx:cluster_total_param_distances} provides evidence that EMD-CFL identifies clusters with smaller average parameter distances, as suggested by theoretical results. Finally, Appendix~\ref{appx:induced} shows EMD-CFL performs well on additional versions of Rotated MNIST and Rotated CIFAR10.

\subsection{Robustness to Hyperparameter Choices}
\label{appx:hparams_flex}

We study the clustering performance under different choices of $\epsilon$ and $\text{dim}(R)$. We show in Figure~\ref{fig:hparams} that our method remains stable on a wide range of $\epsilon$-values as well as under large dimensionality reductions of up to 90\% for both Rotated MNIST and Rotated CIFAR10.

\begin{figure}[ht]
    \centering
    \includegraphics[width=0.8\textwidth]{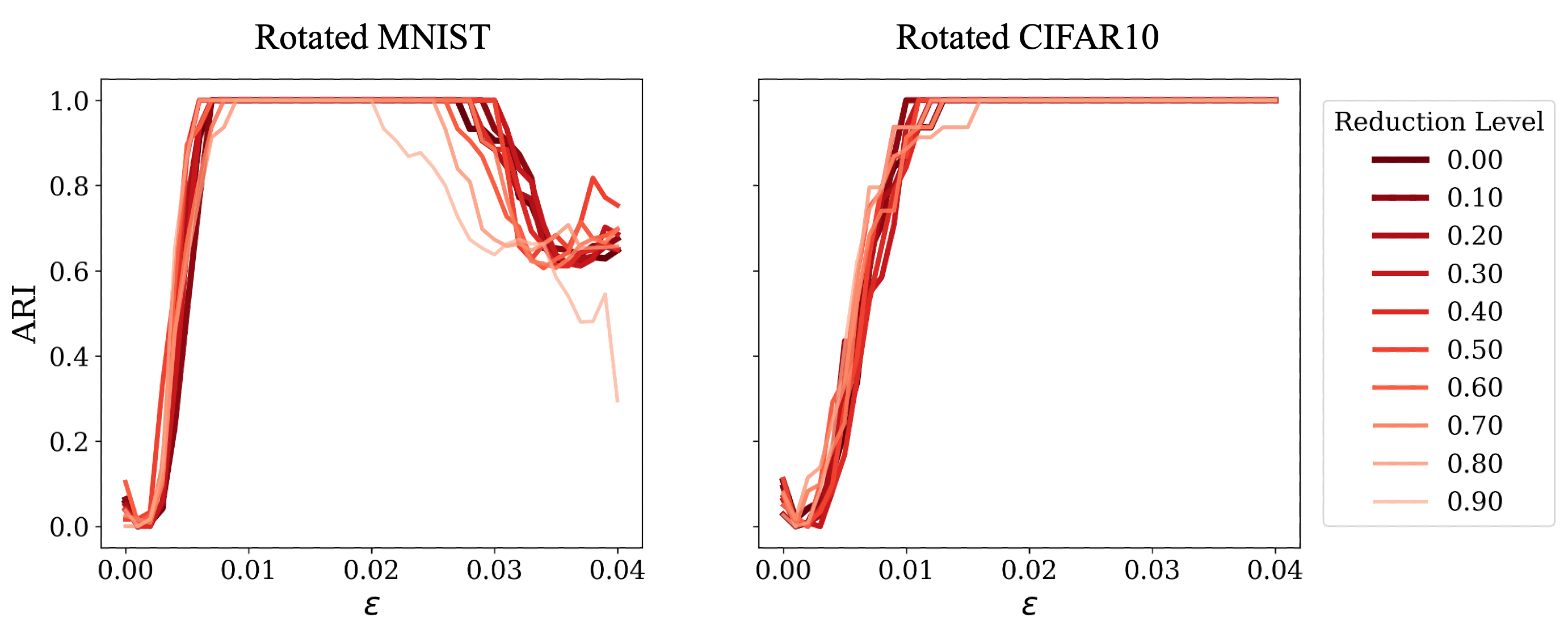}
    \caption{A wide range for $\epsilon$ and dimensionality reduction levels result in correct clustering on both Rotated MNIST and CIFAR10.}
    \label{fig:hparams}
\end{figure}

\subsection{Robustness to Approximate EMDs}
\label{appx:sinkhorn}
The Sinkhorn distance adds an entropic regularisation term to the EMD, which allows for faster computation at the cost of only obtaining an approximation of the EMD \citep{cuturi2013sinkhorn}. We use the solver from Python Optimal Transport \citep{flamary2021pot} with a regularisation strength of 0.1 to show that our method continues to obtain the correct clustering (Figure~\ref{fig:sinkhorn_hparams}).

\begin{figure}[ht]
    \centering
    \includegraphics[width=0.8\textwidth]{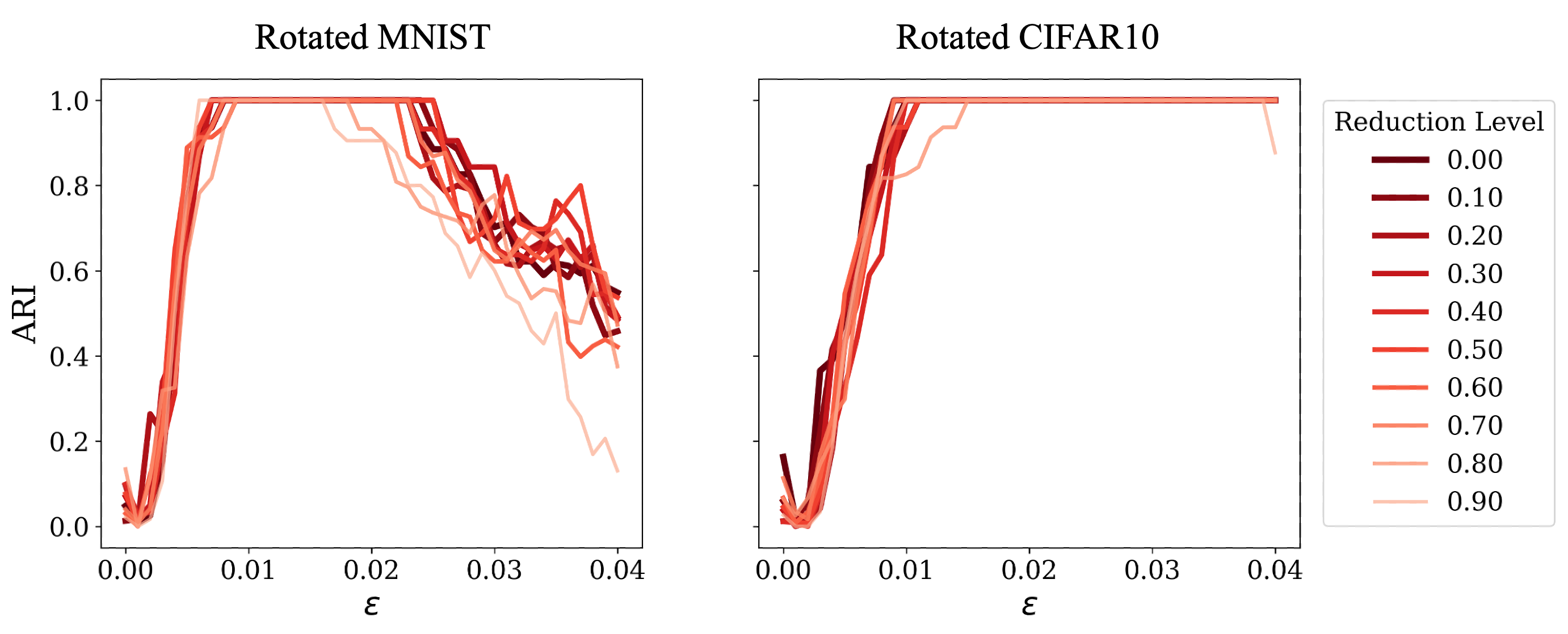}
    \caption{A wide range for $\epsilon$ and dimensionality reduction levels result in correct clustering under Sinkhorn distances on both Rotated MNIST and CIFAR10.}
    \label{fig:sinkhorn_hparams}
\end{figure}

\subsection{Robustness to Model Architectures}
\label{appx:architecture_flex}

The choice of $\epsilon$ depends on the embedding model. We demonstrate that our method is robust to model architecture choices and report the clustering performance on the challenging Rotated CIFAR10 dataset using a range of larger model architectures, both using exact and approximate Sinkhorn distances.(Figure~\ref{fig:hparams_flex}). We visualise the EMDs between clients in Figure~\ref{fig:model_ablation}, which qualitatively shows the clustering. Interestingly, the first set of clients (with unrotated images) consider all other clients to be much farther, as indicated by the darker red in the first rows. We hypothesise that this may be due to an inductive bias from the pre-training of these models, which typically does not make use of rotations for data augmentation, so that the first set of clients will be unfamiliar with rotated images, while clients with rotated data will have seen unrotated images during pre-training.

\begin{figure}[ht]
    \centering
    % First subfigure
    \begin{subfigure}[b]{0.48\linewidth}
        \centering
        \includegraphics[width=\linewidth]{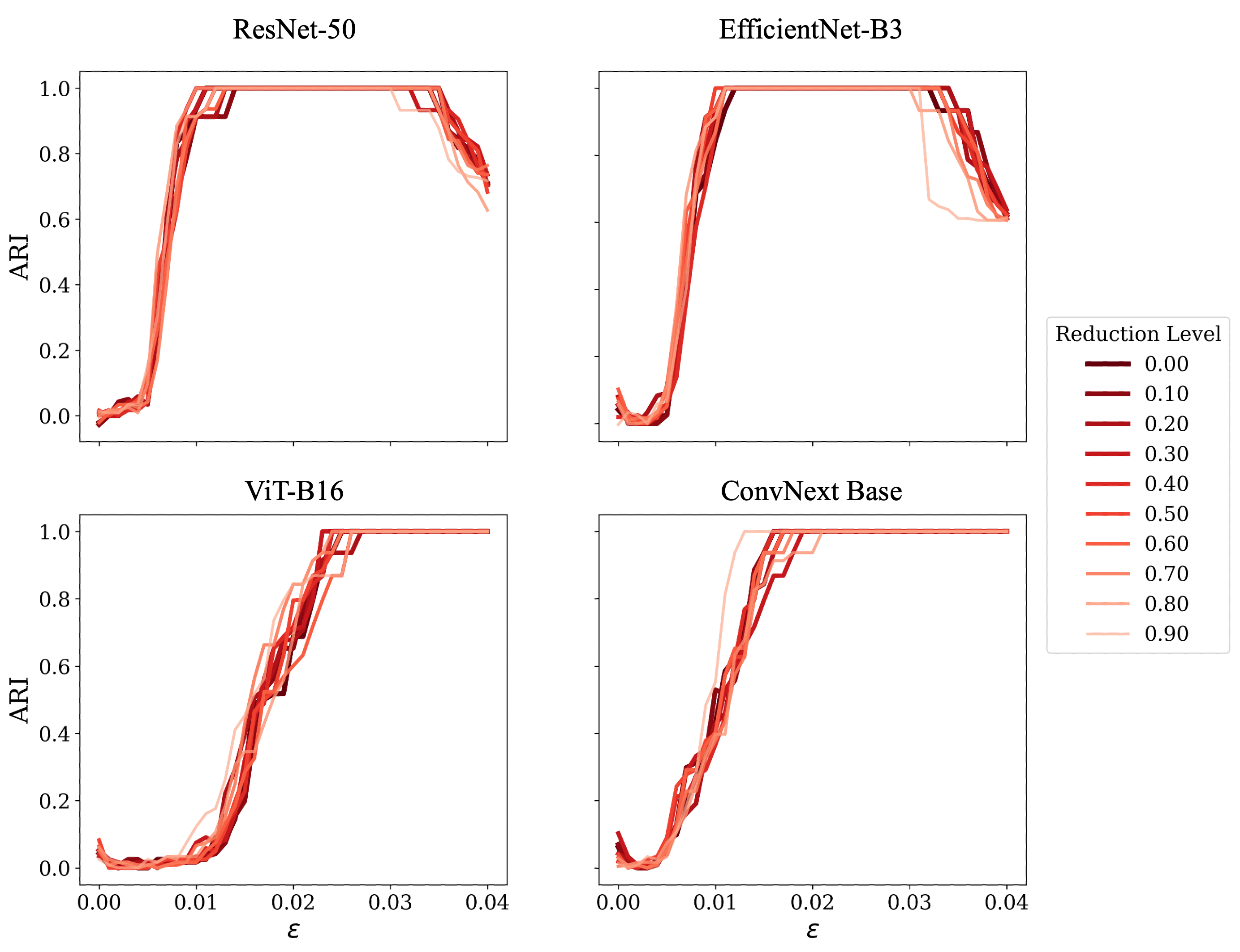}
        \caption{Exact}
    \end{subfigure}
    \hfill
    % Second subfigure
    \begin{subfigure}[b]{0.48\linewidth}
        \centering
        \includegraphics[width=\linewidth]{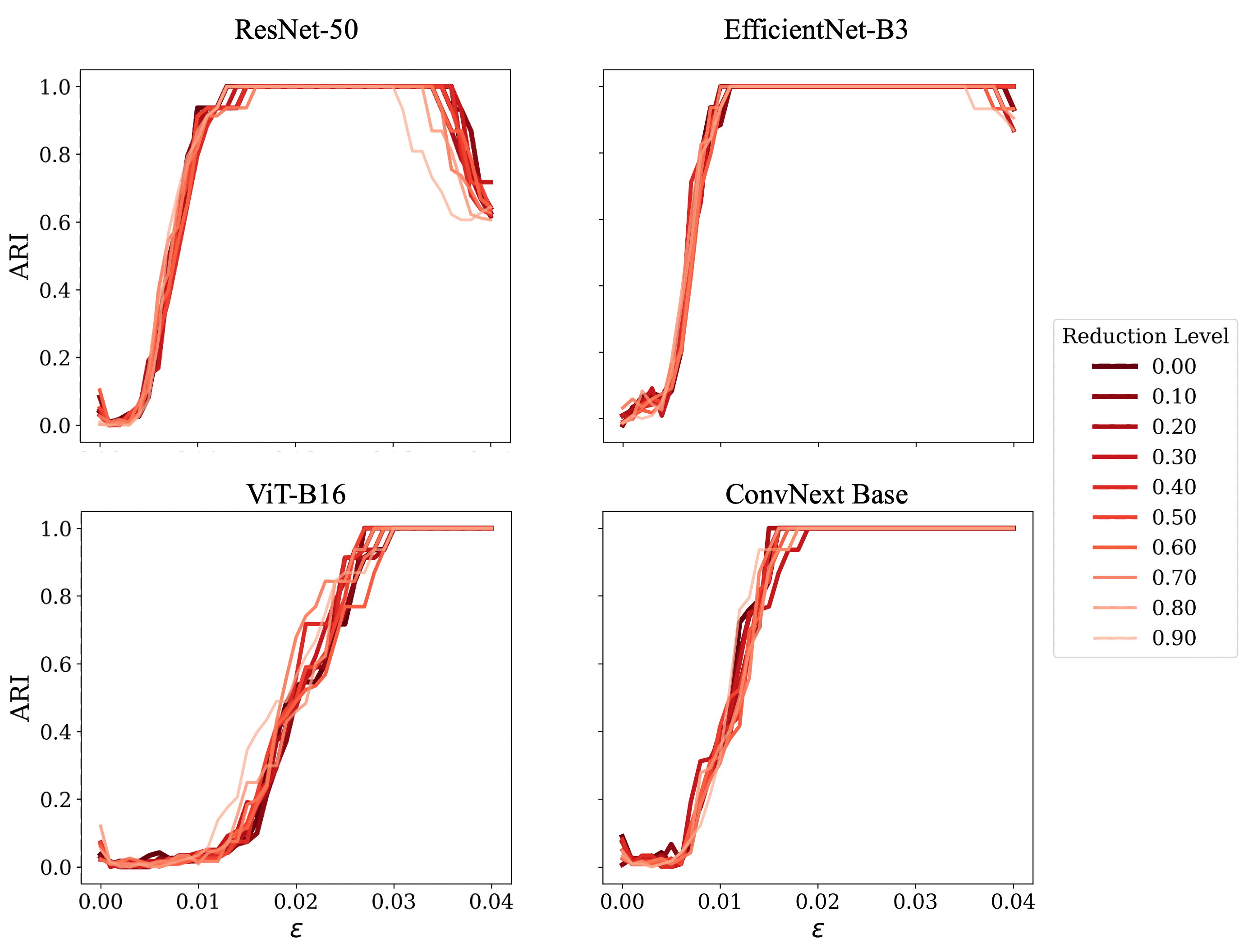}
        \caption{Sinkhorn}
    \end{subfigure}
    \caption{EMD-CFL achieves stable clustering for a range of larger model architectures and using approximate Sinkhorn distances.}
    \label{fig:hparams_flex}
\end{figure}

\begin{figure}[ht]
    \centering
    \includegraphics[width=0.8\textwidth]{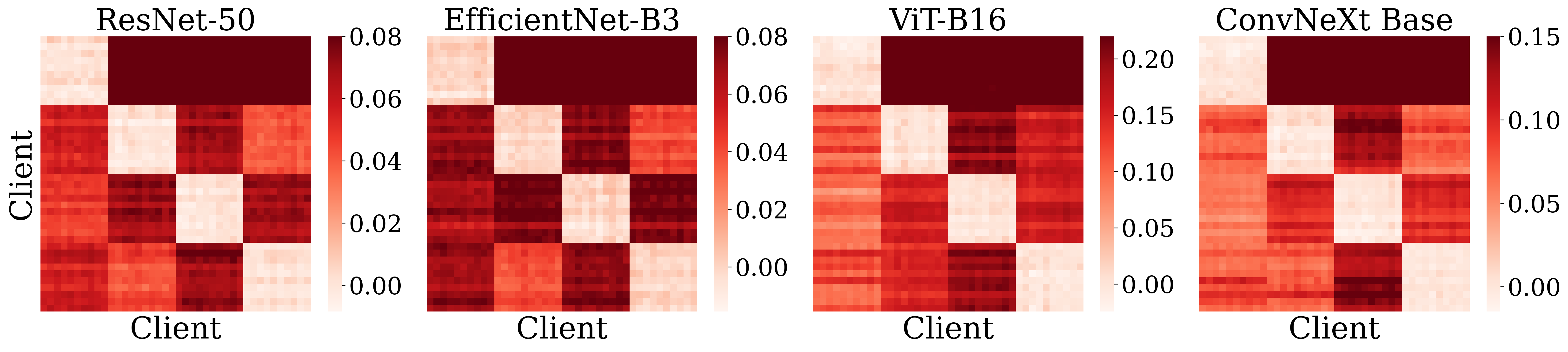}
    \caption{Pairwise EMDs for different model architectures on Rotated CIFAR10 illustrate a clustering structure amongst clients.}
    \label{fig:model_ablation}
\end{figure}

% \newpage
\subsection{Robustness to Partial Participation}
\label{appx:partial}
We consider the same setup as in the main paper but randomly sample $\{10,20,30\}$ clients for participation in each epoch. Table~\ref{tab:pp} shows that our method matches the Oracle for both Rotated MNIST and Rotated CIFAR10.

\begin{table}[htbp]
    \centering
    \caption{Results on Rotated MNIST and Rotated CIFAR10 under varying levels of partial participation.}
    \label{tab:pp}
    \resizebox{\textwidth}{!}{%
    \begin{tabular}{llcccccc}
    \toprule
    \textbf{Dataset} & \textbf{Method} & \multicolumn{2}{c}{\textbf{10}} & \multicolumn{2}{c}{\textbf{20}} & \multicolumn{2}{c}{\textbf{30}} \\
    \cmidrule(lr){3-4} \cmidrule(lr){5-6} \cmidrule(lr){7-8}
    & & \textbf{Acc} & \textbf{ARI} & \textbf{Acc} & \textbf{ARI} & \textbf{Acc} & \textbf{ARI} \\
    \midrule
    \multirow{2}{*}{Rotated MNIST}
        & Oracle   & 93.62$\pm$2.12 & 0.93$\pm$0.03 & 98.54$\pm$0.09 & 1.00$\pm$0.00 & 98.75$\pm$0.03 & 1.00$\pm$0.00 \\
        & EMD-CFL  & 93.70$\pm$2.11 & 0.93$\pm$0.03 & 98.52$\pm$0.03 & 1.00$\pm$0.00 & 98.74$\pm$0.04 & 1.00$\pm$0.00 \\
    \midrule
    \multirow{2}{*}{Rotated CIFAR10}
        & Oracle   & 88.15$\pm$2.00 & 0.93$\pm$0.03 & 93.88$\pm$0.06 & 1.00$\pm$0.00 & 94.47$\pm$0.11 & 1.00$\pm$0.00 \\
        & EMD-CFL  & 88.18$\pm$1.93 & 0.93$\pm$0.03 & 93.86$\pm$0.06 & 1.00$\pm$0.00 & 94.55$\pm$0.04 & 1.00$\pm$0.00 \\
    \bottomrule
    \end{tabular}
    }
\end{table}

\subsection{Full Results on PACS}
\label{appx:full_pacs}
Table~\ref{tab:full_pacs} reports the full set of results on PACS, showing in particular that soft clustering methods underperform hard clustering methods. Figure~\ref{fig:hparams_pacs_flex} shows that EMD-CFL achieves high clustering performance for a range of $\epsilon$-values and dimensionality reduction levels of up to 50\% on PACS, including using approximate Sinkhorn distances. These results demonstrate the stability of the clustering on this naturally clustered dataset.

\begin{table}[ht]
\centering
\caption{Several methods recover the four domains of PACS as separate clusters (bold). Our method does so in the first epoch and matches the Oracle performance.}
\label{tab:full_pacs}
\resizebox{0.5\textwidth}{!}{
\begin{tabular}{lccc}
\toprule
\textbf{Method} & \textbf{Avg Acc} & \textbf{Worst Acc} & \textbf{ARI} \\
\midrule
Oracle & 92.43$\pm$0.19 & 85.76$\pm$0.07 & 1.00$\pm$0.00 \\ \cmidrule(lr){1-4}
\textbf{EMD-CFL } & \textbf{92.47$\pm$0.32} & \textbf{86.31$\pm$1.36} & \textbf{1.00$\pm$0.00} \\
CFL & 86.03$\pm$0.79 & 79.17$\pm$1.86 & 0.00$\pm$0.00 \\
PACFL & 88.43$\pm$0.52 & 45.31$\pm$9.18 & 0.70$\pm$0.00 \\
\textbf{FedClust} & \textbf{87.63$\pm$0.36} & \textbf{75.74$\pm$1.99} & \textbf{1.00$\pm$0.00} \\
IFCA & 90.39$\pm$2.07 & 82.47$\pm$5.52 & 0.69$\pm$0.22 \\
\textbf{FlexCFL} & \textbf{92.36$\pm$0.27} & \textbf{85.42$\pm$1.03} & \textbf{1.00$\pm$0.00} \\
FeSEM & 81.77$\pm$0.70 & 72.62$\pm$1.03 & 0.00$\pm$0.00 \\
\textbf{CFLGP} & \textbf{92.07$\pm$0.21} & \textbf{84.28$\pm$1.30} & \textbf{1.00$\pm$0.00} \\
FedEM & 89.30$\pm$0.58 & 78.27$\pm$1.36 & - \\
FedSoft & 92.02$\pm$0.52 & 83.33$\pm$1.86 & - \\
FedRC & 89.72$\pm$0.86 & 78.87$\pm$2.25 & - \\ %\cmidrule(lr){1-4}
FedCE & 87.83$\pm$2.90 & 77.43$\pm$6.54 & - \\ %\cmidrule(lr){1-4}
FedAvg & 86.20$\pm$0.66 & 79.23$\pm$1.39 & - \\
FedProx & 85.72$\pm$0.83 & 78.70$\pm$0.11 & - \\
pFedGraph & 83.01$\pm$1.00 & 70.92$\pm$3.15 & - \\
FedSaC & 76.94$\pm$3.76 & 52.29$\pm$7.98 & - \\
\bottomrule
\end{tabular}%
}
\end{table}

\begin{figure}[!htbp]
    \centering
    \includegraphics[width=0.8\columnwidth]{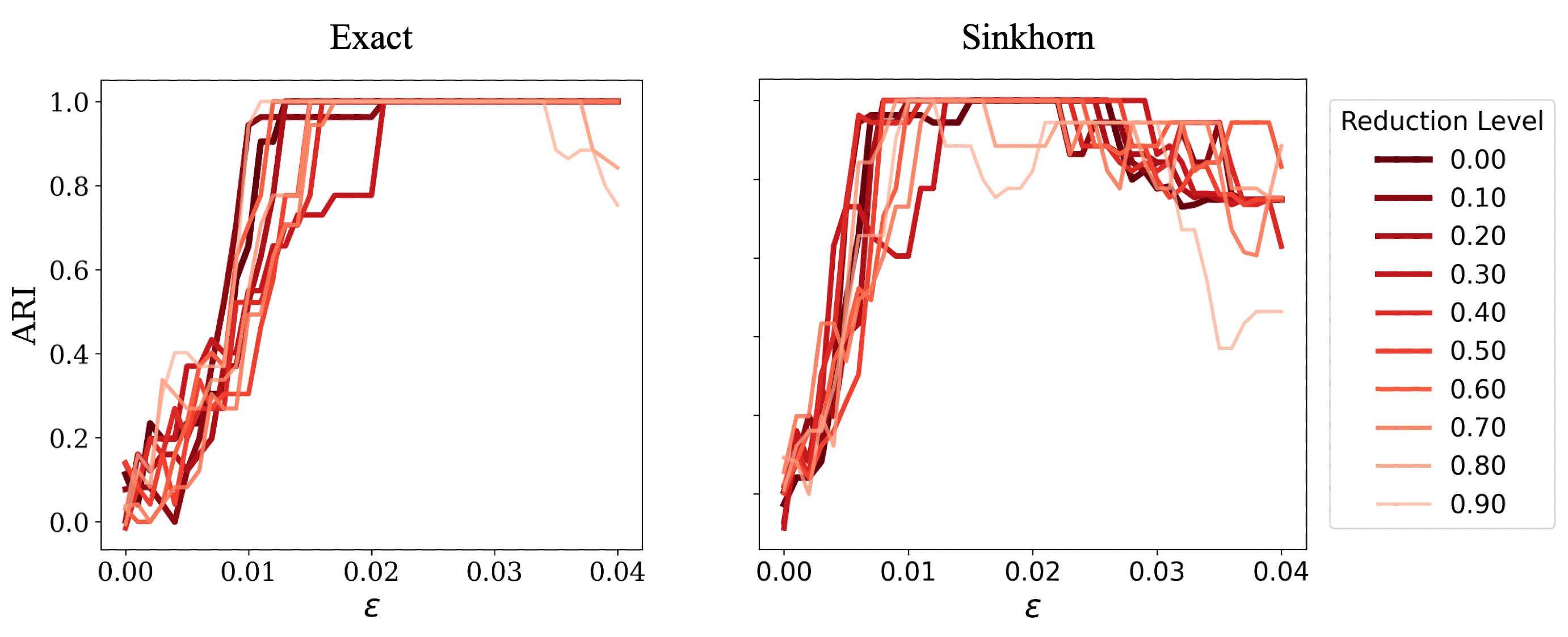}
    \caption{Our method achieves stable clustering performance on PACS using exact EMDs and approximate Sinkhorn distances.}
    \label{fig:hparams_pacs_flex}
\end{figure}

\newpage
\subsection{Backdoor Features}
\label{appx:backdoor}
We report the full results on the backdoor feature experiments in Table~\ref{tab:full_cmnist_cifar}, which include the figures for the soft and non-clustering methods. In general, hard clustering seems to be necessary to protect the target clients from being backdoored. We further show in Figure~\ref{fig:backdoor_proj_flex} that our clustering remains stable for a range of $\epsilon$-values and using approximate Sinkhorn distances.

\begin{table}[!htbp]
\centering
\caption{Several methods segregate clients with backdoor features on Backdoor MNIST and Backdoor CIFAR10 but fail to do so at the start of training, resulting in some sensitivity to the backdoor feature.}
\label{tab:full_cmnist_cifar}
\resizebox{\textwidth}{!}{
\begin{tabular}{lcccccc}
\toprule
 & \multicolumn{3}{c}{\textbf{Backdoor MNIST}} & \multicolumn{3}{c}{\textbf{Backdoor CIFAR10}} \\ 
\cmidrule(lr){2-4} \cmidrule(lr){5-7}
\textbf{Method} & \textbf{Clean Acc} & \textbf{Backdoor Acc} & \textbf{ARI} & \textbf{Clean Acc} & \textbf{Backdoor Acc} & \textbf{ARI} \\
\midrule
Oracle & 98.37$\pm$0.11 & 98.27$\pm$0.16 & 1.00$\pm$0.00 & 97.89$\pm$0.12 & 97.63$\pm$0.06 & 1.00$\pm$0.00 \\ \cmidrule(lr){1-7}
EMD-CFL & 98.40$\pm$0.04 & 98.30$\pm$0.06 & 1.00$\pm$0.00 & 97.84$\pm$0.23 & 97.73$\pm$0.32 & 1.00$\pm$0.00 \\
CFL & 98.43$\pm$0.07 & 95.98$\pm$1.39 & 0.46$\pm$0.02 & 92.37$\pm$0.27 & 39.94$\pm$4.85 & 0.00$\pm$0.00 \\
PACFL & 95.55$\pm$0.14 & 81.55$\pm$1.45 & 0.42$\pm$0.00 & 97.71$\pm$0.11 & 84.57$\pm$3.13 & 0.53$\pm$0.00 \\
FedClust & 93.04$\pm$3.35 & 93.82$\pm$1.54 & 0.53$\pm$0.02 & 82.61$\pm$9.00 & 78.84$\pm$9.60 & 0.16$\pm$0.06 \\
IFCA & 98.54$\pm$0.05 & 98.15$\pm$0.14 & 0.55$\pm$0.00 & 97.62$\pm$0.14 & 77.69$\pm$4.06 & 0.53$\pm$0.00 \\
FlexCFL & 39.51$\pm$51.12 & 35.46$\pm$44.10 & 1.00$\pm$0.00 & 76.39$\pm$22.40 & 29.83$\pm$8.14 & 0.00$\pm$0.10 \\
FeSEM & 97.23$\pm$0.00 & 92.68$\pm$2.50 & -0.00$\pm$0.01 & 94.39$\pm$0.08 & 91.65$\pm$0.37 & 0.00$\pm$0.00 \\
CFLGP & 98.34$\pm$0.08 & 98.27$\pm$0.10 & 1.00$\pm$0.00 & 76.42$\pm$25.39 & 35.19$\pm$4.33 & 0.06$\pm$0.16 \\
FedEM & 97.60$\pm$0.24 & 80.09$\pm$7.82 & - & 96.55$\pm$0.37 & 57.81$\pm$10.92 & - \\
FedSoft & 98.24$\pm$0.06 & 95.54$\pm$0.74 & - & 96.84$\pm$0.85 & 77.86$\pm$2.12 & - \\
FedRC & 97.65$\pm$0.29 & 79.74$\pm$7.94 & - & 96.57$\pm$0.25 & 60.93$\pm$11.65 & - \\
FedCE & 98.52$\pm$0.12 & 95.60$\pm$0.89 & - & 97.65$\pm$2.38 & 82.94$\pm$2.38 & - \\
FedAvg & 97.19$\pm$0.29 & 88.49$\pm$1.39 & - & 92.17$\pm$0.24 & 35.07$\pm$5.25 & - \\
FedProx & 97.85$\pm$0.04 & 91.12$\pm$1.76 & - & 95.23$\pm$0.26 & 89.94$\pm$1.62 & - \\
pFedGraph & 97.88$\pm$0.15 & 82.49$\pm$2.31 & - & 90.69$\pm$1.00 & 25.99$\pm$7.48 & - \\
FedSaC & 97.89$\pm$0.35 & 89.63$\pm$2.14 & - & 32.27$\pm$14.00 & 1.33$\pm$1.28 & - \\
\bottomrule
\end{tabular}
}
\end{table}

\begin{figure}[ht]
    \centering
    % First subfigure
    \begin{subfigure}[b]{0.8\linewidth}
        \centering
        \includegraphics[width=\linewidth]{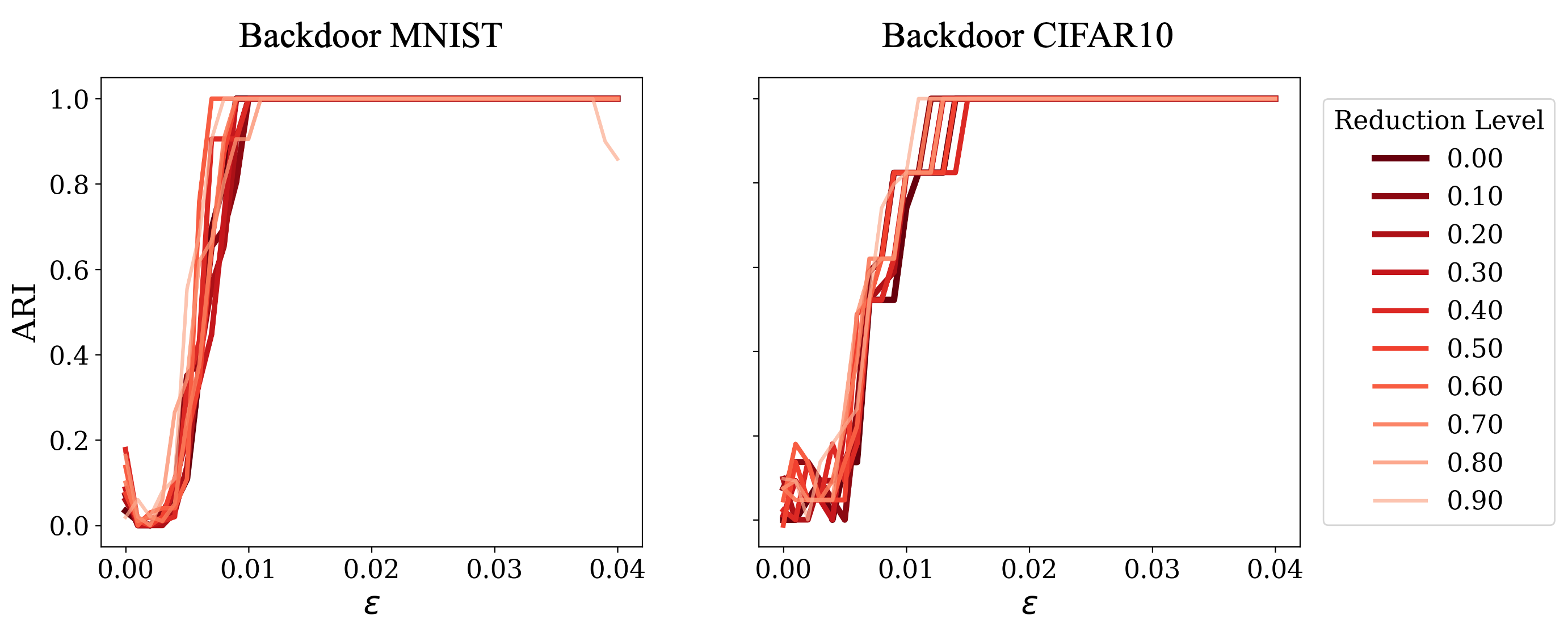}
        \caption{Exact}
    \end{subfigure}

    % Second subfigure
    \begin{subfigure}[b]{0.8\linewidth}
        \centering
        \includegraphics[width=\linewidth]{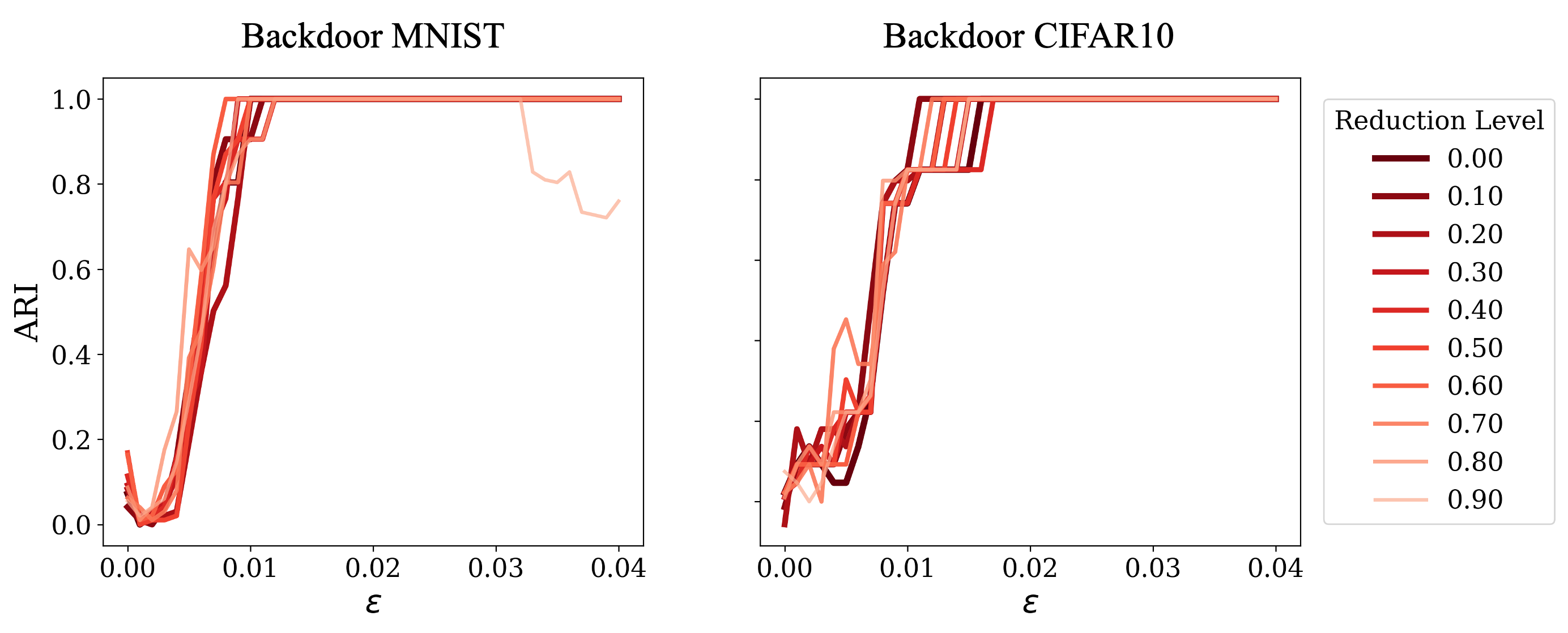}
        \caption{Sinkhorn}
    \end{subfigure}

    \caption{Our method achieves stable clustering performance on Backdoor MNIST and Backdoor CIFAR10 using exact and approximate Sinkhorn distances.}
    \label{fig:backdoor_proj_flex}
\end{figure}

\clearpage
\subsection{Analysing Parameter, Gradient and Raw Data Distances}
\label{appx:failure_analysis}
We found that parameter and gradient-based methods (such as CFL and FedClust) consistently underperformed our method. Similarly, PACFL, which focuses on input space distances, underperformed. We therefore analyse the parameter (Figure~\ref{fig:param_distances}), gradient (Figure~\ref{fig:grad_distances}) and principal angles in input space (Figure~\ref{fig:raw_distances}), which these competing methods use and contrast them with the EMDs our method uses (Figure~\ref{fig:emds}). We find that for more challenging data (Rotated CIFAR10, Backdoor MNIST and Backdoor CIFAR10), these distances are not able to unearth the underlying clustering structure, while the EMDs provide a much clearer reflection thereof. 

\begin{figure}[H]
    \centering
    % First subfigure
    \begin{subfigure}[b]{0.18\columnwidth}
        \centering
        \includegraphics[width=\linewidth]{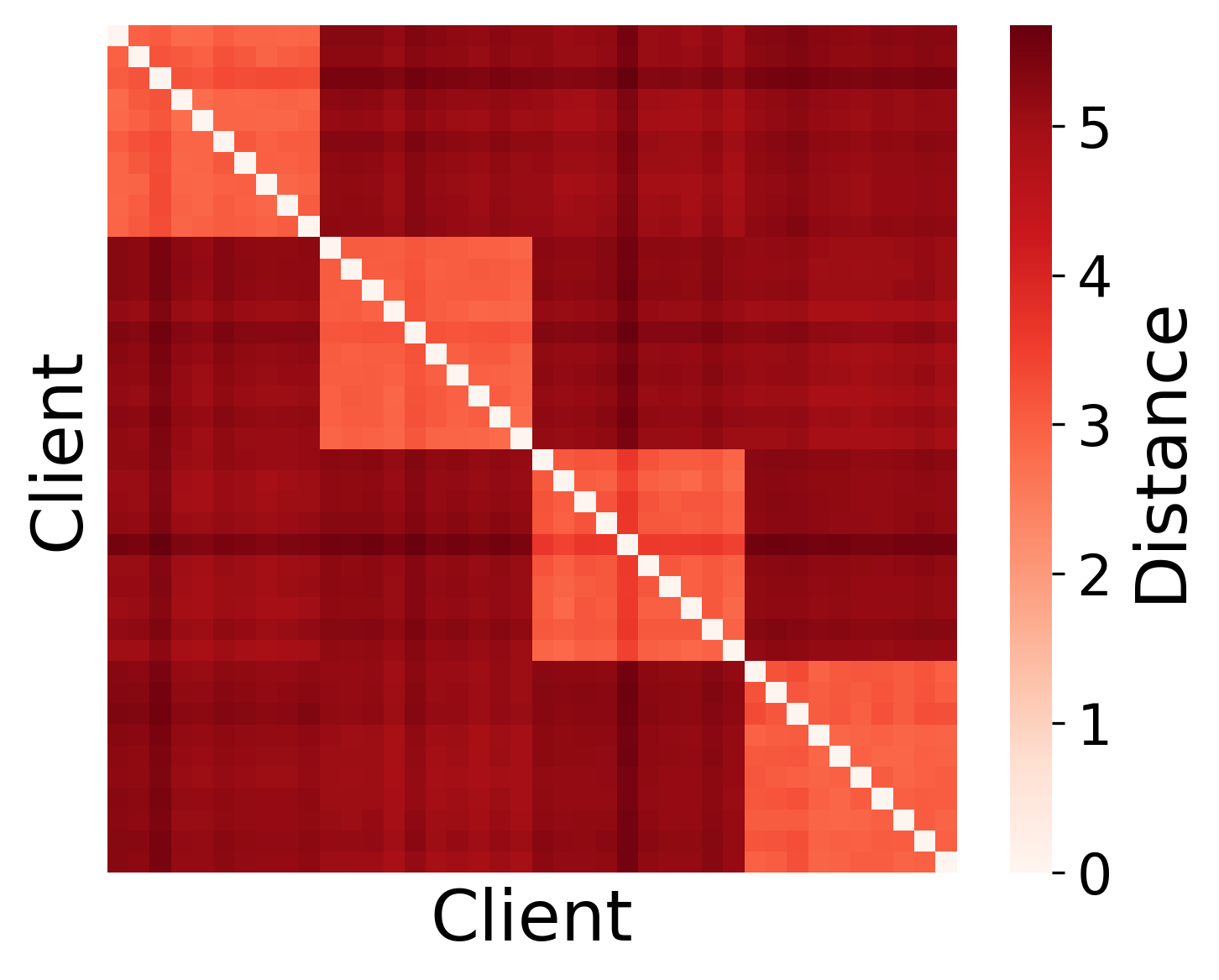}
        \caption{Rot. MNIST}
    \end{subfigure}
    % Second subfigure
    \begin{subfigure}[b]{0.18\columnwidth}
        \centering
        \includegraphics[width=\linewidth]{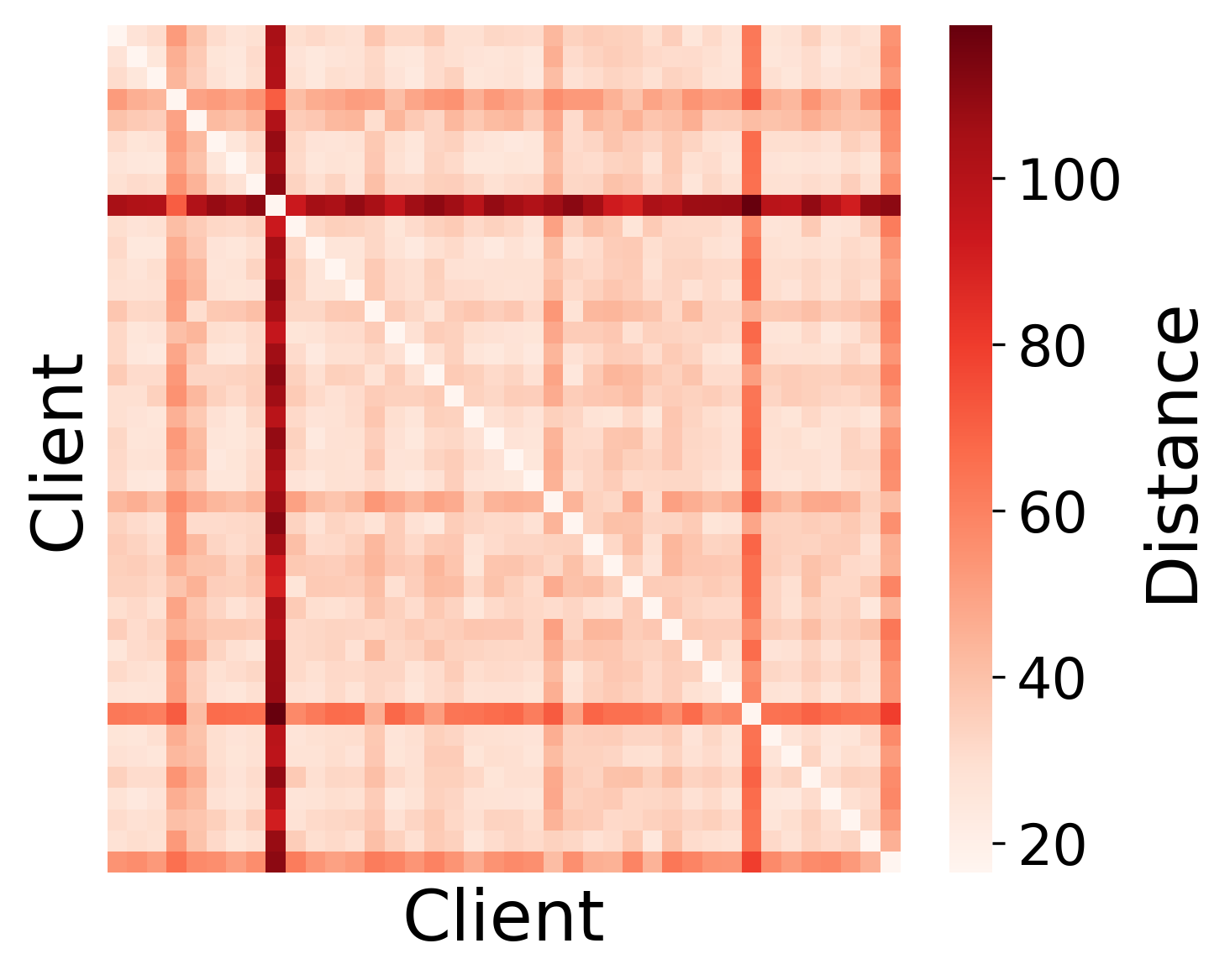}
        \caption{Rot. CIFAR10}
    \end{subfigure}
    % Third subfigure
    \begin{subfigure}[b]{0.18\columnwidth}
        \centering
        \includegraphics[width=\linewidth]{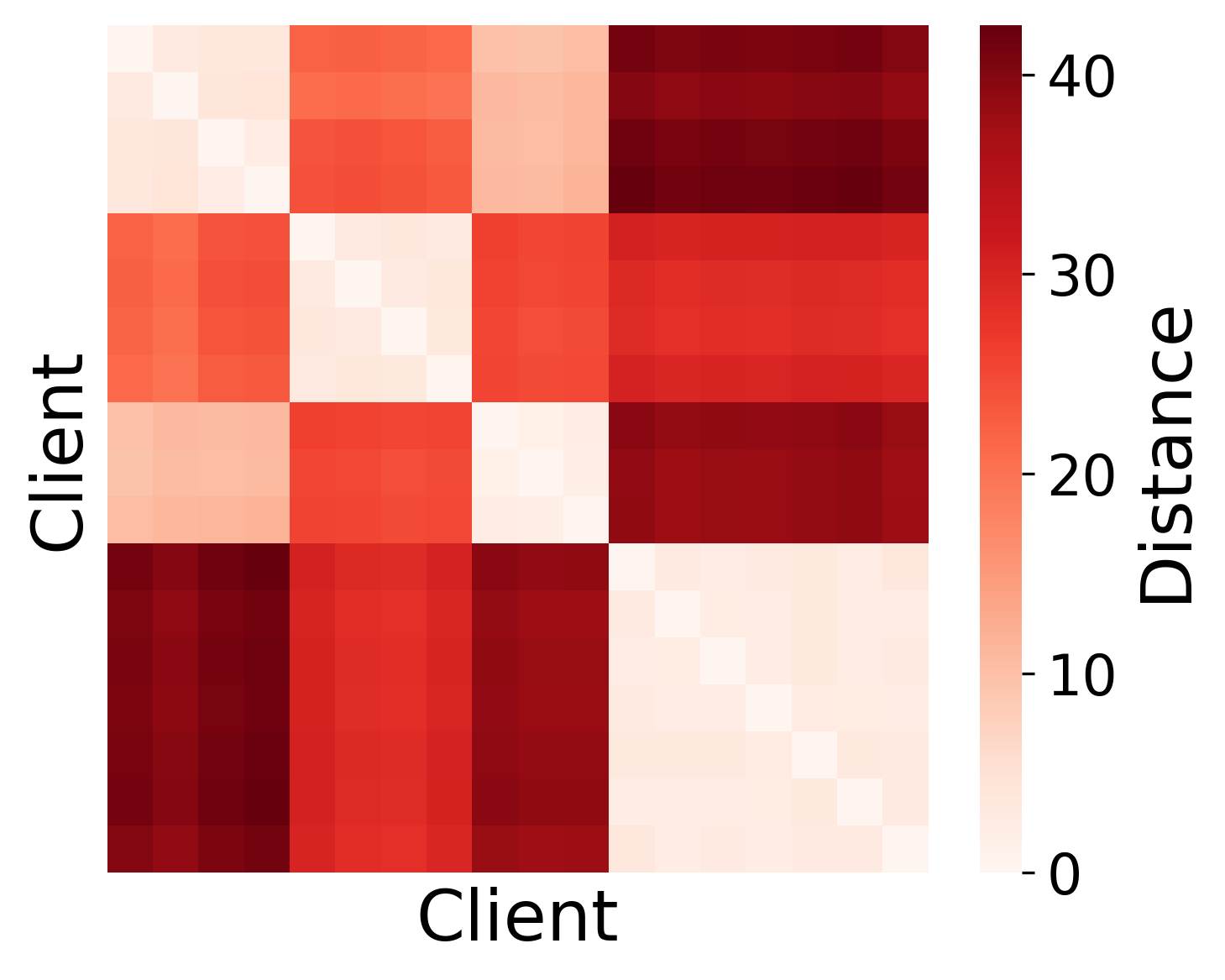}
        \caption{PACS}
    \end{subfigure}
    % Fourth subfigure
    \begin{subfigure}[b]{0.18\columnwidth}
        \centering
        \includegraphics[width=\linewidth]{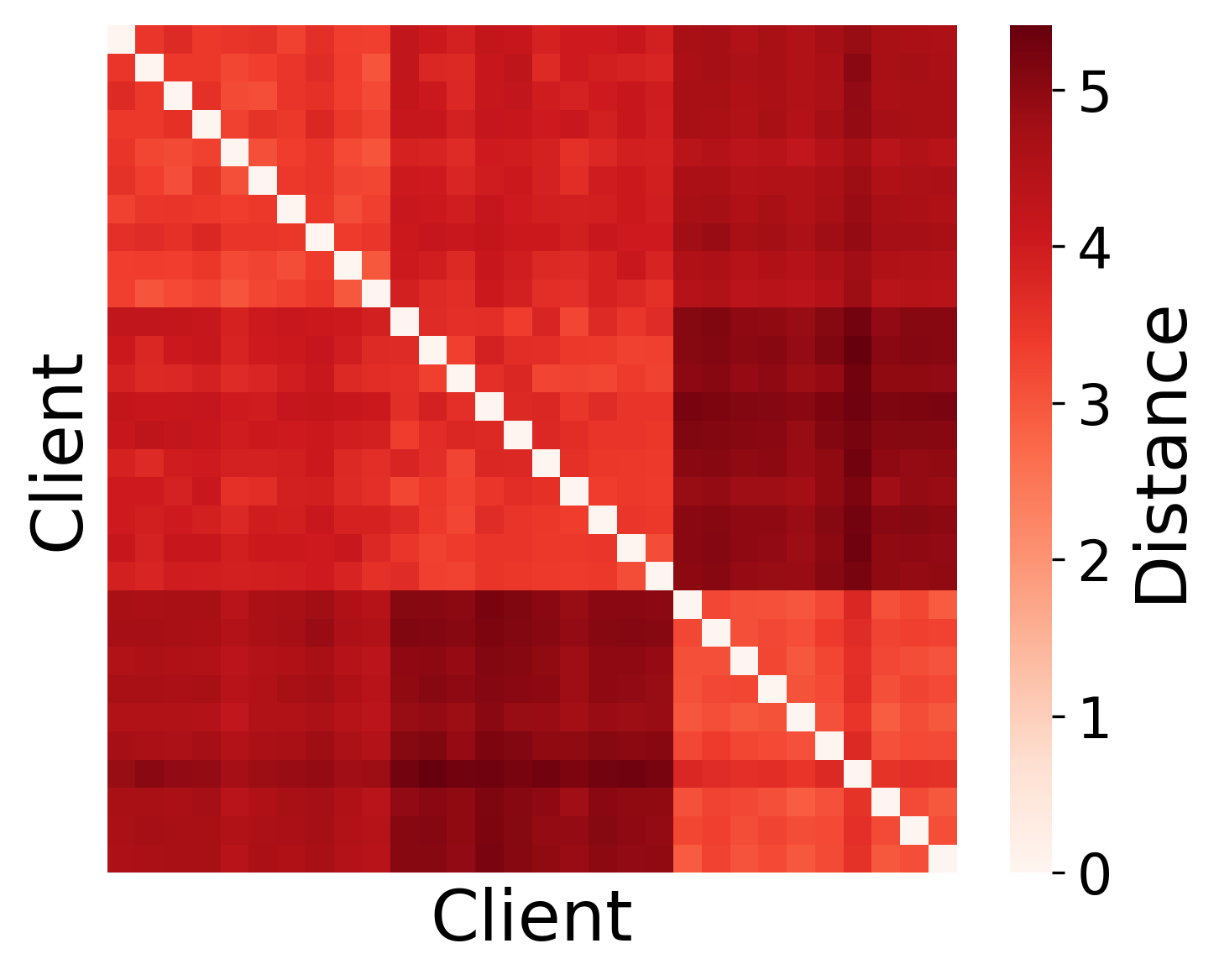}
        \caption{Bd. MNIST}
    \end{subfigure}
    % Fifth subfigure
    \begin{subfigure}[b]{0.18\columnwidth}
        \centering
        \includegraphics[width=\linewidth]{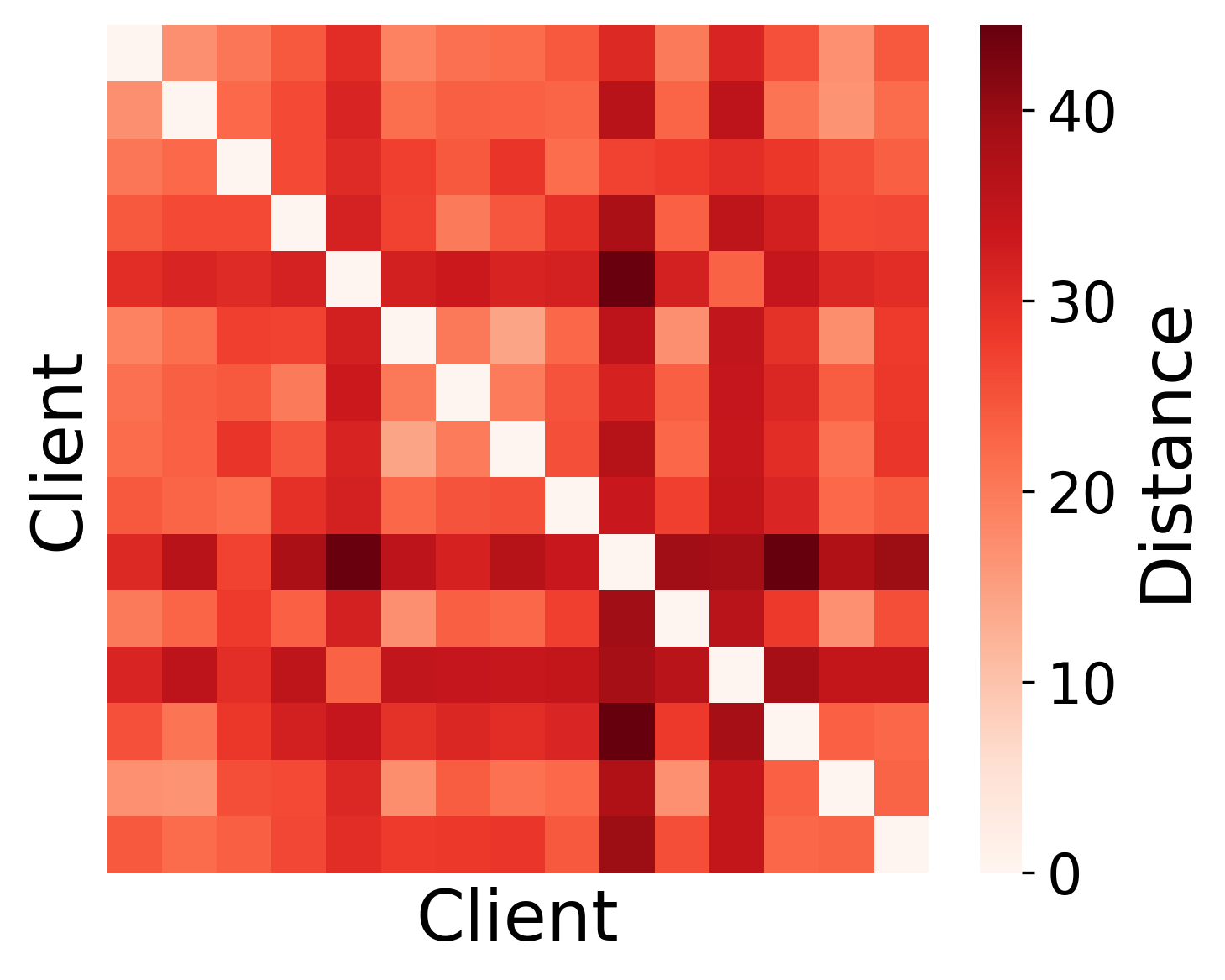}
        \caption{Bd. CIFAR10}
    \end{subfigure}

    \caption{Pairwise L2 distances between model parameters.}
    \label{fig:param_distances}
\end{figure}

\begin{figure}[H]
    \centering
    % First subfigure
    \begin{subfigure}[b]{0.18\columnwidth}
        \centering
        \includegraphics[width=\linewidth]{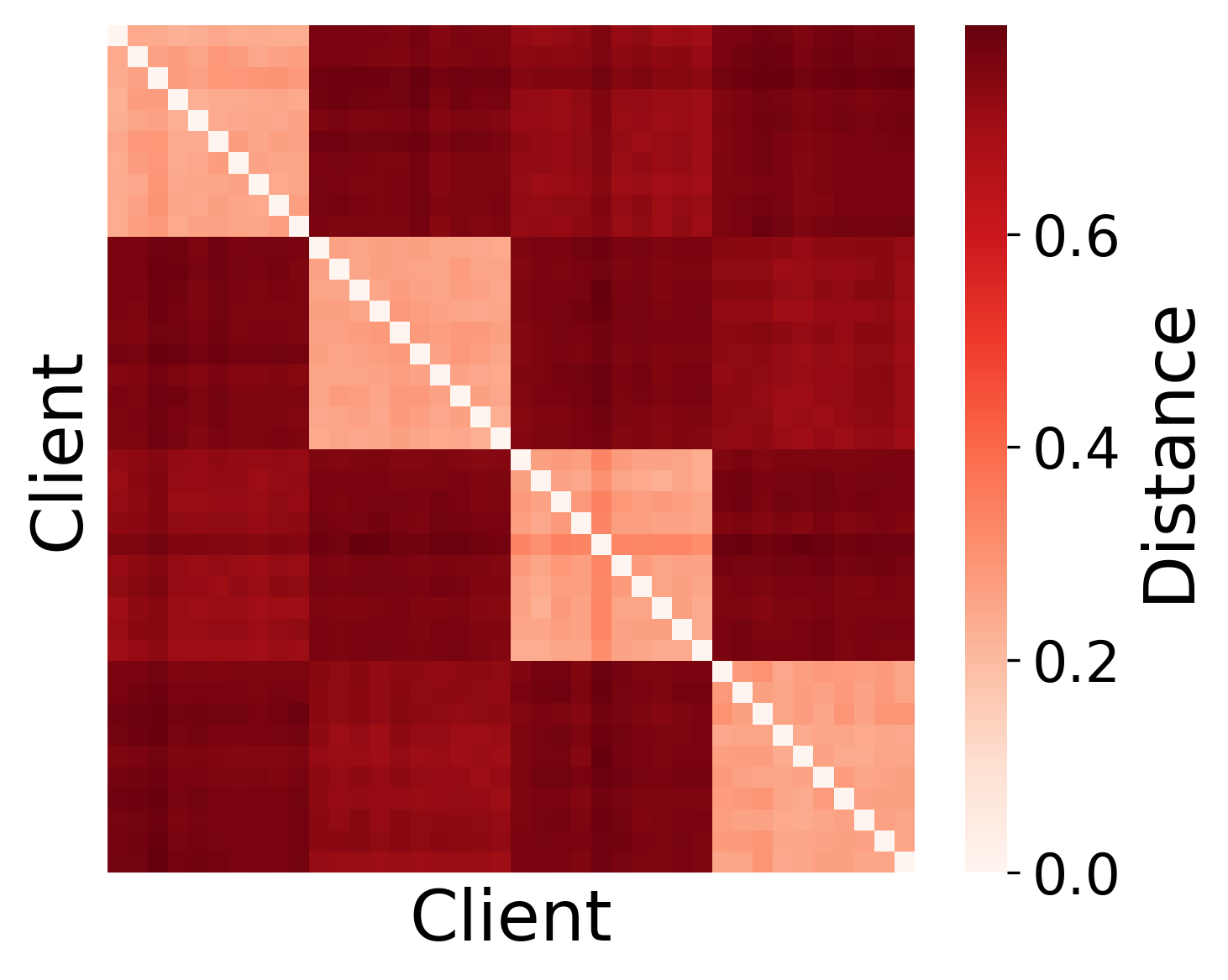}
        \caption{Rot. MNIST}
    \end{subfigure}
    % Second subfigure
    \begin{subfigure}[b]{0.18\columnwidth}
        \centering
        \includegraphics[width=\linewidth]{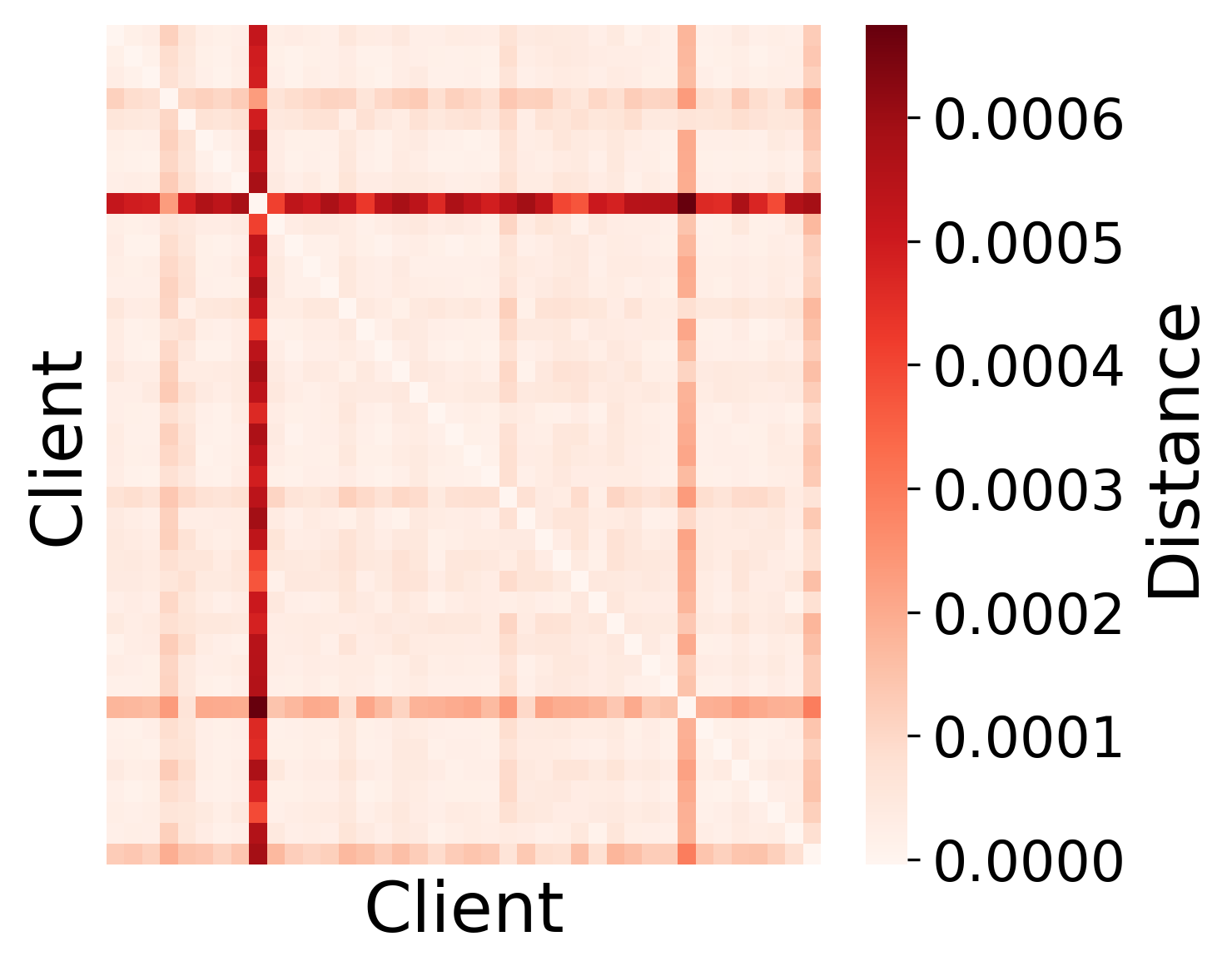}
        \caption{Rot. CIFAR10}
    \end{subfigure}
    % Third subfigure
    \begin{subfigure}[b]{0.18\columnwidth}
        \centering
        \includegraphics[width=\linewidth]{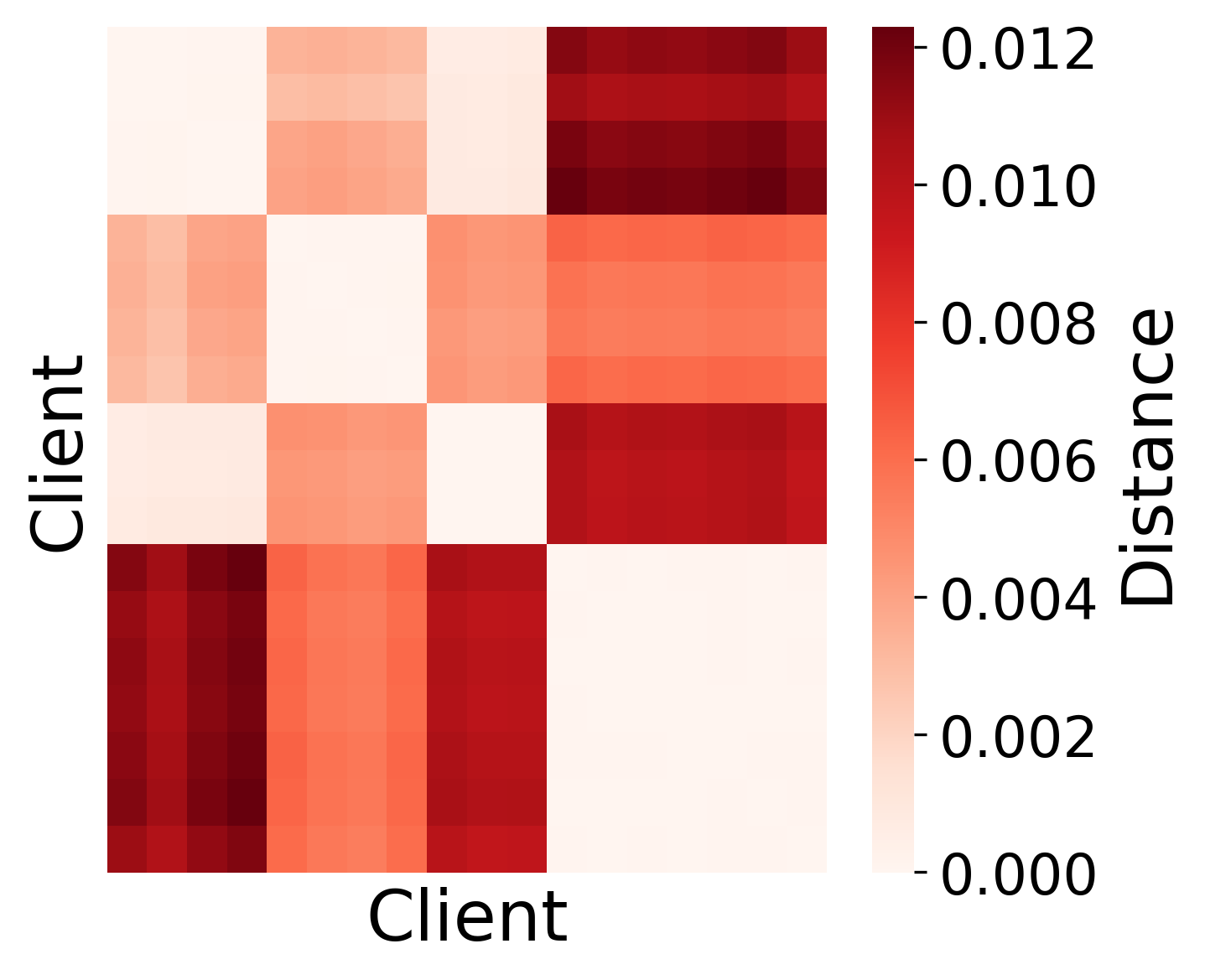}
        \caption{PACS}
    \end{subfigure}
    % Fourth subfigure
    \begin{subfigure}[b]{0.18\columnwidth}
        \centering
        \includegraphics[width=\linewidth]{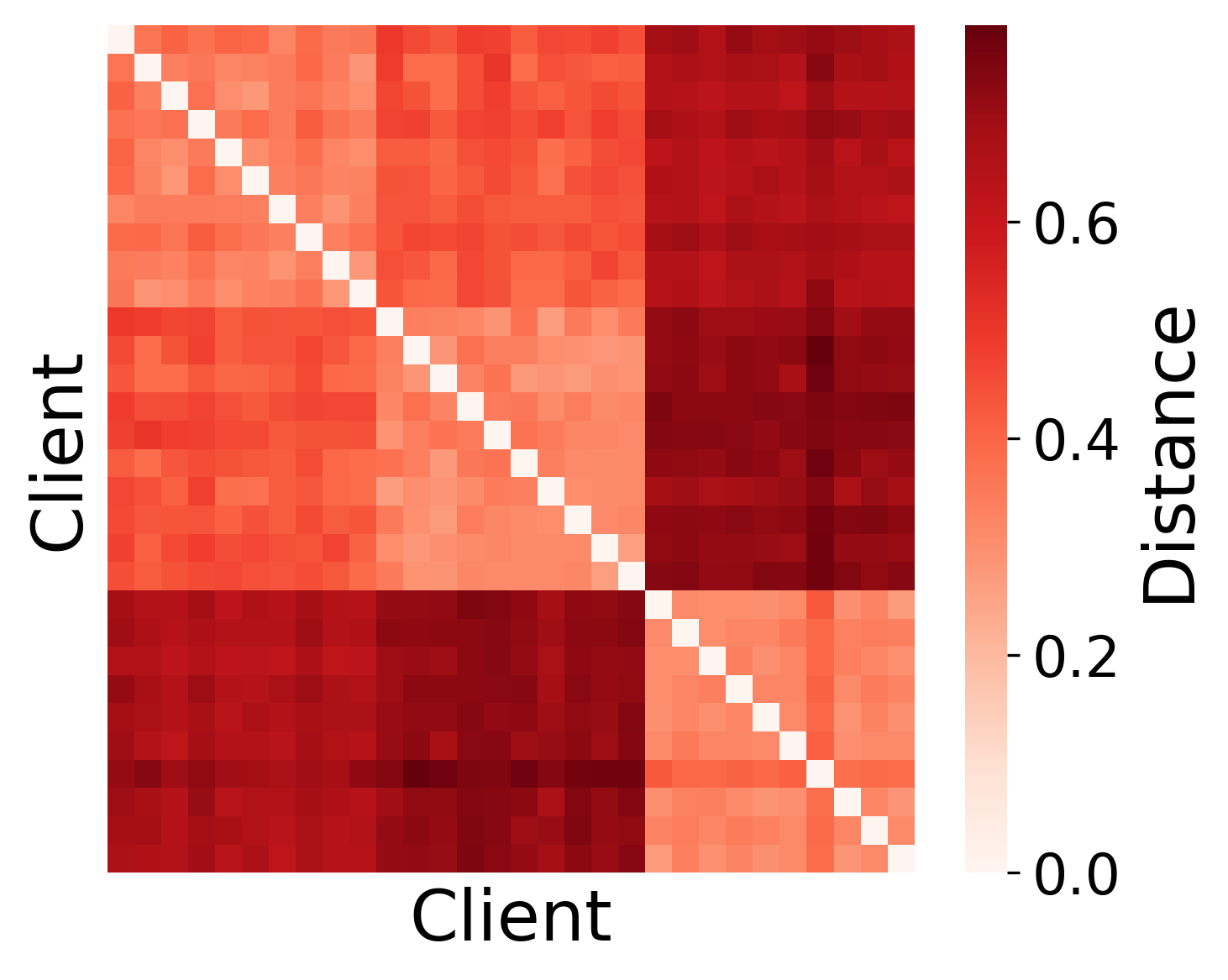}
        \caption{Bd. MNIST}
    \end{subfigure}
    % Fifth subfigure
    \begin{subfigure}[b]{0.18\columnwidth}
        \centering
        \includegraphics[width=\linewidth]{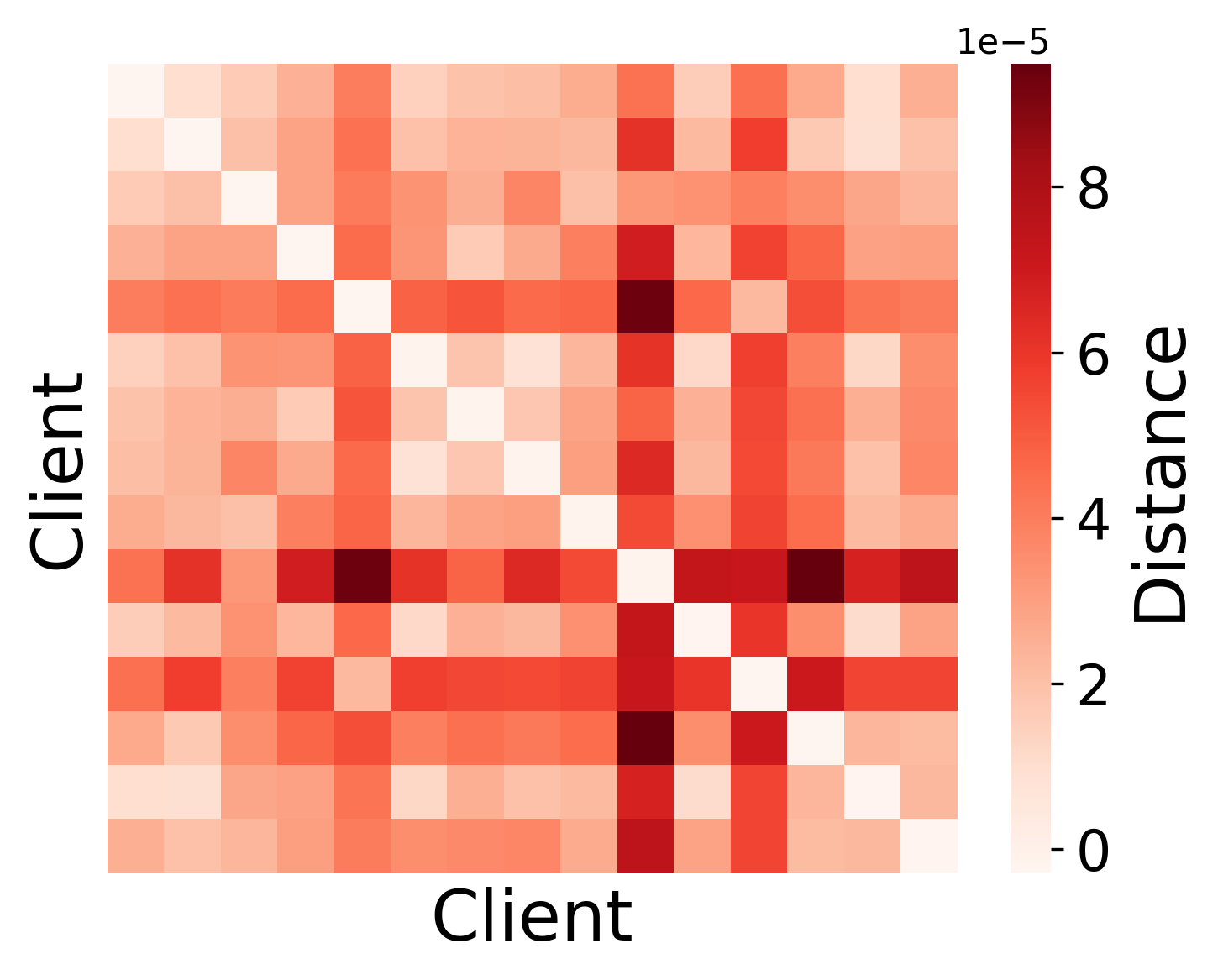}
        \caption{Bd. CIFAR10}
    \end{subfigure}

    \caption{Pairwise cosine dissimilarities between gradients.}
    \label{fig:grad_distances}
\end{figure}

\begin{figure}[H]
    \centering
    % First subfigure
    \begin{subfigure}[b]{0.18\columnwidth}
        \centering
        \includegraphics[width=\linewidth]{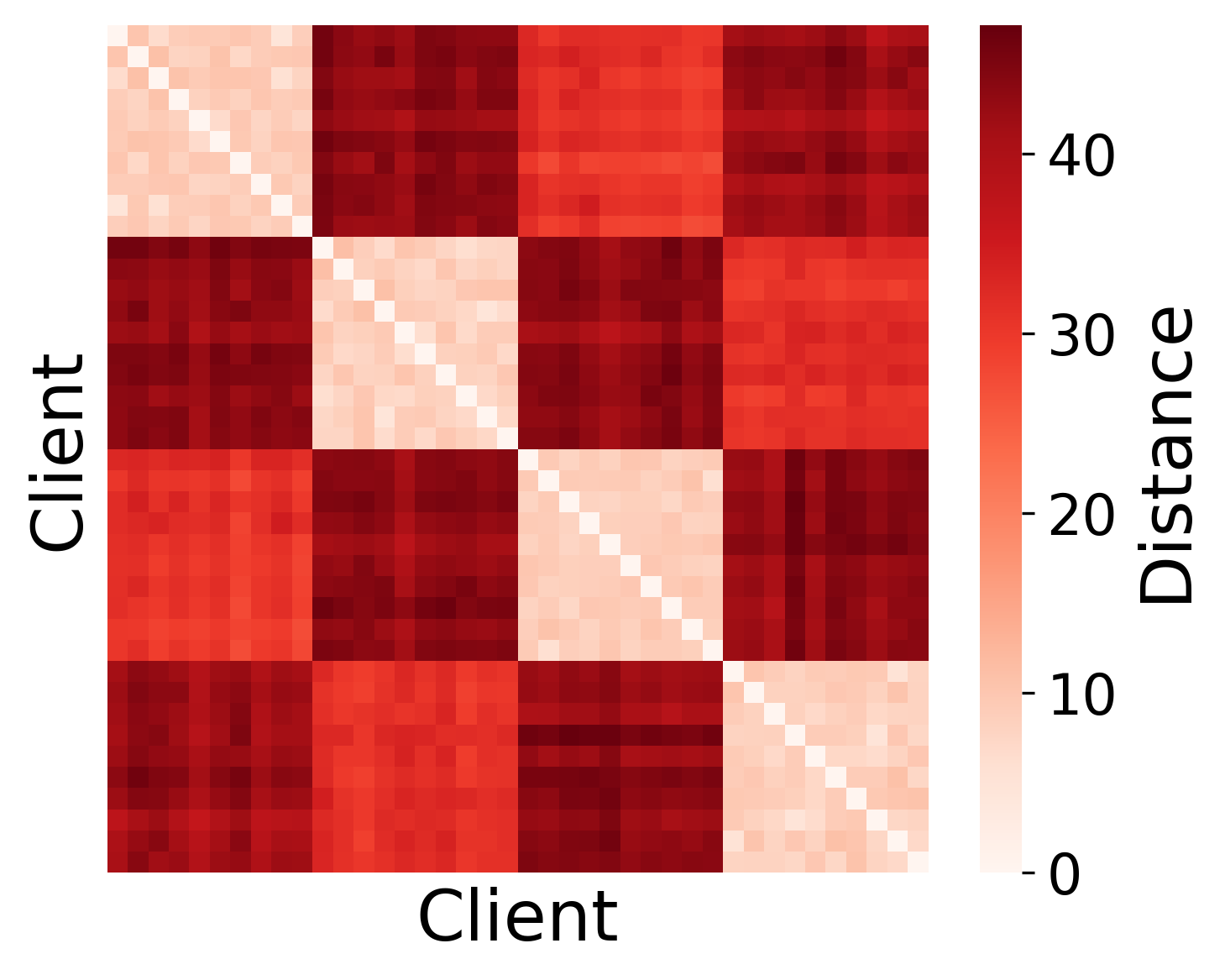}
        \caption{Rot. MNIST}
    \end{subfigure}
    % Second subfigure
    \begin{subfigure}[b]{0.18\columnwidth}
        \centering
        \includegraphics[width=\linewidth]{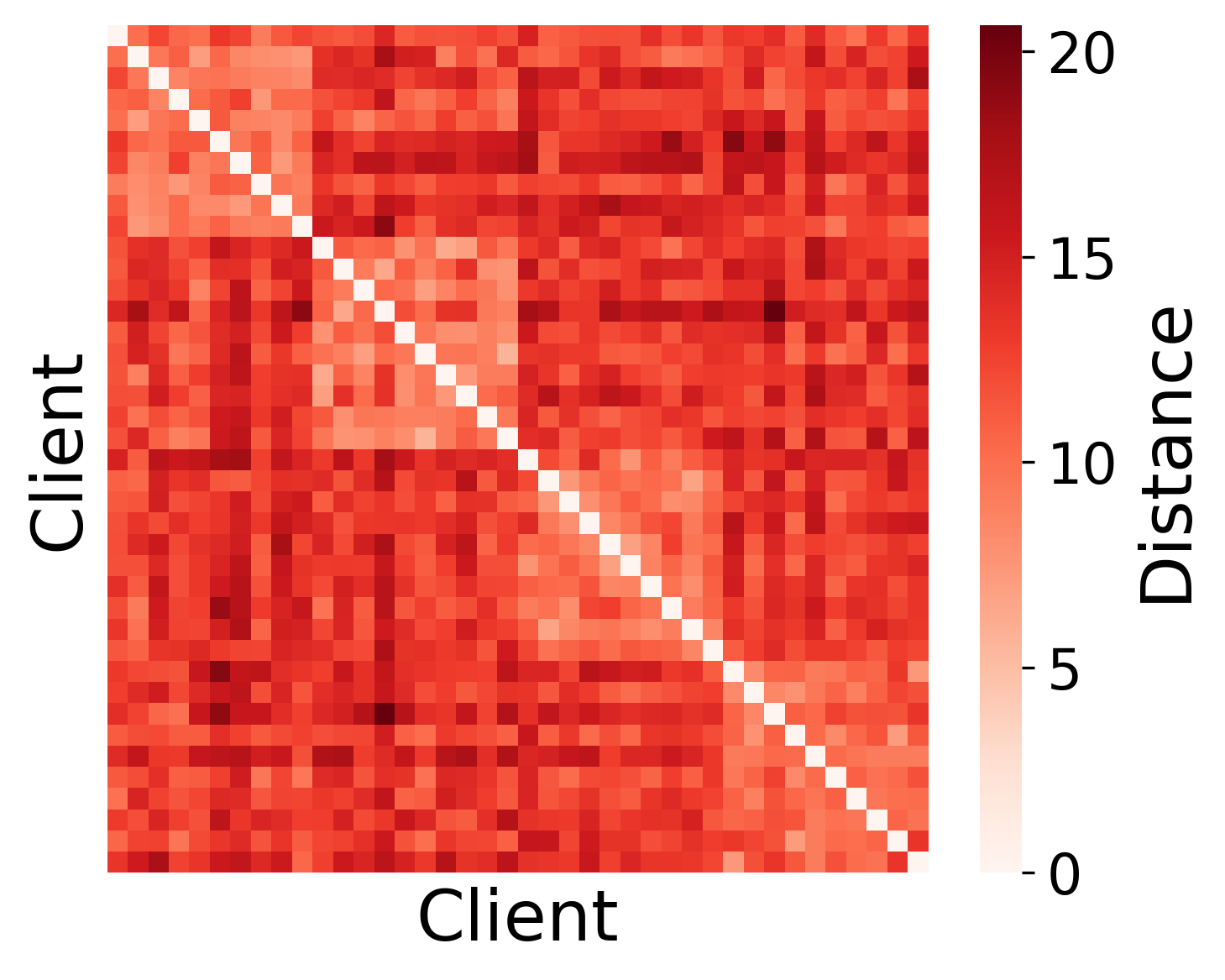}
        \caption{Rot. CIFAR10}
    \end{subfigure}
    % Third subfigure
    \begin{subfigure}[b]{0.18\columnwidth}
        \centering
        \includegraphics[width=\linewidth]{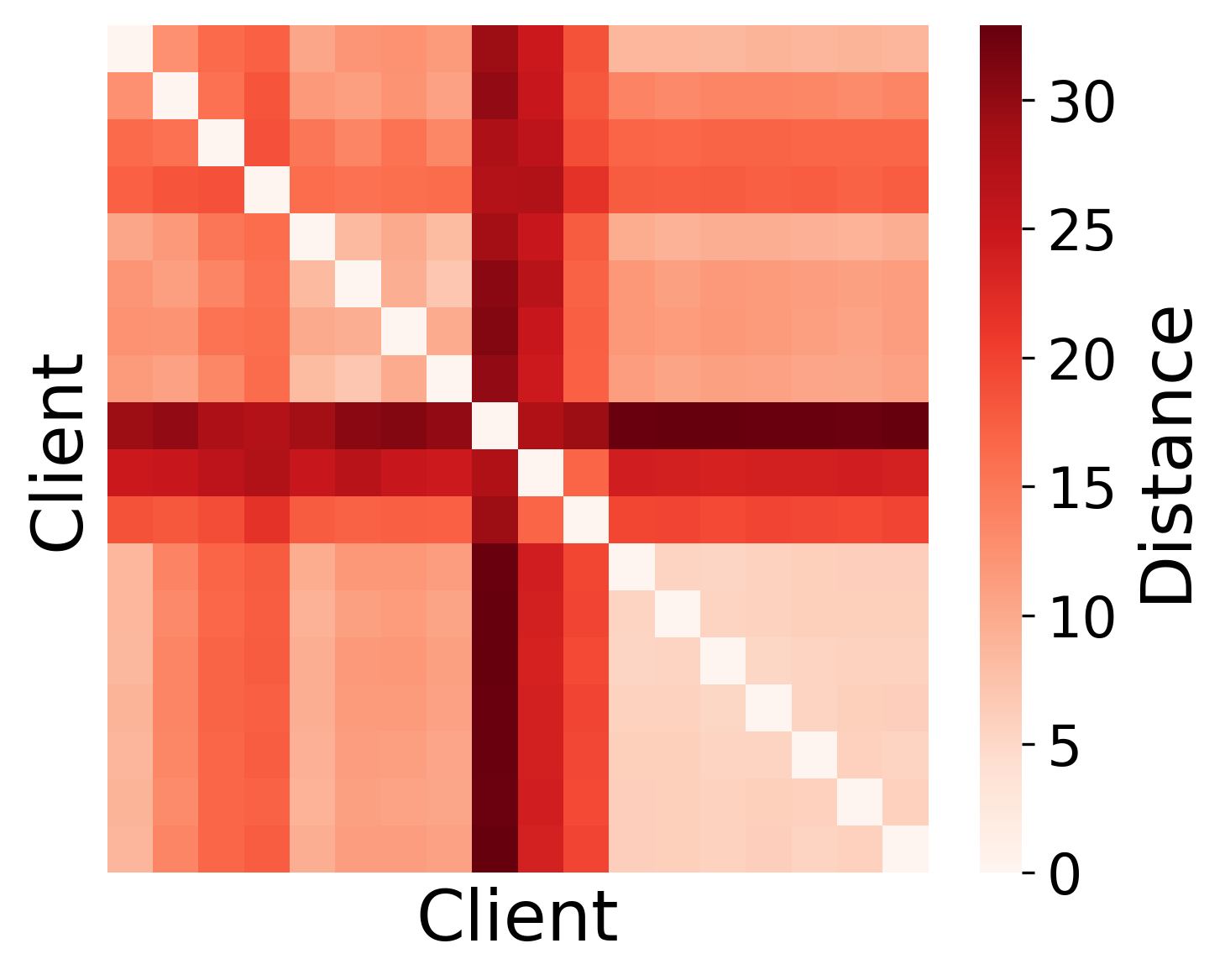}
        \caption{PACS}
    \end{subfigure}
    % Fourth subfigure
    \begin{subfigure}[b]{0.18\columnwidth}
        \centering
        \includegraphics[width=\linewidth]{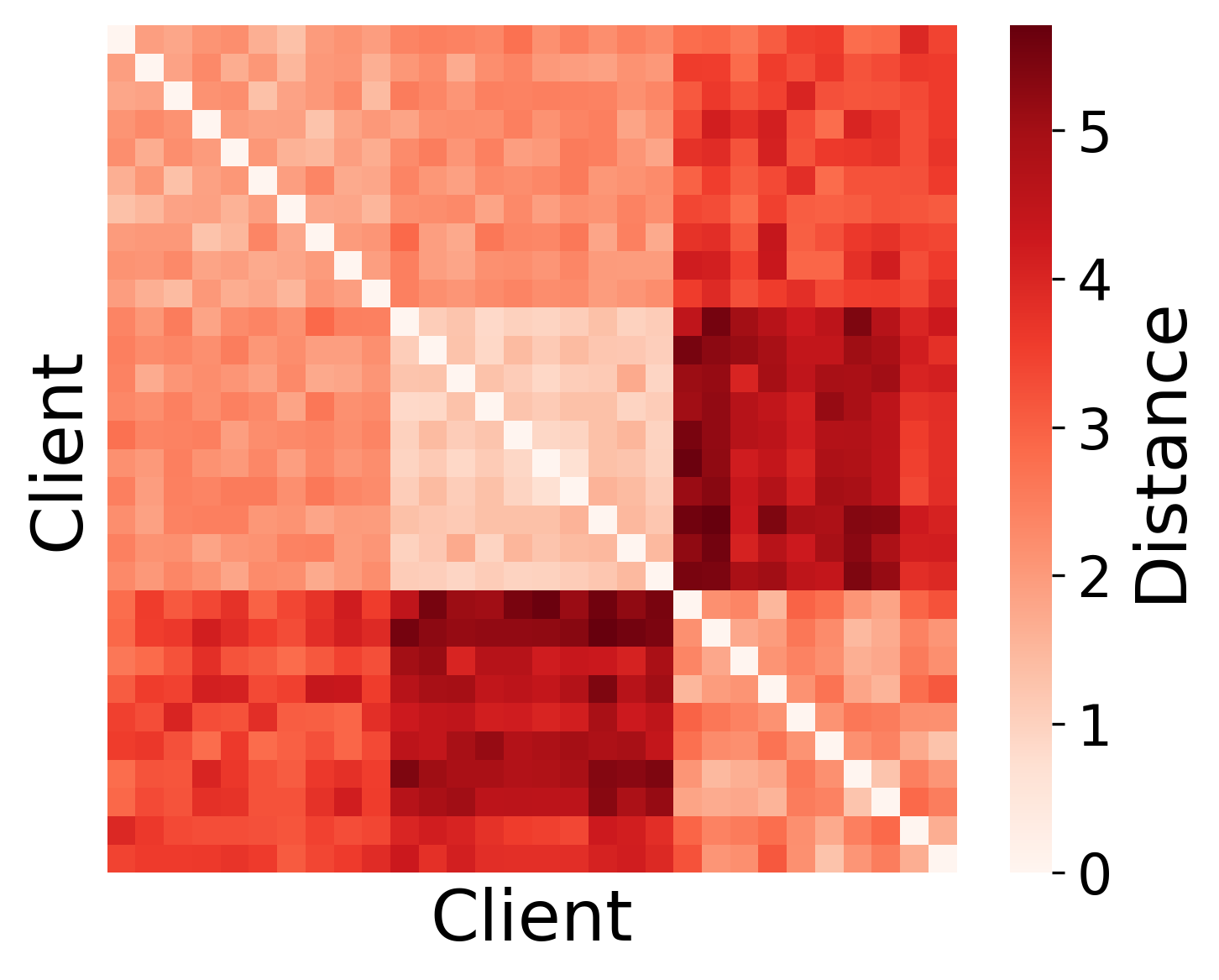}
        \caption{Bd. MNIST}
    \end{subfigure}
    % Fifth subfigure
    \begin{subfigure}[b]{0.18\columnwidth}
        \centering
        \includegraphics[width=\linewidth]{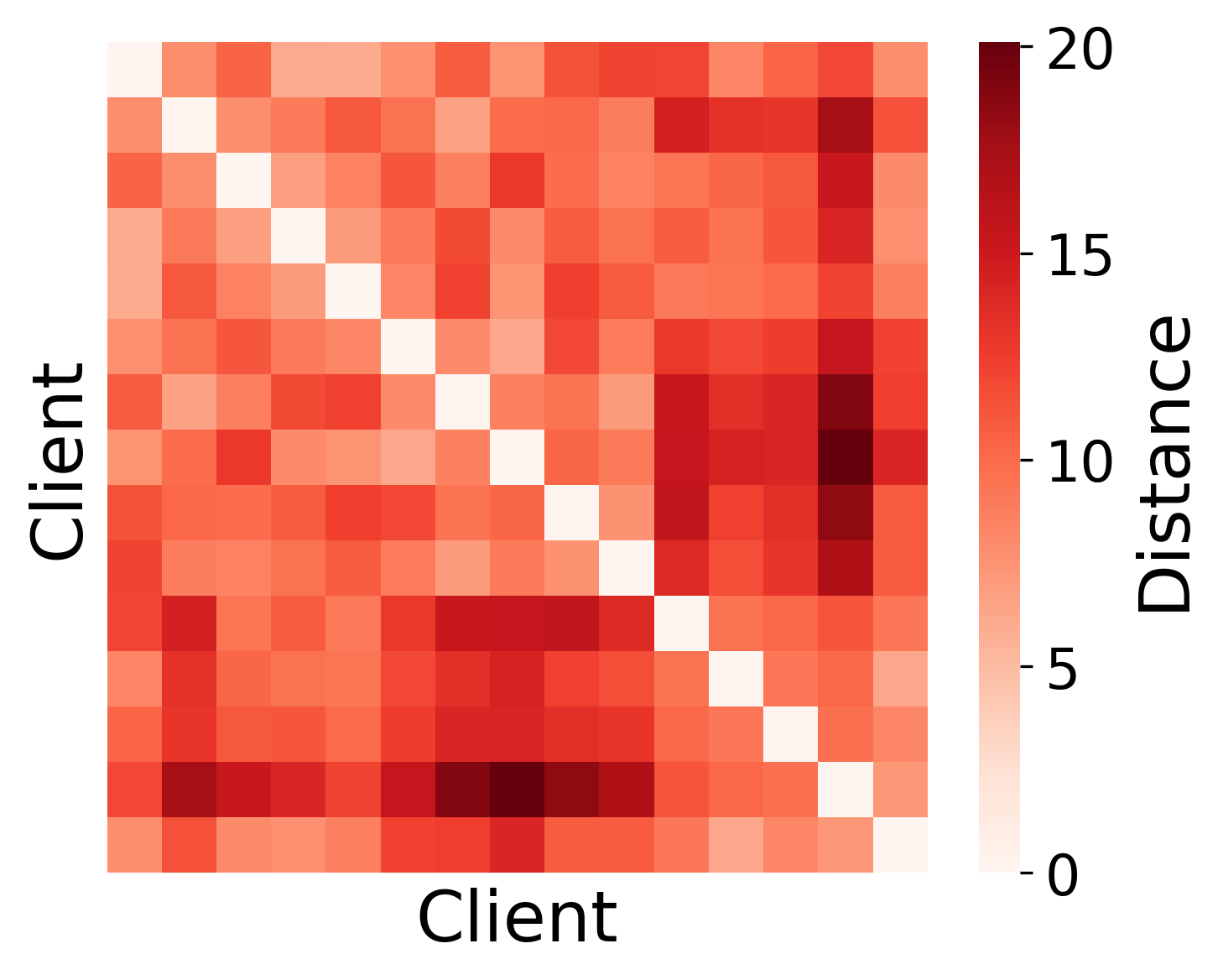}
        \caption{Bd. CIFAR10}
    \end{subfigure}

    \caption{Pairwise cosine distances of principal angles between input data.}
    \label{fig:raw_distances}
\end{figure}

\begin{figure}[H]
    \centering
    % First subfigure
    \begin{subfigure}[b]{0.18\columnwidth}
        \centering
        \includegraphics[width=\linewidth]{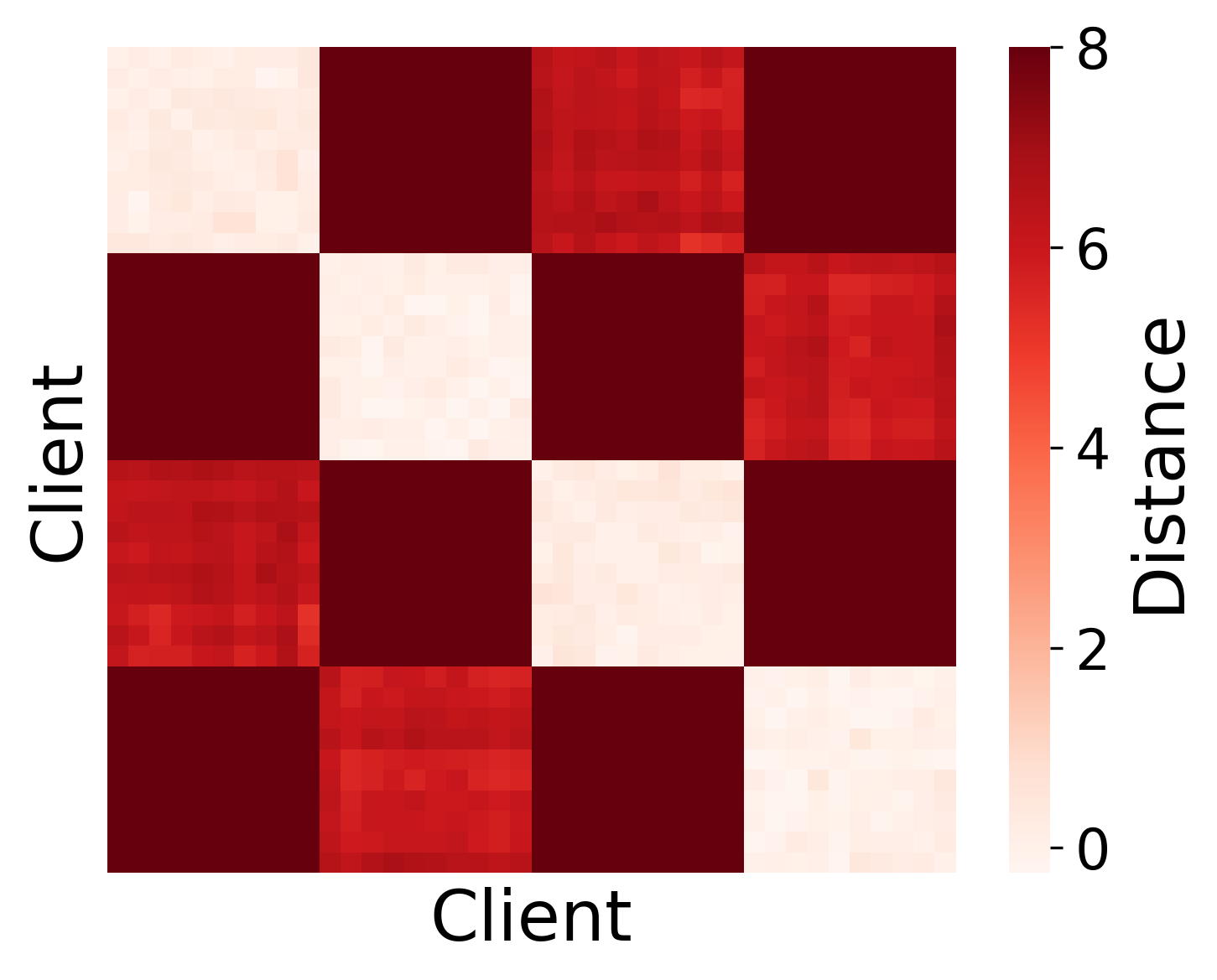}
        \caption{Rot. MNIST}
    \end{subfigure}
    % Second subfigure
    \begin{subfigure}[b]{0.18\columnwidth}
        \centering
        \includegraphics[width=\linewidth]{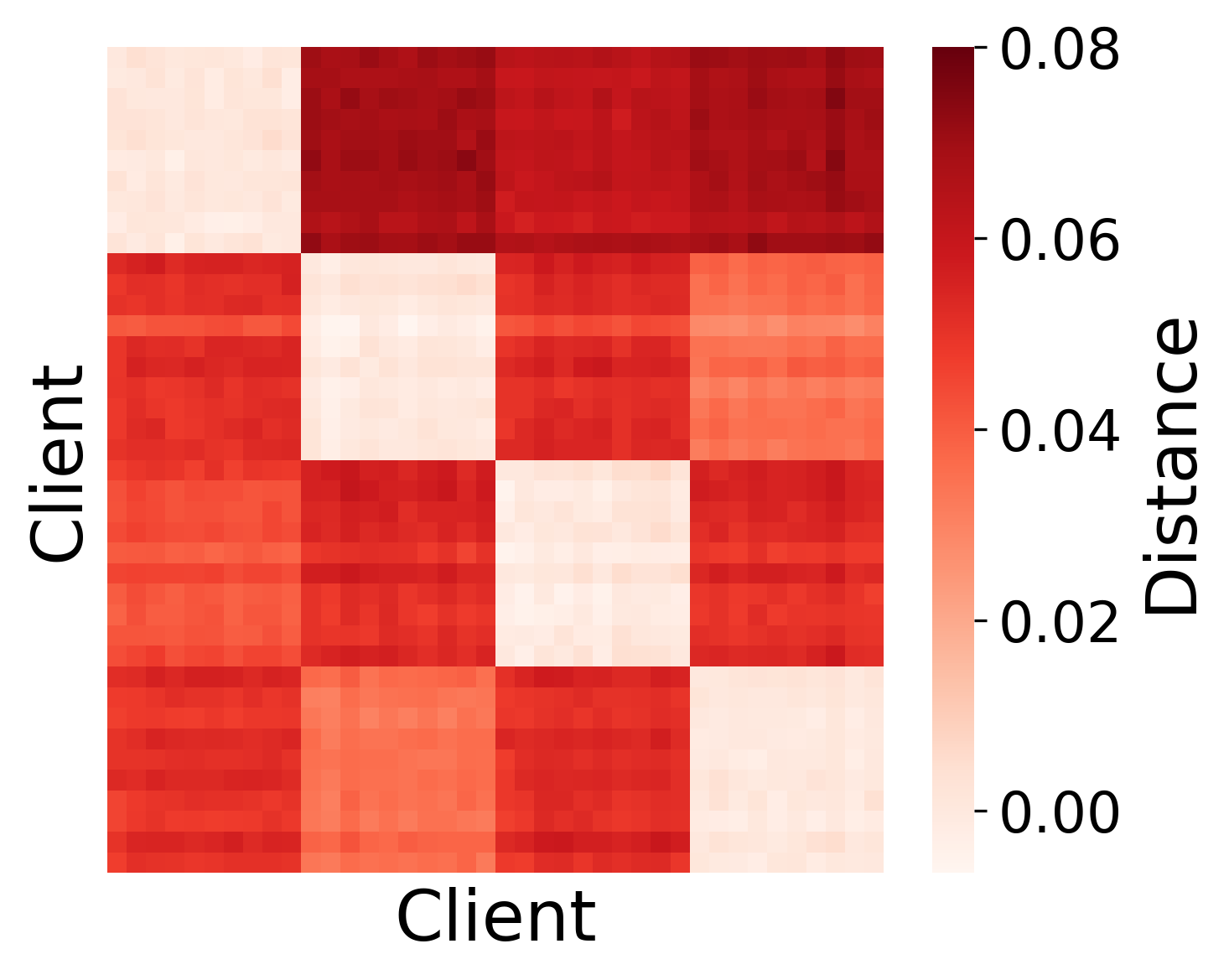}
        \caption{Rot. CIFAR10}
    \end{subfigure}
    % Third subfigure
    \begin{subfigure}[b]{0.18\columnwidth}
        \centering
        \includegraphics[width=\linewidth]{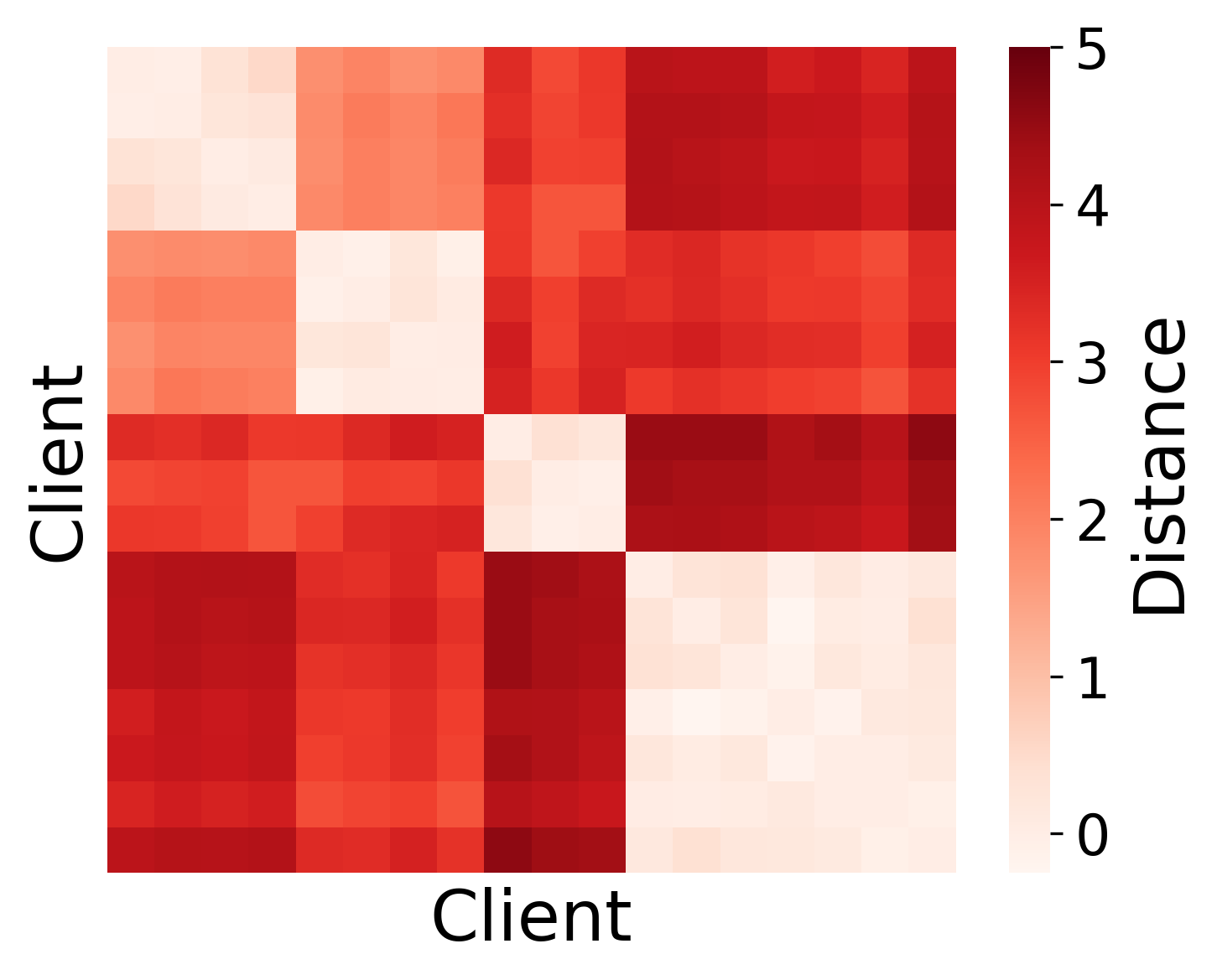}
        \caption{PACS}
    \end{subfigure}
    % Fourth subfigure
    \begin{subfigure}[b]{0.18\columnwidth}
        \centering
        \includegraphics[width=\linewidth]{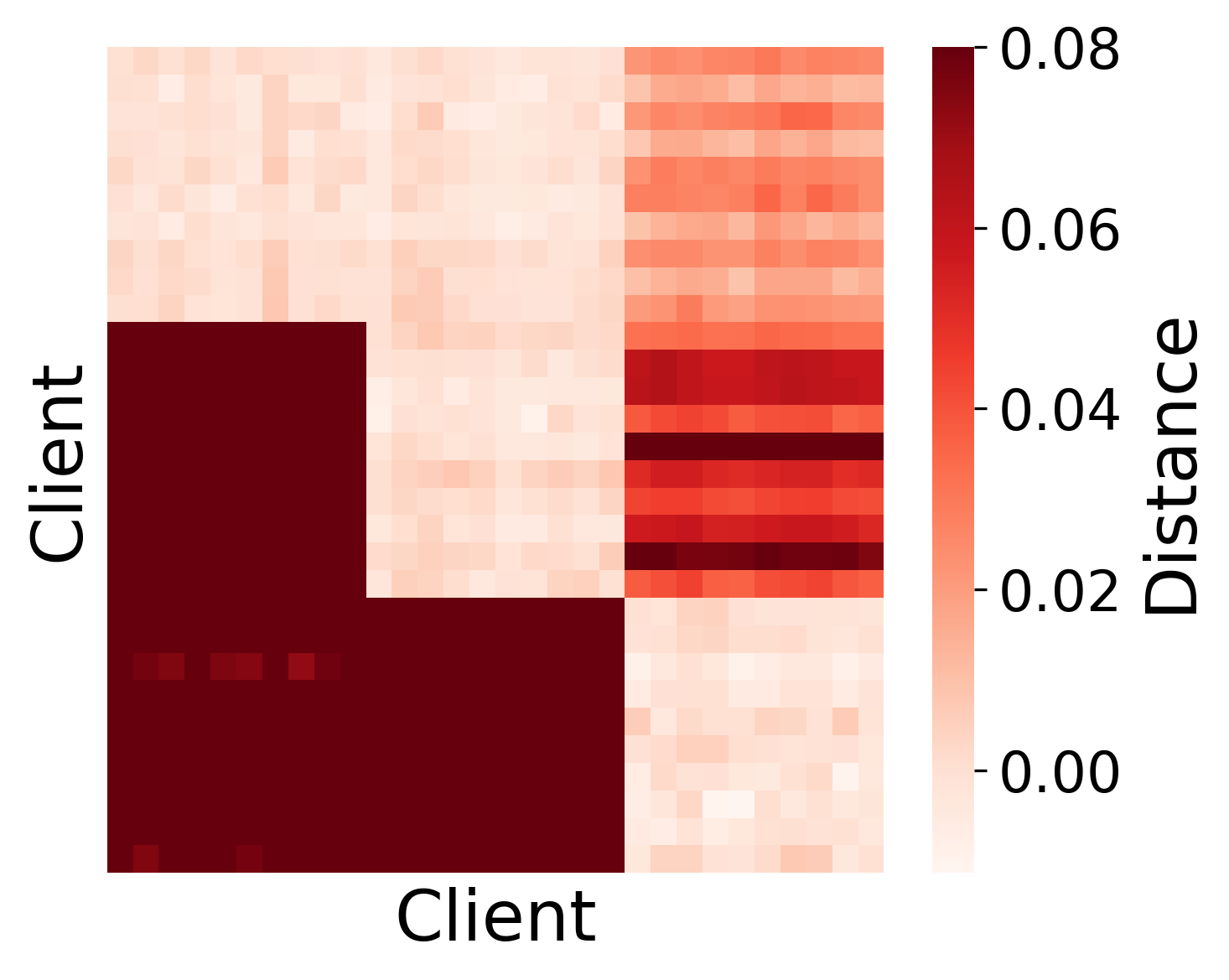}
        \caption{Bd. MNIST}
    \end{subfigure}
    % Fifth subfigure
    \begin{subfigure}[b]{0.18\columnwidth}
        \centering
        \includegraphics[width=\linewidth]{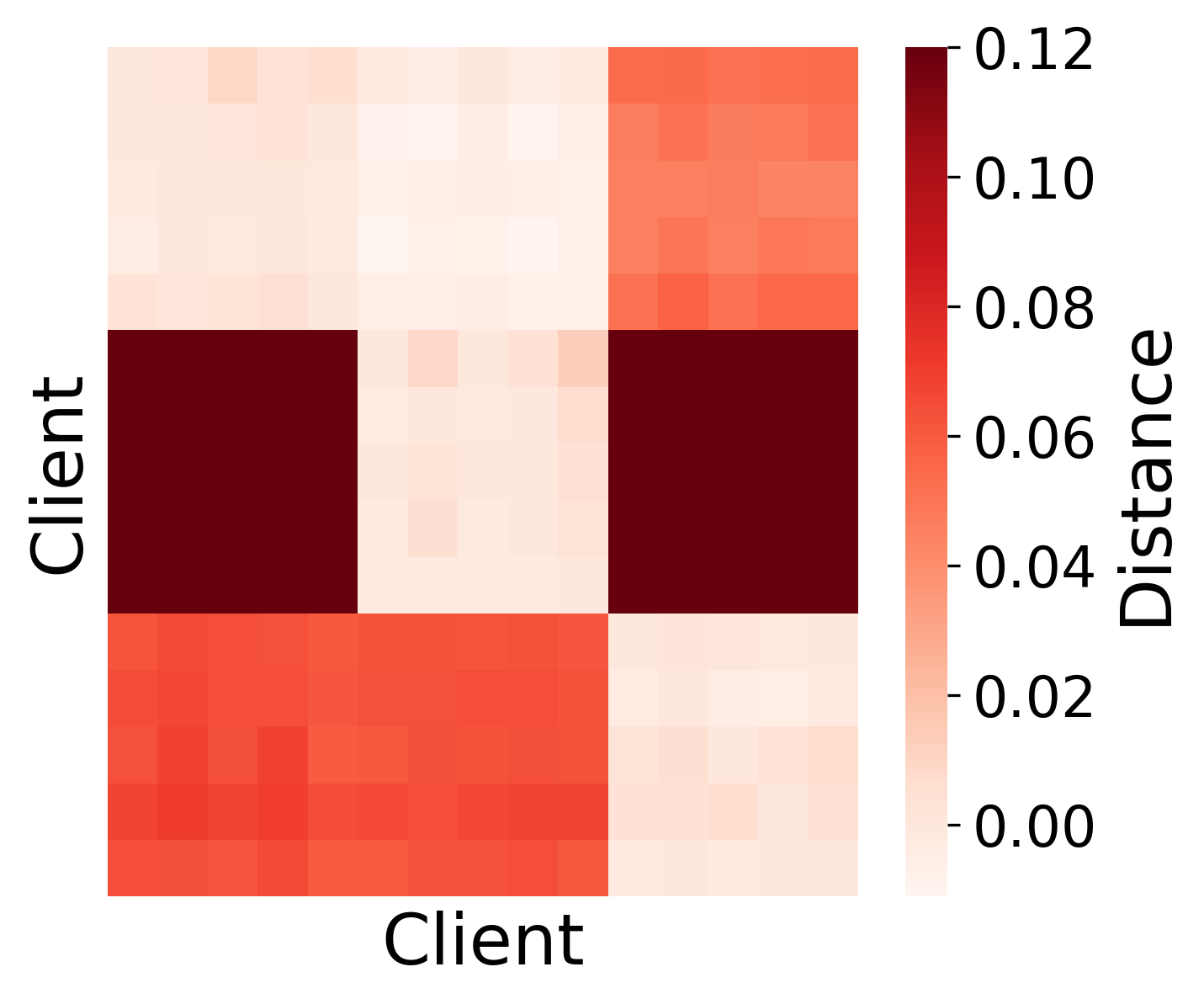}
        \caption{Bd. CIFAR10}
    \end{subfigure}

    \caption{Pairwise EMDs between embedding distributions.}
    \label{fig:emds}
\end{figure}

\subsection{Comparison of Parameter Distances}
\label{appx:cluster_total_param_distances}
We show in Table~\ref{tab:distance_comparison} that the average parameter distance within the clusters we identify is less than the average parameter if we do not cluster. While our theoretical results were specifically about the pairwise distance for the hypothesis parameters, we empirically find that we also obtain a decrease when considering the full set of model parameters. 

\begin{table}[htbp]
\centering
\caption{EMD-CFL identifies clusters with smaller average pairwise L2-distances between parameters (considering only the hypothesis as well as the full model) compared to the overall average.}
\label{tab:distance_comparison}
\begin{tabular}{lcccc}
\toprule
 & \multicolumn{2}{c}{\textbf{Hypothesis Only}} & \multicolumn{2}{c}{\textbf{Full Model}} \\ \cmidrule(lr){2-3} \cmidrule(lr){4-5}
\textbf{Dataset} & \textbf{Clusters} & \textbf{Overall} & \textbf{Clusters} & \textbf{Overall} \\ 
\midrule
Rotated MNIST     & 3.03 & 4.60 & 3.06  & 4.69  \\ 
Rotated CIFAR10   & 1.85 & 2.36 & 39.13 & 39.40 \\ 
PACS              & 0.31 & 0.57 & 2.73  & 23.98 \\ 
Backdoor MNIST    & 3.32 & 3.83 & 3.37  & 4.17  \\ 
Backdoor CIFAR10  & 0.93 & 2.27 & 26.93 & 27.62 \\ 
\bottomrule
\end{tabular}
\end{table}

\subsection{Induced Clustering}
\label{appx:induced}

For both MNIST and CIFAR10, we consider two additional versions with rotations of $\{-3, -1, 1, 3\}$ resulting in $K=1$ minimally varying cluster (Min) and rotations of $\{-3, 3, 177, 183\}$ for $K=2$ main clusters (Two). We also include the original unrotated data (Base). Note that for the baselines that require the number of clusters as an input, we choose $K=2$, since otherwise the baselines would collapse to FedAvg.

Tables~\ref{tab:mnist_extended} and \ref{tab:cifar_extended} show that our method is able to recover the correct clusters and matches the Oracle performance on all three versions of the MNIST dataset. Notably, our method correctly returns a single cluster when the data does not have a clear clustering structure as in the Base case.

\begin{table}[ht]
\centering
\caption{Comparison across additional MNIST experiments.}
\label{tab:mnist_extended}
\resizebox{\columnwidth}{!}{%
\begin{tabular}{lcccccccccc}
\toprule
 & \multicolumn{3}{c}{\textbf{Min}} & \multicolumn{3}{c}{\textbf{Two}} & \multicolumn{3}{c}{\textbf{Base}} \\
\cmidrule(lr){2-4} \cmidrule(lr){5-7} \cmidrule(lr){8-10}
\textbf{Method} & \textbf{Avg Acc} & \textbf{Worst Acc} & \textbf{ARI} & \textbf{Avg Acc} & \textbf{Worst Acc} & \textbf{ARI} & \textbf{Avg Acc} & \textbf{Worst Acc} & \textbf{ARI} \\
\midrule
Oracle & 98.84$\pm$0.03 & 97.58$\pm$0.26 & 1.00$\pm$0.00 & 98.87$\pm$0.02 & 97.73$\pm$0.45 & 1.00$\pm$0.00 & 98.87$\pm$0.06 & 97.58$\pm$0.26 & 1.00$\pm$0.00 \\ \cmidrule(lr){1-10}
EMD-CFL & 98.85$\pm$0.02 & 97.42$\pm$0.26 & 1.00$\pm$0.00 & 98.86$\pm$0.06 & 97.27$\pm$0.00 & 1.00$\pm$0.00 & 98.86$\pm$0.05 & 97.58$\pm$0.26 & 1.00$\pm$0.00 \\
CFL & 98.85$\pm$0.02 & 97.58$\pm$0.26 & 0.67$\pm$0.47 & 94.83$\pm$4.02 & 89.39$\pm$8.41 & 0.33$\pm$0.47 & 98.89$\pm$0.06 & 97.58$\pm$0.26 & 0.67$\pm$0.47 \\
PACFL & 98.85$\pm$0.02 & 97.58$\pm$0.26 & 1.00$\pm$0.00 & 98.87$\pm$0.01 & 97.73$\pm$0.45 & 1.00$\pm$0.00 & 98.89$\pm$0.07 & 97.58$\pm$0.26 & 1.00$\pm$0.00 \\
FedClust & 97.08$\pm$0.18 & 95.30$\pm$0.26 & 1.00$\pm$0.00 & 97.04$\pm$0.16 & 95.15$\pm$0.52 & 1.00$\pm$0.00 & 97.09$\pm$0.20 & 95.15$\pm$0.52 & 1.00$\pm$0.00 \\
IFCA & 98.84$\pm$0.06 & 97.42$\pm$0.52 & 0.00$\pm$0.00 & 98.89$\pm$0.03 & 97.88$\pm$0.26 & 1.00$\pm$0.00 & 98.87$\pm$0.01 & 97.42$\pm$0.52 & 0.00$\pm$0.00 \\
FlexCFL & 98.91$\pm$0.02 & 96.97$\pm$0.69 & 0.00$\pm$0.00 & 98.69$\pm$0.10 & 96.97$\pm$0.52 & 1.00$\pm$0.00 & 98.88$\pm$0.01 & 96.82$\pm$0.45 & 0.00$\pm$0.00 \\
FeSEM & 97.44$\pm$0.06 & 95.45$\pm$0.79 & 0.00$\pm$0.00 & 95.56$\pm$3.38 & 93.18$\pm$4.72 & 0.66$\pm$0.48 & 97.45$\pm$0.04 & 94.85$\pm$1.05 & 0.00$\pm$0.00 \\
CFLGP & 98.87$\pm$0.04 & 97.73$\pm$0.00 & 0.00$\pm$0.00 & 98.83$\pm$0.06 & 97.42$\pm$0.26 & 1.00$\pm$0.00 & 98.81$\pm$0.04 & 97.58$\pm$0.26 & 0.00$\pm$0.00 \\
FedEM & 98.45$\pm$0.07 & 96.82$\pm$0.45 & - & 97.65$\pm$0.69 & 94.55$\pm$1.98 & - & 98.47$\pm$0.11 & 96.82$\pm$0.45 & - \\
FedSoft & 98.73$\pm$0.12 & 97.27$\pm$0.45 & - & 98.71$\pm$0.05 & 97.12$\pm$0.26 & - & 98.73$\pm$0.11 & 96.97$\pm$0.26 & - \\
FedRC & 98.46$\pm$0.11 & 96.97$\pm$0.52 & - & 97.87$\pm$0.34 & 95.00$\pm$0.79 & - & 98.45$\pm$0.07 & 96.67$\pm$0.52 & - \\ %\cmidrule(lr){1-10}
FedCE & 98.87$\pm$0.06 & 97.58$\pm$0.26 & - & 98.77$\pm$0.07 & 97.12$\pm$0.52 & - & 98.87$\pm$0.09 & 97.58$\pm$0.26 & - \\ %\cmidrule(lr){1-10}
FedAvg & 98.86$\pm$0.02 & 97.58$\pm$0.26 & - & 91.73$\pm$3.14 & 82.58$\pm$6.17 & - & 98.87$\pm$0.07 & 97.58$\pm$0.26 & - \\
FedProx & 98.72$\pm$0.02 & 97.58$\pm$0.26 & - & 94.36$\pm$0.37 & 91.36$\pm$0.00 & - & 98.72$\pm$0.06 & 97.58$\pm$0.26 & - \\
pFedGraph & 98.62$\pm$0.07 & 97.27$\pm$0.00 & - & 98.42$\pm$0.05 & 96.82$\pm$0.45 & - & 98.64$\pm$0.12 & 97.12$\pm$0.26 & - \\
FedSaC & 98.74$\pm$0.06 & 97.42$\pm$0.26 & - & 98.21$\pm$0.37 & 95.00$\pm$2.76 & - & 98.78$\pm$0.03 & 97.58$\pm$0.26 & - \\
\bottomrule
\end{tabular}%
}
\end{table}

\begin{table}[ht]
\centering
\caption{Comparison across additional CIFAR10 experiments.}
\label{tab:cifar_extended}
\resizebox{\columnwidth}{!}{%
\begin{tabular}{lcccccccccc}
\toprule
 & \multicolumn{3}{c}{\textbf{Min}} & \multicolumn{3}{c}{\textbf{Two}} & \multicolumn{3}{c}{\textbf{Base}} \\
\cmidrule(lr){2-4} \cmidrule(lr){5-7} \cmidrule(lr){8-10}
\textbf{Method} & \textbf{Avg Acc} & \textbf{Worst Acc} & \textbf{ARI} & \textbf{Avg Acc} & \textbf{Worst Acc} & \textbf{ARI} & \textbf{Avg Acc} & \textbf{Worst Acc} & \textbf{ARI} \\
\midrule
Oracle & 95.68$\pm$0.03 & 93.73$\pm$0.40 & 1.00$\pm$0.00 & 94.63$\pm$0.03 & 92.60$\pm$0.20 & 1.00$\pm$0.00 & 96.34$\pm$0.05 & 95.17$\pm$0.15 & 1.00$\pm$0.00 \\ \cmidrule(lr){1-10}
EMD-CFL & 95.72$\pm$0.11 & 93.50$\pm$0.28 & 1.00$\pm$0.00 & 94.62$\pm$0.27 & 92.65$\pm$0.64 & 1.00$\pm$0.00 & 96.26$\pm$0.10 & 95.10$\pm$0.14 & 1.00$\pm$0.00 \\
CFL & 95.79$\pm$0.17 & 93.80$\pm$0.00 & 1.00$\pm$0.00 & 92.48$\pm$1.12 & 90.35$\pm$1.06 & 0.00$\pm$0.00 & 96.23$\pm$0.03 & 95.10$\pm$0.00 & 1.00$\pm$0.00 \\
PACFL & 95.53$\pm$0.03 & 93.13$\pm$0.50 & 0.00$\pm$0.00 & 94.01$\pm$0.07 & 92.13$\pm$0.15 & 0.42$\pm$0.00 & 96.31$\pm$0.09 & 95.33$\pm$0.21 & 1.00$\pm$0.00 \\
FedClust & 91.23$\pm$0.07 & 87.40$\pm$0.62 & 0.00$\pm$0.00 & 80.27$\pm$0.19 & 68.70$\pm$5.16 & 0.00$\pm$0.01 & 92.90$\pm$0.15 & 90.03$\pm$1.02 & 0.00$\pm$0.00 \\
IFCA & 95.68$\pm$0.02 & 93.35$\pm$0.35 & 1.00$\pm$0.00 & 93.22$\pm$2.01 & 90.95$\pm$2.47 & 0.50$\pm$0.50 & 96.29$\pm$0.03 & 95.20$\pm$0.14 & 1.00$\pm$0.00 \\
FlexCFL & 95.42$\pm$0.03 & 93.45$\pm$0.64 & 0.00$\pm$0.00 & 91.40$\pm$0.04 & 88.70$\pm$0.85 & -0.01$\pm$0.01 & 96.02$\pm$0.01 & 94.15$\pm$0.64 & 0.00$\pm$0.00 \\
FeSEM & 94.58$\pm$0.17 & 92.55$\pm$0.07 & 1.00$\pm$0.00 & 89.84$\pm$0.00 & 86.90$\pm$0.14 & 0.00$\pm$0.00 & 95.15$\pm$0.04 & 94.00$\pm$0.28 & 1.00$\pm$0.00 \\
CFLGP & 95.46$\pm$0.07 & 93.65$\pm$0.21 & 0.00$\pm$0.00 & 91.51$\pm$0.02 & 86.60$\pm$1.84 & -0.01$\pm$0.01 & 96.00$\pm$0.12 & 91.65$\pm$2.47 & 0.00$\pm$0.00 \\
FedEM & 94.22$\pm$0.13 & 91.95$\pm$0.35 & - & 91.29$\pm$0.13 & 88.55$\pm$0.07 & - & 95.05$\pm$0.04 & 93.15$\pm$0.35 & - \\
FedSoft & 94.32$\pm$0.21 & 89.70$\pm$0.28 & - & 92.84$\pm$0.14 & 89.60$\pm$0.14 & - & 95.36$\pm$0.05 & 88.80$\pm$0.14 & - \\
FedRC & 94.11$\pm$0.13 & 91.15$\pm$0.07 & - & 91.34$\pm$0.11 & 88.25$\pm$0.64 & - & 95.01$\pm$0.16 & 93.40$\pm$0.00 & - \\ %\cmidrule(lr){1-10}
FedCE & 95.74$\pm$0.04 & 93.60$\pm$0.30 & - & 92.38$\pm$0.61 & 90.00$\pm$0.26 & - & 96.29$\pm$0.05 & 95.17$\pm$0.15 & - \\ %\cmidrule(lr){1-10}
FedAvg & 95.75$\pm$0.12 & 93.60$\pm$0.00 & - & 92.54$\pm$0.96 & 90.15$\pm$0.64 & - & 96.20$\pm$0.04 & 95.05$\pm$0.07 & - \\
FedProx & 95.50$\pm$0.08 & 93.35$\pm$0.07 & - & 92.12$\pm$0.90 & 90.15$\pm$0.78 & - & 95.93$\pm$0.10 & 94.90$\pm$0.00 & - \\
pFedGraph & 95.58$\pm$0.04 & 93.60$\pm$0.00 & - & 91.29$\pm$0.06 & 84.95$\pm$0.07 & - & 96.17$\pm$0.04 & 94.95$\pm$0.07 & - \\
FedSaC & 95.57$\pm$0.02 & 93.55$\pm$0.07 & - & 88.70$\pm$0.14 & 78.05$\pm$0.21 & - & 96.08$\pm$0.03 & 95.05$\pm$0.07 & - \\
\bottomrule
\end{tabular}%
}
\end{table}

\end{document}